\documentclass{article}

\usepackage[nonatbib,final]{neurips_2020}

\usepackage[utf8]{inputenc} 
\usepackage[T1]{fontenc}    
\usepackage{hyperref}       
\usepackage{url}            
\usepackage{booktabs}       
\usepackage{amsfonts}       
\usepackage{nicefrac}       
\usepackage{microtype}      
\usepackage{babel}

\usepackage{amsmath}
\usepackage{amsthm}
\usepackage{amssymb}
\usepackage{mathtools}
\usepackage{wrapfig}

\usepackage{algorithm}
\usepackage{algorithmic}

\usepackage{xcolor}
\usepackage{enumitem}   

\usepackage{thmtools}
\usepackage{thm-restate}

\usepackage[backgroundcolor=White,textwidth=0.8in]{todonotes}

\def\multistep{h}
\def\eqdef{:=}

\def\Regret{\mathrm{Regret}}

\newtheorem{proposition}{Proposition}

\newtheorem{lemma}[proposition]{Lemma}
\newtheorem{theorem}[proposition]{Theorem}

\newtheorem{remark}{Remark}

\newenvironment{proofsketch}{%
  \proof}{\endproof}
\usepackage{bbm}
\newcommand{\Olog}{\tilde{\mathcal{O}}}

\usepackage{mathtools}
\DeclarePairedDelimiter\br{(}{)}
\DeclarePairedDelimiter\brs{[}{]}
\DeclarePairedDelimiter\brc{\{}{\}}
\DeclarePairedDelimiter\bra{|}{|}

\DeclarePairedDelimiter\abs{\lvert}{\rvert}
\DeclarePairedDelimiter\norm{\lVert}{\rVert}

\newcommand{\E}{\mathbb{E}}

\newcommand{\F}{\mathcal{F}}

\newcommand{\ind}{\mathbbm{1}}

\newcommand{\OuterV}{\bar{V}}

\title{Online Planning with Lookahead Policies}

\author{
  Yonathan Efroni \thanks{Part of this work was done during an internship in Facebook AI Research}\ \ \thanks{Microsoft Research, New York, NY} \ \thanks{Technion, Israel}\   
  \And
  Mohammad Ghavamzadeh \thanks{Google Research}
  \And
  Shie Mannor \footnotemark[3] \ \thanks{Nvidia Research}
  }

\begin{document}

\maketitle

\begin{abstract}
Real Time Dynamic Programming (RTDP) is an online algorithm based on Dynamic Programming (DP) that acts by 1-step greedy planning. Unlike DP, RTDP does not require access to the entire state space, i.e., it explicitly handles the exploration. This fact makes RTDP particularly appealing when the state space is large and it is not possible to update all states simultaneously. In this we devise a multi-step greedy RTDP algorithm, which we call $h$-RTDP, that replaces the 1-step greedy policy with a $h$-step lookahead policy. We  analyze $h$-RTDP in its exact form and establish that increasing the lookahead horizon, $h$, results in an improved sample complexity, with the cost of additional computations. This is the first work that proves improved sample complexity as a result of {\em increasing} the lookahead horizon in online planning. We then analyze the performance of $h$-RTDP in three approximate settings: approximate model, approximate value updates, and approximate state representation. For these cases, we prove that the asymptotic performance of $h$-RTDP remains the same as that of a corresponding approximate DP algorithm, the best one can hope for without further assumptions on the approximation errors.
\end{abstract}

\vspace{-0.2cm}

\section{Introduction}
\label{sec:intro}

\vspace{-0.2cm}

Dynamic Programming (DP) algorithms return an optimal policy, given a model of the environment. Their convergence in the presence of lookahead policies~\cite{bertsekas1996neuro,efroni2019combine} and their performance in different approximate settings~\cite{bertsekas1996neuro,munos2007performance,scherrer2012approximate,geist2013algorithmic,abel2017near,efroni2018multiple} have been well-studied. Standard DP algorithms require simultaneous access to the {\em entire} state space at run time, and as such, cannot be used in practice when the number of states is too large. Real Time Dynamic Programming (RTDP)~\cite{barto1995learning,strehl2006pac} is a DP-based algorithm that mitigates the need to access all states simultaneously. Similarly to DP, RTDP updates are based on the Bellman operator, calculated by accessing the model of the environment. However, unlike DP, RTDP learns how to act by interacting with the environment. In each episode, RTDP interacts with the environment, acts according to the greedy action w.r.t.~the Bellman operator, and samples a trajectory. RTDP is, therefore, an online planning algorithm.


Despite the popularity and simplicity of RTDP and its extensions~\cite{bonet2000planning,bonet2003labeled,mcmahan2005bounded,bulitko2006learning,strehl2006pac,kolobov2012lrtdp}, precise characterization of its convergence was only recently established for finite-horizon MDPs~\cite{efroni2019tight}.  While lookahead policies in RTDP are expected to improve the convergence in some of these scenarios, as they do for DP~\cite{bertsekas1996neuro,efroni2019combine}, to the best of our knowledge, these questions have not been addressed in previous literature. Moreover, previous research haven't addressed the questions of how lookahead policies should be used in RTDP, nor studied RTDP's sensitivity to possible approximation errors. Such errors can arise due to a misspecified model, or exist in value function updates, when e.g.,~function approximation is used.

    In this paper, we initiate a comprehensive study of lookahead-policy based RTDP with approximation errors in \emph{online planning}. We start by addressing the computational complexity of calculating lookahead policies and study its advantages in approximate settings. Lookahead policies can be computed naively by exhaustive search in  $O(A^\multistep)$ for deterministic environments or $O(A^{S\multistep})$ for stochastic environments. Since such an approach is infeasible, we offer in Section~\ref{sec: episodic complexity} an alternative approach for obtaining a lookahead policy with a computational cost that depends linearly on a natural measure: the total number of states reachable from a state in $\multistep$ time steps. The suggested approach is applicable both in deterministic and stochastic environments.

In Section~\ref{sec: mutiple step rtdp}, we introduce and analyze $\multistep$-RTDP, a RTDP-based algorithm that replaces the 1-step greedy used in RTDP by a $h$-step lookahead policy. The analysis of $\multistep$-RTDP reveals that the sample complexity is improved by increasing the lookahead horizon $\multistep$. To the best of our knowledge, this is the first theoretical result that relates sample complexity to the lookahead horizon in online planning setting. In Section~\ref{sec: approximate mutiple step rtdp}, we analyze $\multistep$-RTDP in the presence of three types of approximation: when (i) an inexact model is used, instead of the true one, (ii) the value updates contain error, and finally (iii) approximate state abstraction is used. Interestingly, for approximate state abstraction, $\multistep$-RTDP convergence and computational complexity depends on the size of the \emph{abstract state space}.


In a broader context, this work shows that RTDP-like algorithms could be a good alternative to Monte Carlo tree search (MCTS)~\cite{browne2012survey} algorithms, such as upper confidence trees (UCT)~\cite{kocsis2006bandit}, an issue that was empirically investigated in~\cite{kolobov2012lrtdp}. We establish strong convergence guarantees for extensions of $\multistep$-RTDP: under no assumption other than initial optimistic value, RTDP-like algorithms combined with lookahead policies converge in polynomial time to an optimal policy (see Table~\ref{tab:rtdp vs dp}), and their approximations inherit the asymptotic performance of approximate DP (ADP). Unlike RTDP, MCTS acts by using a $\sqrt{\log N/N}$ bonus term instead of optimistic initialization. However, in general, its convergence can be quite poor, even worse than uniformly random sampling~\cite{coquelin2007bandit,munos2014bandits}.  

\vspace{-0.2cm}

\section{Preliminaries}
\label{sec:prelim}

\vspace{-0.2cm}

{\bf Finite Horizon MDPs. } A finite-horizon MDP~\cite{bertsekas1996neuro} with time-independent dynamics\footnote{The results can also be applied to time-dependent MDPs, however, the notations will be more involved.} is a tuple $\mathcal{M} = \br*{\mathcal{S},\mathcal{A}, r, p, H}$, where $\mathcal{S}$ and $\mathcal{A}$ are the state and action spaces with cardinalities $S$ and $A$, respectively, $r(s,a)\in [0,1]$ is the immediate reward of taking action $a$ at state $s$, and $p(s'|s,a)$ is the probability of transitioning to state $s'$ upon taking action $a$ at state $s$. The initial state in each episode is arbitrarily chosen and  $H\in \mathbb{N}$ is the MDP's horizon. For any $N\in \mathbb{N}$,  denote $[N] \eqdef \brc*{1,\ldots,N}$. 

A deterministic policy $\pi: \mathcal{S}\times[H]\rightarrow \mathcal{A}$ is a mapping from states and time step indices to actions. We denote by $a_t \eqdef \pi_t(s)$ the action taken at time $t$ at state $s$ according to a policy $\pi$. The quality of a policy $\pi$ from a state $s$ at time $t$ is measured by its value function, i.e.,~$V_t^\pi(s) \eqdef \E\big[\sum_{t'=t}^H r\br*{s_{t'},\pi_{t'}(s_{t'})}\mid s_t=s\big]$, where the expectation is over all the randomness in the environment. An optimal policy maximizes this value for all states $s\in\mathcal{S}$ and time steps $t\in [H]$, i.e.,~$V_t^*(s) \eqdef \max_{\pi} V_t^\pi(s)$, and satisfies the optimal Bellman equation, 

\vspace{-0.15in}
\begin{align}
    V_t^*(s) &= TV_{t+1}^*(s) \eqdef \max_{a}\big(r(s,a) + p(\cdot|s,a) V_{t+1}^*\big) \nonumber \\
    &= \max_{a}\E\big[r(s_1,a) + V_{t+1}^*(s_2) \mid s_1 = s\big].\label{eq:bellman}
\end{align}
\vspace{-0.15in}

By repeatedly applying the optimal Bellman operator $T$, for any $\multistep\in[H]$, we have

\vspace{-0.15in}
\begin{align}
\label{eq:multistep bellman}
    V_t^*(s) = T^{\multistep}V_{t+h}^*(s) &= \max_{a}\big(r(s,a) + p(\cdot|s,a) T^{\multistep-1}V_{t+\multistep}^*\big) \nonumber \\ 
    &= \max_{\pi_t,\ldots,\pi_{t+h-1}} \E\Big[\sum_{t'=1}^\multistep r(s_{t'},\pi_{t+t'-1}(s_{t'})) + V_{t+h}^*(s_{\multistep+1}) \mid s_1 = s\Big].
\end{align}
\vspace{-0.15in}

We refer to $T^\multistep$ as the $h$-step optimal Bellman operator. Similar Bellman recursion is defined for the value of a given policy, $\pi$, i.e.,~$V^\pi$, as $V_t^\pi(s) = T^h_\pi V^\pi_{t+h}(s) \eqdef r(s,\pi_t(s)) + p(\cdot|s,\pi_t(s)) T^{h-1}_\pi V_{t+h}^\pi$, where $T^h_\pi$ is the $h$-step Bellman operator of policy $\pi$.

{\bf $\multistep$-Lookahead Policy.} An $\multistep$-lookahead policy w.r.t.~a value function $V\in\mathbb R^{S}$ returns the optimal first action in an $\multistep$-horizon MDP. For a state $s\in\mathcal S$, it returns

\vspace{-0.2in}
\begin{align}
\hspace{-0.1in} a_\multistep(s) &\in \arg\max_{a}\big(r(s,a) + p(\cdot|s,a) T^{\multistep-1}V\big)\nonumber \\
&=\arg\max_{\pi_1(s)} \max_{\pi_2,\ldots,\pi_{h}} \E\Big[\sum_{t=1}^{\multistep}r(s_t,\pi_t(s_t)) + V(s_{\multistep+1}) | s_1 = s\Big] 
    \label{eq: lookahead h greedy preliminaries}.
\end{align}
\vspace{-0.15in}

We can see $V$ represent our `prior-knowledge' of the problem. For example, it is possible to show~\cite{bertsekas1996neuro} that if $V$ is close to $V^*$, then the value of a $h$-lookahead policy w.r.t.~$V$ is close to $V^*$. 

For a state $s\in\mathcal S$ and a number of time steps $h\in[H]$, we define the set of reachable states from $s$ in $h$ steps as $\mathcal{S}_h(s) = \brc*{s'\mid \exists \pi: p^\pi(s_{\multistep+1}=s' \mid s_1=s,\pi)> 0}$, and denote by $S_h(s)$ its cardinality. 
We define the set of reachable states from $s$ in up to $h$ steps as $\mathcal{S}^{Tot}_{\multistep}(s) \eqdef \cup_{t=1}^h \mathcal{S}_t(s)$, its cardinality as $S^{Tot}_{\multistep}(s) \eqdef \sum_{t=1}^h S_{t}(s)$, and the maximum of this quantity over the entire state space as $S^{Tot}_h = \max_s S^{Tot}_h(s)$. Finally, we denote by $\mathcal{N} \eqdef S^{Tot}_1$ the maximum number of accessible states in $1$-step (neighbors) from any state.

{\bf Regret and Uniform-PAC. } We consider an agent that repeatedly interacts with an MDP in a sequence of episodes $[K]$. We denote by $s_t^k$ and $a_t^k$, the state and action taken at the time step $t$ of the $k$'th episode. We denote by $\F_{k-1}$, the filtration that includes all the events (states, actions, and rewards) until the end of the $(k-1)$'th episode, as well as the initial state of the $k$'th episode. Throughout the paper, we denote by $\pi_k$ the policy that is executed during the $k$'th episode and assume it is $\mathcal{F}_{k-1}$ measurable. 
The performance of an agent is measured by its \textit{regret}, defined as $\Regret(K)\eqdef \sum_{k=1}^K \br*{V_1^*(s_1^k) - V_1^{\pi_k}(s_1^k)}$, as well as by the \textit{Uniform-PAC} criterion~\cite{dann2017unifying}, which we generalize to deal with approximate convergence. Let $\epsilon,\delta>0$ and $N_{\epsilon}=\sum_{k=1}^\infty \ind\brc*{V_1^*(s_1^k)-V_1^{\pi_k}(s_1^k)\geq\epsilon}$ be the number of episodes in which the algorithm outputs a policy whose value is $\epsilon$-inferior to the optimal value. An algorithm is called Uniform-PAC, if $\Pr\br*{\exists \epsilon>0: N_\epsilon\geq F(S,1/\epsilon,\log1/\delta,H)}\leq \delta$, where $F(\cdot)$ depends polynomially (at most) on its parameters. Note that Uniform-PAC implies $(\epsilon,\delta)$-PAC, and thus, it is a stronger property. As we analyze algorithms with inherent errors in this paper, we use a more general notion of $\Delta$-Uniform-PAC by defining the random variable $N^{\Delta}_{\epsilon}\!\!=\!\!\sum_{k=1}^\infty \ind\brc*{V_1^*(s_1^k)\!-\!V_1^{\pi_k}(s_1^k)\geq  \Delta \!+\! \epsilon}$, where $\Delta>0$. Finally, we use $\Olog(x)$ to represent $x$ up to constants and poly-logarithmic factors in $\delta$, and $O(x)$ to represent $x$ up to constants.

\vspace{-0.2cm}

\section{Computing $h$-Lookahead Policies}
\label{sec: episodic complexity}

\vspace{-0.2cm}

Computing an action returned by a $h$-lookahead policy at a certain state is a main component in the RTDP-based algorithms we analyze in Sections~\ref{sec: mutiple step rtdp} and~\ref{sec: approximate mutiple step rtdp}. A `naive' procedure that returns such action is the exhaustive search. Its computational cost is  $O(A^\multistep)$ and $O(A^{S\multistep })$ for deterministic and stochastic systems, respectively. Such an approach is impractical, even for moderate values of $h$ or $S$. 



Instead of the naive approach, we formulate a Forward-Backward DP (FB-DP) algorithm, whose pseudo-code is given in Appendix~\ref{supp: epsiodic complexity h rtdp}. The FB-DP returns an action of an $\multistep$-lookahead policy from a given state $s$. Importantly, in both deterministic and stochastic systems, the computation cost of FB-DP depends linearly on the total \emph{number of reachable states} from $s$ in up to $\multistep$ steps, i.e.,~$S^{Tot}_\multistep(s)$. In the worst case, we may have $S_\multistep(s)=O(\min\br*{A^\multistep,S})$. However, when $S^{Tot}_\multistep(s)$ is small, significant improvement is achieved by avoiding unnecessary repeated computations.

FB-DP has two subroutines. It first constructs the set of reachable states from state $s$ in up to $h$ steps, $\{\mathcal{S}_t(s)\}_{t=1}^h$, in the `forward-pass'. Given this set, in the second `backward-pass' it simply applies backward induction (Eq.~\ref{eq: lookahead h greedy preliminaries}) and returns an action suggested by the $h$-lookahead policy, $a_h(s)$. Note that at each stage $t\in[h]$ of the backward induction (applied on the set $\{\mathcal{S}_t(s)\}_{t=1}^h$) there are $S_t(s)$ states on which the Bellman operator is applied. Since applying  the  Bellman operator costs  $O(\mathcal NA)$ computations, the computational cost of the `backward-pass' is $O\big(\mathcal{N} AS^{Tot}_{h}(s)\big)$.

%
%



In Appendix~\ref{supp: epsiodic complexity h rtdp}, we describe a DP-based approach to efficiently implement `forward-pass' and analyze its complexity. Specifically, we show the computational cost of the `forward-pass' is equivalent to that of the `backward-pass' (see Propsition~\ref{proposition: computational complexity of forward pass}). Meaning, the computational cost of FB-DP is $O\big(\mathcal{N} AS^{Tot}_{h}(s))$ - same order as the cost of backward induction given the set $\mathcal{S}^{Tot}_h(s)$. 




\vspace{-0.2cm}

\section{Real-Time Dynamic Programming}
\label{sec:RTDP}

\vspace{-0.2cm}

Real-time dynamic programming (RTDP)~\cite{barto1995learning} is a well-known online planning algorithm that assumes access to a transition model and a reward function. Unlike DP algorithms (policy, value iteration, or asynchronous value iteration)~\cite{bertsekas1996neuro} that solve an MDP using offline calculations and sweeps over the entire states (possibly in random order), RTDP solves it in real-time, using samples from the environment (either simulated or real) and DP-style Bellman updates from the current state. Furthermore, unlike DP algorithms, RTDP needs to tradeoff exploration-exploitaion, since it interacts with the environment via sampling trajectories. This makes RTDP a good candidate for problems in which having access to the entire state space is not possible, but interaction is.



Algorithm~\ref{algo: RTDP} contains the pseudo-code of RTDP in finite-horizon MDPs. The value is initialized optimistically, ~$\bar{V}^0_{t+1}(s)=H-t\geq V^*_{t+1}(s)$. At each time step $t\in[H]$ and episode $k\in[K]$, the agent updates the value of the current state $s_t^k$ by the optimal Bellman operator. It then acts greedily w.r.t.~the current value at the next time step $\bar{V}^{k-1}_{t+1}$. Finally, the next state, $s_{t+1}^k$, is sampled either from the model or the real-world. When the model is exact, there is no difference in sampling from the model and real-world, but these are different in case the model is inexact as in Section~\ref{sec: appr model}.


The following high probability bound on the regret of a Decreasing Bounded Process (DBP), proved in~\cite{efroni2019tight}, plays a key role in our analysis of exact and approximate RTDP with lookahead policies in Sections~\ref{sec: mutiple step rtdp} and~\ref{sec: approximate mutiple step rtdp}. An adapted process ~$\brc*{X_k,\F_k}_{k\geq 0}$ is a DBP, if for all $k\geq 0$, {\bf (i)} $X_k\leq X_{k-1}$ almost surely (a.s.), {\bf (ii)} $X_k\geq C_2$, and {\bf (iii)} $X_0=C_1\geq C_2$. Interestingly, contrary to the standard regret bounds (e.g.,~ in bandits), this bound does not depend on the number of rounds $K$.



\begin{theorem}[Regret Bound of a DBP \cite{efroni2019tight}]
\label{theorem: regret of decreasing bounded process}
Let $\brc*{X_k,\F_k}_{k\geq 0}$ be a DBP and $R_K = \sum_{k=1}^K X_{k-1} - \E[X_{k}\mid \F_{k-1}]$ be its $K$-round regret. Then, $$\Pr\brc*{\exists K>0: R_K \geq 9(C_1-C_2)\ln(3/\delta)} \le \delta.$$
\end{theorem}

\vspace{-0.2cm}

\section{RTDP with Lookahead Policies} 
\label{sec: mutiple step rtdp}

\vspace{-0.2cm}


In this section, we devise and analyze a lookahead-based RTDP algorithm, called $h$-RTDP, whose pseudo-code is shown in Algorithm~\ref{algo: multi step RTDP}. Without loss of generality, we assume that $H/\multistep\in \mathbbm{N}$. We divide the horizon $H$ into $H/\multistep$ intervals, each of length $h$ time steps. $h$-RTDP stores $HS/\multistep$ values in the memory, i.e.,~the values at time steps  $\mathcal H=\{1,\multistep+1,\ldots,H+1\}$.\footnote{In fact, $\multistep$-RTDP does not need to store $V_1$ and $V_{H+1}$, they are only used in the analysis.} 
For each time step $t\in[H]$, we denote by $h_c\in\mathcal H$, the next time step for which a value is stored in the memory, and by $t_c=h_c-t$, the number of time steps until there (see Figure~\ref{fig:h greedy policy}). At each time step $t$ of an episode $k\in[K]$, given the current state $s_t^k$, $\multistep$-RTDP selects an action $a_t^k$ returned by the $t_c$-lookahead policy w.r.t.~$\bar {V}_{h_c}^{k-1}$, 

\vspace{-0.15in}
\begin{equation}
    a_t^k = a_{t_c}(s_t^k)\in \arg\max_{\pi_1(s_t^k)} \max_{\pi_2,\ldots,\pi_{t_c}}\E\Big[\sum_{t'=1}^{t_c}r(s_{t'},\pi_{t'}(s_{t'})) + \bar{V}^{k-1}_{h_c}(s_{t_c+1}) \mid s_1=s_t^k\Big]. 
\label{eq: h greedy policy}
\end{equation}
\vspace{-0.1in}

\begin{wrapfigure}{r}{0.47\textwidth}
\vspace{-0.15in}
\centering
\def\svgwidth{5cm}
\begingroup%
  \makeatletter%
  \providecommand\color[2][]{%
    \errmessage{(Inkscape) Color is used for the text in Inkscape, but the package 'color.sty' is not loaded}%
    \renewcommand\color[2][]{}%
  }%
  \providecommand\transparent[1]{%
    \errmessage{(Inkscape) Transparency is used (non-zero) for the text in Inkscape, but the package 'transparent.sty' is not loaded}%
    \renewcommand\transparent[1]{}%
  }%
  \providecommand\rotatebox[2]{#2}%
  \newcommand*\fsize{\dimexpr\f@size pt\relax}%
  \newcommand*\lineheight[1]{\fontsize{\fsize}{#1\fsize}\selectfont}%
  \ifx\svgwidth\undefined%
    \setlength{\unitlength}{27.96714182bp}%
    \ifx\svgscale\undefined%
      \relax%
    \else%
      \setlength{\unitlength}{\unitlength * \real{\svgscale}}%
    \fi%
  \else%
    \setlength{\unitlength}{\svgwidth}%
  \fi%
  \global\let\svgwidth\undefined%
  \global\let\svgscale\undefined%
  \makeatother%
  \begin{picture}(1,0.53648228)%
    \lineheight{1}%
    \setlength\tabcolsep{0pt}%
    \put(0,0){\includegraphics[width=\unitlength,page=1]{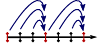}}%
    \put(0.7690776,0.01894445){\color[rgb]{0,0,0}\makebox(0,0)[lt]{\lineheight{1.25}\smash{\begin{tabular}[t]{l}{$H\!+\!1\!=\!7$}\end{tabular}}}}%
    \put(0.02476583,0.01894443){\color[rgb]{0,0,0}\makebox(0,0)[lt]{\lineheight{1.25}\smash{\begin{tabular}[t]{l}{$t\!=\!1$}\end{tabular}}}}%
    \put(0.33775495,0.01894443){\color[rgb]{0,0,0}\makebox(0,0)[lt]{\lineheight{1.25}\smash{\begin{tabular}[t]{l}{$h\!+\!1\!=\!4$}\end{tabular}}}}%
    \put(0.90917855,0.21597288){\color[rgb]{0,0,0}\makebox(0,0)[lt]{\lineheight{1.25}\smash{\begin{tabular}[t]{l}{\Large $t$}\end{tabular}}}}%
  \end{picture}%
\endgroup%

\caption{\begin{small}Varying lookahead horizon of a $\multistep$-greedy policy in $\multistep$-RTDP (see Eq.~\ref{eq: h greedy policy}) with $\multistep=3$ and $H=6$. The blue arrows show the lookahead horizon from a specific time step $t$, and the red bars are the time steps for which a value is stored in memory, i.e.,~${\mathcal H=\{1\;,\;h+1=4\;,\;2h+1=H+1=7\}}$.\end{small}}
\label{fig:h greedy policy}
\end{wrapfigure}
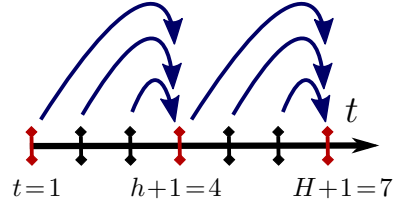

Thus, $h$-RTDP uses a varying lookahead horizon $t_c$ that depends on how far the current time step is to the next one for which a value is stored. Throughout the paper, with an abuse of notation, we refer to this policy as a $\multistep$-lookahead policy. Finally, it can be seen that $h$-RTDP generalizes RTDP as they are equal for $h=1$. 

We are now ready to establish finite-sample performance guarantees for $h$-RTDP; see Appendix~\ref{supp: multistep rtdp} for the detailed proofs. We start with two lemmas from which we derive the main convergence result of this section.

\begin{restatable}{lemma}{multistepRtdpProperties}
\label{lemma:multistep rtdp properties}
For all $s\in \mathcal{S}$, $n\in \{0\}\cup[\frac{H}{\multistep}]$, and ${k\in [K]}$, the value function of $h$-RTDP is (i) Optimistic: $V^*_{n\multistep+1}(s) \leq \bar{V}^k_{n\multistep+1}(s)$, and (ii) Non-Increasing: $\bar{V}^{k}_{n\multistep+1}(s) \leq \bar{V}^{k-1}_{n\multistep+1}(s)$.
\end{restatable}


\vspace{-0.2cm}

\begin{minipage}[t]{0.49\linewidth}
\footnotesize
\begin{algorithm}[H]
\begin{algorithmic}
\caption{ \footnotesize
 Real-Time DP (RTDP)}
\label{algo: RTDP}
    \STATE {\bf init:} $\forall s\in \mathcal S,\; \forall t\in \{0\}\cup[H],$ 
    \STATE \quad  \quad$\bar{V}^0_{t+1}(s)=H-t$
    \FOR{$k\in [K]$}
        \STATE Initialize $s^k_1$ arbitrarily
        \FOR{$t\in [H]$}
            \STATE $\bar{V}^{k}_{t}(s_t^k) =  T \bar{V}^{k-1}_{t+1}(s_t^k)$
            \STATE $a_t^k\in \arg\max_a  r(s_t^k,a)+ p(\cdot|s_t^k,a) \bar{V}^{k-1}_{t+1}$
            \STATE Act by $\;a_t^k\;$, observe $\;s_{t+1}^k\sim p(\cdot \mid s_t^k,a_t^k)$ 
        \ENDFOR
    \ENDFOR
\end{algorithmic}
\vspace*{1.605 cm}
\end{algorithm}
\end{minipage}
\hspace{3pt}
\begin{minipage}[t]{0.49\linewidth}
\footnotesize
\begin{algorithm}[H]
\begin{algorithmic}
\caption{\footnotesize
 RTDP with Lookahead ($h$-RTDP)}
\label{algo: multi step RTDP}
    \STATE {\bf init:}: $\forall s\in \mathcal S,\;  n\in \{0\}\cup[\frac{H}{\multistep}],$
    \STATE \quad \quad$\bar{V}^0_{n \multistep +1}(s)=H-n\multistep$
    \FOR{$k\in [K]$}
        \STATE Initialize $s^k_1$ arbitrarily
        \FOR{$t\in [H]$}
            \IF{$(t-1) \mod \multistep = 0$}
                \STATE $h_c = t + h$
                \STATE $\bar{V}^{k}_{t}(s_t^k) = T^{\multistep}\bar{V}^{k-1}_{h_c}(s_t^k)$
            \ENDIF
            \STATE $a_t^k\in$
            \STATE ${\arg\max_a r(s_t^k,a)+ p(\cdot|s_t^k,a) T^{h_c-t-1}\bar{V}^{k-1}_{h_c}}$
            \STATE Act by $\;a_t^k\;$, observe $\;s_{t+1}^k\sim p(\cdot \mid s_t^k,a_t^k)$  
        \ENDFOR
    \ENDFOR
\end{algorithmic}
\end{algorithm}
\end{minipage}
\vspace{0.2cm}

\begin{restatable}[Optimality Gap and Expected Decrease]{lemma}{MultistepRdtpExpectedValueUpdate}
\label{lemma: multistep RTDP expected value difference}
The expected cumulative value update at the $k$'th episode of $h$-RTDP satisfies $\bar{V}_1^{k}(s^k_1)-V_1^{\pi_k}(s^k_1) = \sum_{n=1}^{\frac{H}{\multistep}-1}\sum_{s\in\mathcal{S}} \bar{V}^{k-1}_{n\multistep+1}(s) - \E[\bar{V}^{k}_{n\multistep+1}(s)\mid \F_{k-1}]$.
%
\end{restatable}

Properties (i) and (ii) in Lemma~\ref{lemma:multistep rtdp properties} show that $\{\bar V^k_{n\multistep + 1}(s)\}_{k\geq 0}$ is a DBP, for any $s$ and $n$. Lemma~\ref{lemma: multistep RTDP expected value difference} relates $\bar{V}_1^{k}(s^k_1)-V_1^{\pi_k}(s^k_1)$ (LHS) to the expected decrease in $\bar V^{k}$ at the $k$'th episode (RHS). When the LHS is small, then $\bar V_1^{k}(s_1^k)\simeq  V^*_1(s_1^k)$, due to the optimism of $\bar V_1^{k}$, and $\multistep$-RTDP is about to converge to the optimal value. This is why we refer to the LHS as the {\em optimality gap}. Using these two lemmas and the regret bound of a DBP (Theorem~\ref{theorem: regret of decreasing bounded process}), we prove a finite-sample convergence result for $h$-RTDP (see Appendix~\ref{supp: multistep rtdp} for the full proof).

\begin{restatable}[Performance of $\multistep$-RTDP]{theorem}{TheoremRegretMultistepRTDP}
\label{theorem: regret multistep rtdp}
Let $\epsilon,\delta>0$. The following holds for $h$-RTDP:

$\;\;$ 1. With probability $1-\delta$, for all $K>0$, $\;\;\Regret(K) \leq  \frac{9SH(H-\multistep)}{\multistep}\ln(3/\delta)$.

$\;\;$ 2. 
$\;\;\Pr\brc*{\exists \epsilon>0 \; : \; N_\epsilon \geq \frac{9SH(H-\multistep)\ln(3/\delta)}{\multistep\epsilon}}\leq \delta$.
\end{restatable}

\vspace{-0.2cm}

\begin{proofsketch}
Applying Lemmas~\ref{lemma:multistep rtdp properties} and~\ref{lemma: multistep RTDP expected value difference}, we may write 

\vspace{-0.25in}
\begin{align}
\Regret(K) &\leq \sum_{k=1}^K \bar{V}_1^{k}(s^k_1)- V_1^{\pi_k}(s^k_1) = \sum_{k=1}^K \sum_{n=1}^{\frac{H}{\multistep}-1}\sum_s \bar{V}^{k-1}_{n\multistep+1}(s) - \E[\bar{V}^{k}_{n\multistep+1}(s)\mid \F_{k-1}] \nonumber \\ 
&= \sum_{k=1}^K X_{k-1} - \E[X_k \mid \mathcal{F}_{k-1}].
\label{eq:bound-temp0}
\end{align}
\vspace{-0.15in}

Where we define ${X_k \eqdef \sum_{n=1}^{\frac{H}{\multistep}-1}\sum_s \bar{V}^{k-1}_{n\multistep+1}(s)}$ and use linearity of expectation. By Lemma~\ref{lemma:multistep rtdp properties}, $\brc*{X_k}_{k\geq 0}$ is decreasing and bounded from below by $\sum_{n=1}^{\frac{H}{\multistep}-1}\sum_{s} V^*_{n\multistep +1}(s) \geq  0$. We conclude the proof by observing that $X_0\leq \sum_{n=1}^{\frac{H}{\multistep}-1}\sum_{s} V^0_{n\multistep +1}(s) \leq SH(H-\multistep)/\multistep$, and applying Theorem~\ref{theorem: regret of decreasing bounded process}.
\end{proofsketch}

\begin{remark}[RTDP and Good Value Initialization]
A closer look into the proof of Theorem~\ref{theorem: regret multistep rtdp} shows we can easily obtain a stronger result  which depends on the initial value $V^0$. The regret can be bounded by $\Regret(K)\leq \Olog\br*{\sum_{n=1}^{\frac{H}{\multistep}-1} \br*{V^0_{n\multistep +1}(s) -V^*_{n\multistep +1}(s)}},$
which formalizes the intuition the algorithm improves as the initial value $V^0$ better estimates $V^*$.  For clarity purposes we provide the worse-case bound.
\end{remark}


\begin{remark}[Computational Complexity of $\multistep$-RTDP]\label{remark: space-comp compleixty of h rtdp}
Using FB-DP (Section~\ref{sec: episodic complexity}) as a solver of a $\multistep$-lookahead policy, the per-episode \emph{computation cost} of $h$-RTDP amounts to applying FB-DP for $H$ time steps, i.e.,~it is bounded by $O(H\mathcal{N}AS^{Tot}_{\multistep})$. Since $S^{Tot}_{\multistep}$ -- the total number of reachable states in up to $\multistep$ time steps -- is an increasing function of $\multistep$, the computation cost of $\multistep$-RTDP increases with $\multistep$, as expected. When $S^{Tot}_{\multistep}$ is significantly smaller than $S$, the per-episode computational complexity of $\multistep$-RTDP is $S$ independent. As discussed in Section~\ref{sec: episodic complexity}, using FB-DP, in place of exhaustive search, can significantly improve the computational cost of $h$-RTDP.
\end{remark}

\begin{remark}[Improved Sample Complexity of $\multistep$-RTDP]\label{remark: sample compleixty of h rtdp}
Theorem~\ref{theorem: regret multistep rtdp} shows that $h$-RTDP improves the \emph{sample complexity} of RTDP by a factor $1/\multistep$. This is consistent with the intuition that larger horizon of the applied lookahead policy results in faster convergence (less samples). Thus, if RTDP is used in a real-time manner, one way to boost its performance is to combine with lookahead policies.
\end{remark}

\begin{remark}[Sparse Sampling Approaches]
In this work, we assume $\multistep$-RTDP has access to a $\multistep$-lookahead policy~\eqref{eq: lookahead h greedy preliminaries} solver, such as FB-DP presented in Section~\ref{sec: episodic complexity}. We leave studying the sparse sampling approach~\cite{kearns2002sparse, sidford2018variance} for approximately solving $\multistep$-lookahead policy for future work.
\end{remark}

\vspace{-0.4cm}

\section{Approximate RTDP with Lookahead Policies} \label{sec: approximate mutiple step rtdp}

\vspace{-0.2cm}


In this section, we consider three approximate versions of $h$-RTDP in which the update deviates from its exact form described in Section~\ref{sec: mutiple step rtdp}. We consider the cases in which there are errors in the {\bf 1)} {\em model}, {\bf 2)} {\em value updates}, and when we use {\bf 3)} {\em approximate state abstraction}. We prove finite-sample bounds on the performance of $h$-RTDP in the presence of these approximations. Furthermore, in Section~\ref{sec: appr abstractions}, given access to an approximate state abstraction, we show that the convergence of $\multistep$-RTDP depends on the cardinality of the \emph{abstract state space} -- which can be much smaller than the original one. The proofs of this section generalize that of Theorem~\ref{theorem: regret multistep rtdp}, while following the same `recipe'. This shows the generality of the proof technique, as it works for both exact and approximate settings. 



\subsection{$h$-RTDP with Approximate Model ($h$-RTDP-AM)}\label{sec: appr model}


In this section, we analyze a more practical scenario in which the transition model used by $h$-RTDP to act and update the values is not exact. We assume it is close to the true model in the total variation ($TV$) norm,~$\forall (s,a) \in\mathcal{S}\times \mathcal{A},\ ||p(\cdot|s,a) - \hat{p}(\cdot|s,a) ||_1\leq \epsilon_P$, where~$\hat{p}$ denotes the approximate model. Throughout this section and the relevant appendix (Appendix~\ref{supp: multistep rtdp approximate model}), we denote by $\hat T$ and $\hat{V}^*$ the optimal Bellman operator and optimal value of the approximate model $\hat{p}$, respectively. Note that $\hat T$ and $\hat{V}^*$ satisfy~\eqref{eq:bellman} and~\eqref{eq:multistep bellman} with $p$ replaced by $\hat{p}$. $h$-RTDP-AM is exactly the same as $h$-RTDP (Algorithm~\ref{algo: multi step RTDP}) with the model $p$ and optimal Bellman operator $T$ replaced by their approximations $\hat p$ and $\hat T$. We report the pseudocode of $h$-RTDP-AM in Appendix~\ref{supp: multistep rtdp approximate model}.

Although we are given an approximate model, $\hat p$, we are still interested in the performance of (approximate) $h$-RTDP on the \emph{true MDP}, $p$, and relative to its optimal value, $V^*$. If we solve the approximate model and act by its optimal policy, the Simulation Lemma~\cite{kearns2002near,strehl2009reinforcement} suggests that the regret is bounded by $O(H^2\epsilon_P K)$. For $h$-RTDP-AM, the situation is more involved, as its updates are based on the approximate model and the samples are gathered by interacting with the true MDP. Nevertheless, by properly adjusting the techniques from Section~\ref{sec: mutiple step rtdp}, we derive performance bounds for $h$-RTDP-AM. These bounds reveal that the asymptotic regret increases by at most $O(H^2\epsilon_P K)$, similarly to the regret of the optimal policy of the approximate model.  
%
%
%
Interestingly, the proof technique follows that of the exact case in Theorem~\ref{theorem: regret multistep rtdp}. We generalize Lemmas~\ref{lemma:multistep rtdp properties} and~\ref{lemma: multistep RTDP expected value difference} from Section~\ref{sec: mutiple step rtdp} to the case that the update rule uses an inexact model (see Lemmas~\ref{lemma: approximate model properties} and~\ref{lemma: RTDP approximate modle expected value difference} in Appendix~\ref{supp: multistep rtdp approximate model}). This allows us to establish the following performance bound for $h$-RTDP-AM (proof in Appendix~\ref{supp: multistep rtdp approximate model}).

\begin{restatable}[Performance of $\multistep$-RTDP-AM]{theorem}{TheoremRegretRTDPApproximateModel}
\label{theorem: regret rtdp approximate model}
Let $\epsilon,\delta>0$. The following holds for $h$-RTDP-AM:

$\;\;$ 1. With probability $1-\delta$, for all $K>0$,  $\;\;\Regret(K) \leq  \frac{9SH(H-\multistep)}{\multistep}\ln(3/\delta)+ H(H-1)\epsilon_P K$.

$\;\;$ 2. Let $\Delta_P=H(H-1)\epsilon_P$. Then,  $\;\;\Pr\Big\{\exists \epsilon>0 \; : \; N^{\Delta_P}_\epsilon \geq \frac{9SH(H-\multistep)\ln(3/\delta)}{\multistep\epsilon}\Big\}\leq \delta$.
\end{restatable}

\vspace{-0.2cm}

These bounds show the approximate convergence resulted from the approximate model. However, the asymptotic performance gaps -- both in terms of the regret and Uniform PAC -- of $\multistep$-RTDP-AM approach those experienced by an optimal policy of the approximate model. Interestingly, although $\multistep$-RTDP-AM updates using the approximate model, while interacting with the true MDP, its convergence rate (to the asymptotic performance) is similar to that of $h$-RTDP (Theorem~\ref{theorem: regret multistep rtdp}).

\vspace{-0.1cm}

\subsection{$h$-RTDP with Approximate Value Updates ($h$-RTDP-AV)}\label{sec: appr value}

\vspace{-0.1cm}

Another important question in the analysis of approximate DP algorithms is their performance under approximate value updates, motivated by the need to use function approximation. This is often modeled by an extra noise $|\epsilon_V(s)|\leq \epsilon_V$ added to the update rule~\cite{bertsekas1996neuro}. Following this approach, we study such perturbation in $\multistep$-RTDP. Specifically, in $h$-RTDP-AV the value update rule is modified such that it contains an error term (see Algorithm~\ref{algo: multi step RTDP}), 
\begin{align*}
    \bar{V}^{k}_{t}(s_t^k) =\epsilon_V(s_t^k)+ T^{\multistep}\bar{V}^{k-1}_{h_c}(s_t^k).
\end{align*}
For $\epsilon_V(s_t^k)=0$, the exact $\multistep$ is recovered. The pseudocode of $h$-RTDP-AV is supplied in Appendix~\ref{supp: multistep rtdp approximate value updates}.

Similar to the previous section, we follow the same proof technique as for Theorem~\ref{theorem: regret multistep rtdp} to establish the following performance bound for $h$-RTDP-AV (proof in Appendix~\ref{supp: multistep rtdp approximate value updates}).

\begin{restatable}[Performance of $\multistep$-RTDP-AV]{theorem}{TheoremRegretApproximateValueRTDP}
\label{theorem: regret rtdp appoximate value updates}
Let $\epsilon,\delta>0$. The following holds for $h$-RTDP-AV:

$\;\;$ 1. With probability $1-\delta$, for all $K>0$, $\;\;\Regret(K) \leq \frac{9SH(H-\multistep)}{\multistep}(1+\frac{H}{\multistep}\epsilon_V)\ln(\frac{3}{\delta})+ \frac{2H}{\multistep}\epsilon_V K$.

$\;\;$ 2. Let $\Delta_V = 2H\epsilon_V$. Then, $\;\;\Pr\Big\{\exists \epsilon>0 \; : \; N^{\frac{\Delta_V}{h}}_\epsilon \geq\frac{9SH(H-\multistep)(1+\frac{\Delta_V}{2\multistep})\ln(\frac{3}{\delta})}{\multistep \epsilon}\Big\}\leq \delta$.
\end{restatable}

\vspace{-0.2cm}

As in Section~\ref{sec: appr model}, the results of Theorem~\ref{theorem: regret rtdp appoximate value updates} exhibit an asymptotic linear regret $O(H\epsilon_V K/\multistep )$. As proven in Proposition~\ref{prop: approximate value updates} in Appendix~\ref{supp: approximate dp bounds}, such performance gap exists in ADP with approximate value updates. Furthermore, the convergence rate in $S$ to the asymptotic performance of $\multistep$-RTDP-AV is similar to that of its exact version (Theorem~\ref{theorem: regret multistep rtdp}). Unlike in $h$-RTDP-AM, the asymptotic performance of $h$-RTDP-AV \emph{improves} with $\multistep$. This quantifies a clear benefit of using lookahead policies in online planning when the value function is approximate.

\vspace{-0.1cm}

\subsection{$h$-RTDP with Approximate State Abstraction ($h$-RTDP-AA)}\label{sec: appr abstractions}
\vspace{-0.1cm}

We conclude the analysis of approximate $\multistep$-RTDP with exploring the advantages of combining it with approximate state abstraction~\cite{abel2017near}. The central result of this section establishes that given an approximate state abstraction, $h$-RTDP converges with sample, computation, and space complexity \emph{independent} of the size of the state space $S$, as long as $S^{Tot}_\multistep$ is smaller than $S$ (i.e., when performing $\multistep$-lookahead is $S$ independent, Remark~\ref{remark: space-comp compleixty of h rtdp}). This is in contrast to the computational complexity of ADP in this setting, which is still $O(HSA)$ (see Appendix~\ref{supp: adp approximate abstractions} for further discussion). 
State abstraction has been widely investigated in approximate planning~\cite{dearden1997abstraction,dean1997model,even2003approximate,abel2017near}, as a means to deal with large state space problems. Among existing approximate abstraction settings, we focus on the following one. For any $n\in\{0\}\cup[\frac{H}{h}-1]$, we define ${\phi_{nh+1}: \mathcal{S}\rightarrow \mathcal{S}_\phi}$ to be a mapping from the state space $\mathcal{S}$ to reduced space $\mathcal{S}_\phi$,~$S_\phi=|\mathcal{S}_\phi|\ll S$. We make the following assumption:
\begin{restatable}[Approximate Abstraction, \cite{li2006towards}, definition 3.3]{assumption}{assumptionModelAbstraction}
\label{assumption: model abstraction}
For any $s,s'\in \mathcal{S}$ and $n\in\{0\}\cup[\frac{H}{\multistep}-1]$ for which $\phi_{n\multistep+1}(s) = \phi_{n\multistep+1}(s')$, we have $|V_{nh+1}^*(s)-V_{nh+1}^*(s')|\leq \epsilon_A$.
\end{restatable}

\vspace{-0.2cm}

Let us denote by $\{\bar{V}^k_{\phi,n\multistep+1}\}_{n=0}^{H/\multistep}$ the values stored in memory by $\multistep$-RTDP-AA at the $k$'th episode. Unlike previous sections, the value function per time step contains $S_\phi$ entries, $\bar V^k_{\phi,1+n\multistep}\in \mathbb{R}^{S_\phi}$. Note that if $\epsilon_A=0$, then optimal value function can be represented in the reduced state space~$\mathcal{S}_\phi$. However, if $\epsilon_A$ is positive, exact representation of $V^*$ is not possible. Nevertheless, the asymptotic performance of $\multistep$-RTDP-AA will be `close',  up to error of $\epsilon_A$, to the optimal policy. 




Furthermore, the definition of the multi-step Bellman operator~\eqref{eq:multistep bellman} and $\multistep$-greedy policy~\eqref{eq: lookahead h greedy preliminaries} should be revised, and with some abuse of notation, defined as

\vspace{-0.5cm}
\begin{align}
    &a_t^k\in \arg\max_{\pi_0(s_t^k)}  \max_{\pi_1,\ldots,\pi_{t_c-1}}\E\brs*{ \sum_{t'=0}^{t_c-1}r_{t'} + \bar{V}^{k-1}_{\phi, h_c}(\phi_{h_c}(s_{t_c}))\mid s_0=s_t^k}, \label{eq: h greedy policy and bellman abstraction}\\
    &T_\phi^{\multistep}\bar{V}^{k-1}_{\phi, h_c}(s_{t}^k) \eqdef \max_{\pi_0,\ldots,\pi_{\multistep-1}}\E\brs*{ \sum_{t'=0}^{\multistep-1} r_{t'} + \bar{V}^{k-1}_{\phi,t+\multistep}(\phi_{t+\multistep}(s_{\multistep})) \mid s_0 =s_t^k}. \label{eq: h value and bellman abstraction}
\end{align}
\vspace{-0.4cm}

Eq.~\eqref{eq: h greedy policy and bellman abstraction} and~\eqref{eq: h value and bellman abstraction} indicate that similar to~\eqref{eq: lookahead h greedy preliminaries}, the $\multistep$-lookahead policy uses the given model to plan for $\multistep$ time steps ahead. Differently from~\eqref{eq: lookahead h greedy preliminaries}, the value after $\multistep$ time steps is the one defined in the \emph{reduced state} space $\mathcal{S}_\phi$. Note that the definition of the $\multistep$-greedy policy for $\multistep$-RTDP-AA  in~\eqref{eq: h greedy policy and bellman abstraction} is equivalent to the one used in Algorithm~\ref{algo: RTDP with abstractions}, obtained by similar recursion as for the optimal Bellman operator~\eqref{eq:multistep bellman}. 


$\multistep$-RTDP-AA modifies both the value update and the calculation of the $\multistep$-lookahead policy (the value update and action choice in algorithm~\ref{algo: multi step RTDP}). The $\multistep$-lookahead policy is replaced by $\multistep$-lookahead defined in~\eqref{eq: h greedy policy and bellman abstraction}. The value update is substituted by~\eqref{eq: h value and bellman abstraction}, i.e, $\bar{V}_{\phi,t}^k(\phi_{t}(s_t^k)) = T_\phi^{\multistep}\bar{V}^{k-1}_{\phi, h_c}(s_{t}^k)$. The full pseudocode of $h$-RTDP-AA is supplied in Appendix~\ref{supp: multistep rtdp abstractions}. By similar technique, as in the proof of Theorem~\ref{theorem: regret multistep rtdp}, we establish the following performance guarantees to $\multistep$-RTDP-AA (proof in Appendix~\ref{supp: multistep rtdp abstractions}). 

\begin{restatable}[Performance of $\multistep$-RTDP-AA]{theorem}{TheoremRegretRTDPAbstraction}
\label{theorem: regret rtdp abstraction}
Let $\epsilon,\delta>0$. The following holds for $h$-RTDP-AA:

$\;\;$ 1. With probability $1-\delta$, for all $K>0$, $\;\;\Regret(K) \leq  \frac{9S_\phi H(H-\multistep)}{\multistep}\ln(3/\delta)+ \frac{H\epsilon_A}{\multistep}K$.

$\;\;$ 2. Let $\Delta_A = H\epsilon_A$. Then, $\;\;\Pr\Big\{\exists \epsilon>0 \; : \; N^{\frac{\Delta_A}{h}}_\epsilon \geq \frac{9S_\phi H(H-\multistep)\ln(3/\delta)}{\multistep \epsilon}\Big\}\leq \delta$.
\end{restatable}

\vspace{-0.2cm}

Theorem~\ref{theorem: regret rtdp abstraction} establishes $S$-independent performance bounds that depend on the size of the reduced state space $S_\phi$. The asymptotic regret and Uniform PAC guarantees are approximate, as the state abstraction is approximate. Furthermore, they are improving with the quality of approximation $\epsilon_A$, i.e.,~their asymptotic gap is $O(H\epsilon_A/\multistep)$ relative to the optimal policy. Moreover, the asymptotic performance of $h$-RTDP-AA improves as $\multistep$ is increased. Importantly, since the computation complexity of each episode of $h$-RTDP is independent of $S$ (Section~\ref{sec: episodic complexity}), the computation required to reach the approximate solution in $h$-RTDP-AA is also $S$-independent. This is in contrast to the computational cost of DP that depends on $S$ and is $O(SHA)$ (see Appendix~\ref{supp: adp approximate abstractions} for further discussion).  


\vspace{-0.2cm}

\section{Discussion and Conclusions}
\label{sec: rtdp vs dp}

\vspace{-0.2cm}

\paragraph{RTDP vs.~DP.} The results of Sections~\ref{sec: mutiple step rtdp} and~\ref{sec: approximate mutiple step rtdp} established finite-time convergence guarantees for the exact $h$-RTDP and its three approximations. In the approximate settings, as expected, the regret has a linear term of the form $\Delta K$, where $\Delta$ is linear in the approximation errors $\epsilon_P$, $\delta$, and $\epsilon_A$, and thus, the performance is continuous in these parameters, as we would desire. We refer to $\Delta K$ as the \emph{asymptotic regret}, since it dominates the regret as $K\rightarrow \infty$. 

A natural measure to evaluate the quality of $h$-RTDP in the approximate settings is comparing its regret to that of its corresponding approximate DP (ADP). Table~\ref{tab:rtdp vs dp} summarizes the regrets of the approximate $h$-RTDPs studied in this paper and their corresponding ADPs. ADP calculates approximate values $\{V_{n\multistep+1}^*\}_{n=0}^{H/h}$ by backward induction. Based on these values, the same $h$-lookahead policy by which $h$-RTDP acts is evaluated. In the analysis of ADP, we use standard techniques developed for the discounted case in~\cite{bertsekas1996neuro}. From Table~\ref{tab:rtdp vs dp}, we reach the following conclusion: \emph{the asymptotic performance (in terms of regret) of approximate $h$-RTDP is equivalent to that of a corresponding approximate DP algorithm}. Furthermore, it is important to note that the asymptotic error decreases with $\multistep$ for the approximate value updates and approximate abstraction settings for both RTDP and DP algorithms. In these settings, the error is caused by approximation in the value function. By increasing the lookahead horizon $\multistep$, the algorithm uses less such values and relies more on the model which is assumed to be correct. Thus, the algorithm becomes less affected by the value function approximation.


\begin{table}[t]
\begin{center}
\begin{tabular}{|c | c | c  | c | }\hline
   { Setting} & {$h$-RTDP Regret (This work)} & {ADP Regret~\cite{bertsekas1996neuro}} & {UCT} \\ 
  \hline
  Exact~(\ref{sec: mutiple step rtdp})  & $\Olog\big(SH(H\!-\!\multistep)/\multistep\big)$ & 0 & $\Omega(\exp(\exp(H)) )$~\cite{coquelin2007bandit}\\
  \hline
 App. Model~(\ref{sec: appr model}) &  $\Olog\big(SH(H\!-\!\multistep)/\multistep \!+\! \Delta_P K\big)$ & $\Delta_P K$ & N.A \\
 \hline
 App. Value~(\ref{sec: appr value})& $\Olog\big(SH(H\!-\!\multistep)g^\epsilon_{H/h}/\multistep  \!+\!\Delta_V K/\multistep\big)$ &  $ \Delta_V K/\multistep$ & N.A\\
 \hline
 App. Abstraction~(\ref{sec: appr abstractions}) & 
$\Olog\big(S_\phi H(H\!-\!\multistep)/\multistep +\Delta_A K/\multistep\big) $ & $ \Delta_A K/\multistep$ & N.A\\
 \hline
\end{tabular}
\end{center}
\caption{\begin{small} The lookhead horizon is $\multistep$ and the horizon of the MDP is $H$. We denote $g^\epsilon_{H/h}=(1+H\epsilon_V/\multistep)$, $\Delta_P=H(H-1)\epsilon_P$, $\Delta_V = 2H\epsilon_V$, and $\Delta_A=H \epsilon_A$. The table summarizes the regret bounds of the $h$-RTDP settings studied in this work and compares them to those of their corresponding ADP approaches. The performance of ADP is based on standard analysis, supplied in Propositions~\ref{prop: misspecified model bound},~\ref{prop: approximate value updates},~\ref{prop: approximate abstraction} in Appendix~\ref{supp: approximate dp bounds}.\end{small}}
\label{tab:rtdp vs dp}
\end{table}

\paragraph{Conclusions.}
In this paper, we formulated $\multistep$-RTDP, a generalization of RTDP that acts by a lookahead policy, instead of by a 1-step greedy policy, as in RTDP. We analyzed the finite-sample performance of $\multistep$-RTDP in its exact form, as well as in  three approximate settings. The results indicate that $\multistep$-RTDP converges in a very strong sense. Its regret is constant w.r.t. to the number of episodes, unlike in, e.g., reinforcement learning where a lower bound of $\Olog(\sqrt{SAHT})$ exists~\cite{azar2017minimax,jin2018q}. Furthermore, the analysis reveals that the sample complexity of $\multistep$-RTDP  improves by increasing the lookahead horizon $\multistep$ (Remark~\ref{remark: sample compleixty of h rtdp}). Moreover, the asymptotic performance of $\multistep$-RTDP was shown to be equivalent to that of ADP (Table~\ref{tab:rtdp vs dp}), which under no further assumption on the approximation error, is the best we can hope for.

We believe this work opens interesting research venues, such as studying alternatives to the solution of the $\multistep$-greedy policy (see Section~\ref{supp: epsiodic complexity h rtdp}), studying a Receding-Horizon extension of RTDP, RTDP with function approximation, and formulating a Thompson-Sampling version of RTDP, as the standard RTDP is an `optimistic' algorithm. As the analysis developed in this work was shown to be quite generic, we hope that it can assist with answering some of these questions. On the experimental side, more needs to be understood, especially comparing RTDP with MCTS and studying how RTDP can be combined with deep neural networks as the value function approximator.


\section{Broader Impact}

Online planning algorithms, such as $A^*$ and RTDP, have been extensively studied and applied in AI for well over two decades. 
Our work quantifies the benefits of using lookahead-policies in this class of algorithms. Although lookahead-policies have also been widely used in online planning algorithms, their theoretical justification was lacking. Our study sheds light on the benefits of lookahead-policies. Moreover, the results we provide in this paper suggest improved ways for applying lookahead-policies in online planning with benefits when dealing with various types of approximations.
This work opens up the room for practitioners to improve their algorithms and base lookahead policies on solid theoretical ground. 

\section{Acknowledgements}
We thank the reviewers for their helpful comments and feedback.

\bibliographystyle{plain}
\bibliography{citation}

\begin{thebibliography}{10}

\bibitem{abel2017near}
David Abel, D.~Hershkowitz, and Michael Littman.
\newblock Near optimal behavior via approximate state abstraction.
\newblock In {\em Proceedings of the 33rd International Conference on
  International Conference on Machine Learning}, pages 2915--2923, 2016.

\bibitem{azar2017minimax}
Mohammad~Gheshlaghi Azar, Ian Osband, and R{\'e}mi Munos.
\newblock Minimax regret bounds for reinforcement learning.
\newblock In {\em Proceedings of the 34th International Conference on Machine
  Learning-Volume 70}, pages 263--272. JMLR. org, 2017.

\bibitem{barto1995learning}
Andrew Barto, Steven Bradtke, and Satinder Singh.
\newblock Learning to act using real-time dynamic programming.
\newblock {\em Artificial intelligence}, 72(1-2):81--138, 1995.

\bibitem{bertsekas1996neuro}
D.~Bertsekas and J.~Tsitsiklis.
\newblock {\em Neuro-dynamic programming}.
\newblock Athena Scientific, 1996.

\bibitem{bonet2000planning}
Blai Bonet and Hector Geffner.
\newblock Planning with incomplete information as heuristic search in belief
  space.
\newblock In {\em Proceedings of the Fifth International Conference on
  Artificial Intelligence Planning Systems}, pages 52--61. AAAI Press, 2000.

\bibitem{bonet2003labeled}
Blai Bonet and Hector Geffner.
\newblock Labeled rtdp: Improving the convergence of real-time dynamic
  programming.
\newblock In {\em ICAPS}, volume~3, pages 12--21, 2003.

\bibitem{browne2012survey}
Cameron Browne, Edward Powley, Daniel Whitehouse, Simon Lucas, Peter Cowling,
  Philipp Rohlfshagen, Stephen Tavener, Diego Perez, Spyridon Samothrakis, and
  Simon Colton.
\newblock A survey of {M}onte {C}arlo tree search methods.
\newblock {\em IEEE Transactions on Computational Intelligence and AI in
  games}, 4(1):1--43, 2012.

\bibitem{bulitko2006learning}
Vadim Bulitko and Greg Lee.
\newblock Learning in real-time search: A unifying framework.
\newblock {\em Journal of Artificial Intelligence Research}, 25:119--157, 2006.

\bibitem{coquelin2007bandit}
Pierre-Arnaud Coquelin and R{\'e}mi Munos.
\newblock Bandit algorithms for tree search.
\newblock {\em arXiv preprint cs/0703062}, 2007.

\bibitem{dann2017unifying}
Christoph Dann, Tor Lattimore, and Emma Brunskill.
\newblock Unifying pac and regret: Uniform pac bounds for episodic
  reinforcement learning.
\newblock In {\em Advances in Neural Information Processing Systems}, pages
  5713--5723, 2017.

\bibitem{dean1997model}
Thomas Dean, Robert Givan, and Sonia Leach.
\newblock Model reduction techniques for computing approximately optimal
  solutions for {M}arkov decision processes.
\newblock In {\em Proceedings of the 13th conference on Uncertainty in
  artificial intelligence}, pages 124--131, 1997.

\bibitem{dearden1997abstraction}
Richard Dearden and Craig Boutilier.
\newblock Abstraction and approximate decision-theoretic planning.
\newblock {\em Artificial Intelligence}, 89(1-2):219--283, 1997.

\bibitem{efroni2019combine}
Y.~Efroni, G.~Dalal, B.~Scherrer, and S.~Mannor.
\newblock How to combine tree-search methods in reinforcement learning.
\newblock In {\em Proceedings of the AAAI Conference on Artificial
  Intelligence}, pages 3494--3501, 2019.

\bibitem{efroni2018multiple}
Yonathan Efroni, Gal Dalal, Bruno Scherrer, and Shie Mannor.
\newblock Multiple-step greedy policies in approximate and online reinforcement
  learning.
\newblock In {\em Advances in Neural Information Processing Systems}, pages
  5238--5247, 2018.

\bibitem{efroni2019tight}
Yonathan Efroni, Nadav Merlis, Mohammad Ghavamzadeh, and Shie Mannor.
\newblock Tight regret bounds for model-based reinforcement learning with
  greedy policies.
\newblock In {\em Advances in Neural Information Processing Systems}, pages
  12203--12213, 2019.

\bibitem{even2003approximate}
Eyal Even-Dar and Yishay Mansour.
\newblock Approximate equivalence of markov decision processes.
\newblock In {\em Learning Theory and Kernel Machines}, pages 581--594, 2003.

\bibitem{geist2013algorithmic}
Matthieu Geist and Olivier Pietquin.
\newblock Algorithmic survey of parametric value function approximation.
\newblock {\em IEEE Transactions on Neural Networks and Learning Systems},
  24(6):845--867, 2013.

\bibitem{jin2018q}
Chi Jin, Zeyuan Allen-Zhu, Sebastien Bubeck, and Michael~I Jordan.
\newblock Is q-learning provably efficient?
\newblock In {\em Advances in Neural Information Processing Systems}, pages
  4863--4873, 2018.

\bibitem{kearns2002sparse}
Michael Kearns, Yishay Mansour, and Andrew Ng.
\newblock A sparse sampling algorithm for near-optimal planning in large
  {M}arkov decision processes.
\newblock {\em Machine learning}, 49(2-3):193--208, 2002.

\bibitem{kearns2002near}
Michael Kearns and Satinder Singh.
\newblock Near-optimal reinforcement learning in polynomial time.
\newblock {\em Machine learning}, 49(2-3):209--232, 2002.

\bibitem{kocsis2006bandit}
Levente Kocsis and Csaba Szepesv{\'a}ri.
\newblock Bandit based {M}onte-{C}arlo planning.
\newblock In {\em European conference on machine learning}, pages 282--293,
  2006.

\bibitem{kolobov2012lrtdp}
Andrey Kolobov, Daniel~S Weld, et~al.
\newblock Lrtdp versus uct for online probabilistic planning.
\newblock In {\em Twenty-Sixth AAAI Conference on Artificial Intelligence},
  2012.

\bibitem{li2006towards}
L.~Li, T.~Walsh, and M.~Littman.
\newblock Towards a unified theory of state abstraction for {MDP}s.
\newblock In {\em Proceedings of the 9th International Symposium on Artificial
  Intelligence and Mathematics}, pages 531--539, 2006.

\bibitem{mcmahan2005bounded}
Brendan McMahan, Maxim Likhachev, and Geoffrey Gordon.
\newblock Bounded real-time dynamic programming: Rtdp with monotone upper
  bounds and performance guarantees.
\newblock In {\em Proceedings of the 22nd international conference on Machine
  learning}, pages 569--576. ACM, 2005.

\bibitem{munos2007performance}
R{\'e}mi Munos.
\newblock Performance bounds in l\_p-norm for approximate value iteration.
\newblock {\em SIAM journal on control and optimization}, 46(2):541--561, 2007.

\bibitem{munos2014bandits}
R{\'e}mi Munos.
\newblock From bandits to {M}onte-{C}arlo tree search: The optimistic principle
  applied to optimization and planning.
\newblock {\em Foundations and Trends{\textregistered} in Machine Learning},
  7(1):1--129, 2014.

\bibitem{scherrer2012approximate}
Bruno Scherrer, Mohammad Ghavamzadeh, Victor Gabillon, and Matthieu Geist.
\newblock Approximate modified policy iteration.
\newblock In {\em Proceedings of the 29th International Conference on Machine
  Learning}, pages 1207--1214, 2012.

\bibitem{sidford2018variance}
Aaron Sidford, Mengdi Wang, Xian Wu, and Yinyu Ye.
\newblock Variance reduced value iteration and faster algorithms for solving
  {M}arkov decision processes.
\newblock In {\em Proceedings of the 29th Annual ACM-SIAM Symposium on Discrete
  Algorithms}, pages 770--787, 2018.

\bibitem{strehl2006pac}
A.~Strehl, L.~Li, and M.~Littman.
\newblock {PAC} reinforcement learning bounds for {RTDP} and rand-{RTDP}.
\newblock In {\em Proceedings of AAAI workshop on learning for search}, 2006.

\bibitem{strehl2009reinforcement}
Alexander Strehl, Lihong Li, and Michael Littman.
\newblock Reinforcement learning in finite {MDP}s: {PAC} analysis.
\newblock {\em Journal of Machine Learning Research}, 10(Nov):2413--2444, 2009.

\end{thebibliography}


\newpage

\section{Per-Episode Complexity of $h$-RTDP}\label{supp: epsiodic complexity h rtdp}
In this section, we define and analyze the \emph{Forward-Backward DP} by which an $\multistep$-greedy policy can be calculated from a current state $s_t^k$ according to~\eqref{eq: lookahead h greedy preliminaries}. Observe that the algorithm is based on a `local' information, i.e., it does not need access to the entire state space, but to a portion of the state space in the `vicinity' of the current state $s_t^k$. Furthermore, it does not assume prior knowledge on this vicinity.

\subsection{Forward-Backward Dynamic Programming Approach}

\begin{algorithm}
\begin{algorithmic}
\caption{$\multistep$-Forward-Backward DP}
\label{algo: forward backward DP}
    \STATE {\bf Input:}  $s$, transition $p$, reward $r$, lookahead horizon $\multistep$, value at the end of lookahead horizon $\OuterV$
    \STATE $\brc*{S_{t'}(s)}_{t'=1}^{\multistep+1}$ = Forward-Pass($s$,$p$, $\multistep$)
    \STATE action = Backward-Pass($\brc*{S_{t'}(s)}_{t'=1}^{\multistep+1}$, $r$, $p$, $\multistep$,$\OuterV$)
    \STATE {\bf return:} action
\end{algorithmic}
\end{algorithm}

\begin{center}
\begin{minipage}{.45\linewidth}
\vspace{-0.05cm}
\begin{algorithm}[H]
\begin{algorithmic}
\caption{Forward-Pass}
\label{algo: forward pass}
    \STATE {\bf Input:} Starting state $s$, $p$, $\multistep$
    \STATE \bf{Init:} $\mathcal{S}_1=\brc*{s}$, $\;\forall t'\in[h]/ \brc*{1}$, $\;\mathcal{S}_{t'}(s) =\brc*{}$
    \FOR{$t'=2,3,\ldots, h+1$}
        \FOR{$s_{t'-1}\in \mathcal{S}_{t'-1}(s)$}
            \STATE {\color{gray}\# acquire possible next states from $s_{t'-1}$}
            \FOR{$a\in \mathcal{A}$}
                \STATE $\mathcal{S}_{t'}(s) = \mathcal{S}_{t'}(s)\cup \brc*{s' : p (s'\mid s,a)>0}$ 
            \ENDFOR 
        \ENDFOR        
    \ENDFOR
    \STATE {\bf return:} $\brc*{\mathcal{S}_{t'}(s)}_{t'=1}^{h+1}$ 
\end{algorithmic}
\end{algorithm}
\end{minipage}
\hspace{0.5cm}
\begin{minipage}{.5\linewidth}
\begin{algorithm}[H]
\begin{algorithmic}
\caption{Backward-Pass}
\label{algo: backward pass}
    \STATE {\bf Input:} $\brc*{\mathcal{S}_{t'}(s)}_{t'=1}^{\multistep+1}$, $r$, $p$, $\multistep$, $\OuterV$
    \STATE {\color{gray}\# initialize values by arbitrary value $C$}
    \STATE {\bf Init:} $\forall t'\in [h-1]$, $\;\forall s\in \mathcal{S}_{t'}(s)$, $\;V_{t'}(s) = C$ 
    \STATE {\color{gray}\# Assign the value at $t'=h$ to the current value, $V$}.
    \FOR{$s\in \mathcal{S}_{h+1}(s)$}
        \STATE $ V_{\multistep+1}(s)= \OuterV(s)$ 
    \ENDFOR    
    \FOR{$t'=h,h-1,\ldots,2$}
        \FOR{$s\in \mathcal{S}_{t'}(s)$}
            \STATE $V_{t'}(s)= \max_{a} r(s,a)+ p(\cdot \mid s,a) V_{t'+1}$
        \ENDFOR
    \ENDFOR
    \STATE {\bf return:} $\arg\max_a r(s,a) + p(\cdot \mid s,a) V_{2}$
\end{algorithmic}
\end{algorithm}

\end{minipage}
\end{center}

The Forward-Backword DP (Algorithm~\ref{algo: forward backward DP}) approach is built on the following observation: would we known the accessible state space from $s$ in next $\multistep$ time steps we could use Backward Induction (i.e., Value Iteration) on a finite-horizon MDP, with an horizon of $\multistep$, and calculate the optimal policy from $s$. Unfortunately, as we do not assume such a prior knowledge, we have to calculate this set before applying the backward induction step. Thus, Forward-Backword DP first build this set (in the first, `Forward' stage) and later applies standard backward induction (in the `Backward' stage). In Proposition~\ref{proposition: computational complexity of forward pass}, we establish that calculating the set of accessible states can be done efficiently

Let us first analyze the \emph{computational complexity} of Algorithm~\ref{algo: forward backward DP} using the following definitions.  Let $\mathcal{S}_{t'}(s)$ be the set of reachable states from state $s$ in $t'$ times steps, formally, $$\mathcal{S}_{t'}(s) = \brc*{s' \mid \exists \pi: p^\pi(s_{t'}=s'\mid s_0=s,\pi)> 0},$$ where $p^\pi(s_{t'}=s'\mid s_0=s,\pi) = \E[\ind\brc*{s_{t'}=s'}\mid s_0=s,\pi]$. The cardinality of this set is denoted by $\left| \mathcal{S}_{t'}(s) \right|$. let $\mathcal{N} \eqdef \max_{s}\left| \mathcal{S}_{2}(s) \right|$ be the maximal number of accessible states in 1-step (maximal `nearest neighbors' from any state). Furthermore, let the total reachable states in $\multistep$ time steps from state $s$ be $S^{Tot}_\multistep(s) = \sum_{t'=1}^\multistep   \left| \mathcal{S}_{t'}(s) \right|$. When $S^{Tot}_\multistep(s)$ is small, as we establish in this section, local search up to an horizon of $h$ can be done efficiently with the Forward-Backward DP, unlike the exhaustive search approach.

Based on the above definitions we analyze the computational complexity of Forward-Backward DP starting from the Forward-Pass stage.


\begin{restatable}[Computation Cost of Forward-Pass]{proposition}{fbdpCost}
\label{proposition: computational complexity of forward pass}
The Forward-Pass stage of FB-DP can be implemented with the computation cost of $O\big(\mathcal{N}A S^{Tot}_{\multistep}(s)\big)$.
\end{restatable}

\begin{proof}
Calculating the set $\brc*{s' : p (s'\mid s,a)>0}$ cost is upper bounded by $O(\mathcal{N})$ as we need to enumerate at most all possible $O(\mathcal{N})$ next-states. We assume that $\mathcal{S}_{t'} = \mathcal{S}_{t'}(s) \cup \brc*{s' : p (s'\mid s,a)>0}$ can be done by $O(\mathcal{N})$, e.g., when using a hash-table for saving $\mathcal{S}_{t'}$ in memory. As we need to repeat this operation $A$ times, the complexity for each $t'\in \brc*{2,3,.,,\multistep+1}$ is upper bounded by $O(\mathcal{N}A\abs{S_{t'-1}(s)})$. Summing over all $t'$ we get that the computational complexity of the Forward pass is upper bounded by
\begin{align*}
    O\br*{\mathcal{N}A\sum_{t'=2}^{h+1} \abs{S_{t'-1}(s)}} =O\br*{\mathcal{N}A\abs*{S^{Tot}_{\multistep}(s)}},
\end{align*}
where the second equality holds by definition of total number of accessible states in $\multistep$ time steps.
\end{proof}

The computational complexity of the backward passage is the computational complexity of Backward Induction, which is the total number of states in which actions can be taken times the number of actions per state, i.e.,
\begin{align}
    O(A \mathcal{N} S^{Tot}_{\multistep}(s)).\label{eq: complexity of backward pass},
\end{align}
where the origin  of the factor $\mathcal{N}$ is due to the need to calculate the sum $\sum_{s'}p(s'\mid s,a)V(s')$ for each $(s,a)$ pair, and, by definition, this sum contain at most $\mathcal{N}$ elements.

Using Proposition~\ref{proposition: computational complexity of forward pass} and \eqref{eq: complexity of backward pass} we get that for every $t\in [H]$, the computational complexity of calculating an $h$-lookahead policy from a state $s$ using the Forward-Backward DP is bounded by,
\begin{align*}
    O( (\mathcal{N}A +\mathcal{N} A) S^{Tot}_{\multistep}(s)) =  O( \mathcal{N} A S^{Tot}_{\multistep}),
\end{align*}
where the last relation holds by definition, $S^{Tot}_{\multistep} = \max_{s} S^{Tot}_{\multistep}(s).$


Finally, the \emph{space complexity} of Forward-Backward DP is the space required the save in memory the possible visited states  in $\multistep$ time steps (their identity in the Forward-Pass and their values in the Backward-Pass). By definition it is at most $O(\multistep S_\multistep).$

\newpage



\section{Real Time Dynamic Programming with Lookahead}\label{supp: multistep rtdp}

This section contains the full proofs of all the results of Section~\ref{sec: mutiple step rtdp} in chronological order.

\multistepRtdpProperties*
\begin{proof}
Both claims are proven using induction. 

\paragraph{{\em (i)}} 
Let $n\in\{0\}\cup[\frac{H}{h}]$. By the initialization, $\forall s,n,\  V^*_{nh+1}(s) \leq V^0_{nh+1}(s)$. Assume the claim holds for the first $(k-1)$ episodes. Let $s_t^{k}$ be the state of the algorithm at a time step $t$ of the $k$'th episode at which a value update takes place, i.e.,~$t=nh+1$, for some $n\in\{0\}\cup[\frac{H}{h}]$. By the value update of Algorithm~\ref{algo: multi step RTDP} and~\eqref{eq:multistep bellman}, we have
\begin{align*}
     \bar{V}_t^{k}(s_t^{k}) =  (T^h \bar{V}^{k-1}_{h_c})(s_t^k) = (T^h \bar{V}^{k-1}_{t+h})(s_t^k) \geq (T^h V^*_{t+h})(s_t^k) = V^*_t(s_t^k).
\end{align*}
The inequality holds by the induction hypothesis and the monotonicity of $T^h$, a consequence of the monotonicity of $T$, the optimal Bellman operator~\cite{bertsekas1996neuro}. The last equality holds by the fact that the recursion is satisfied by the optimal value function~\eqref{eq:multistep bellman}. Thus, the induction step is proven for the first claim.

\paragraph{{\em (ii)}} 
Let $n\in\{0\}\cup[\frac{H}{h}]$ and $t=nh+1$ be a time step in which a value update takes place. To prove the base case, we use the optimistic initialization. Let $s^1_t$ be the state of the algorithm in the $t$'th time step of the first episode. By the update rule, we have
\begin{align*}
    \bar{V}^1_t(s^1_t) &= (T^h \bar{V}^0_{t+h})(s_t^0) = \max_{a_0,\ldots,a_{h-1}}\E\left[\sum_{t'=0}^{h-1}r(s_{t'},a_{t'}) +\bar{V}^0_{t+h}(s_h)\mid s_0=s_t^0\right] \\
    &\stackrel{\text{(a)}}{\le} h+H-(t+h-1) = H-(t-1) \stackrel{\text{(b)}}{=}\bar{V}^0_t(s^1_t).
\end{align*}
{\bf (a)} holds since $r(s,a)\in [0,1]$ and by the optimistic initialization. \\ 
{\bf (b)} observe that $H-(t-1)$ is the value of the optimistic initialization. \\

Assume that the claim holds for the first $(k-1)$ episodes. Let $s_t^{k}$ be the state of the algorithm at a time step $t$ of the $k$'th episode at which a value update takes place, i.e.,~$t=nh+1$, for some $n\in\{0\}\cup[\frac{H}{h}]$. By the value update rule of Algorithm~\ref{algo: multi step RTDP}, we have $\bar{V}_t^{k}(s_t^{k}) = (T^h \bar{V}^{k-1}_{h_c})(s_t^k) = (T^h \bar{V}^{k-1}_{t+h})(s_t^k)$. If $s_t^k$ was previously updated, let $\bar{k}$ be the last episode in which the update occurred, i.e.,~$\bar{V}^{\bar k}_t(s_t^k)=(T^h\bar{V}^{\bar{k}-1}_{t+h})(s_t^k)=\bar{V}^{k-1}_t(s_t^k)$. By the induction hypothesis, we have that $\forall s,t,\ \bar{V}^{\bar{k}-1}_t(s)\geq \bar{V}^{k-1}_t(s)$. Using the monotonicity of $T^h$, we may write 
\begin{align*}
    \bar{V}_t^{k}(s_t^{k}) = (T^h \bar{V}^{k-1}_{t+h})(s_t^k) \leq (T^h \bar{V}^{\bar{k}-1}_{t+h})(s_t^k) = \bar{V}^{k-1}_t(s_t^k).
\end{align*}
Thus, $\bar{V}_t^{k}(s_t^{k}) \leq \bar{V}^{k-1}(s_t^k)$ and the induction step is proved. If $s_t^k$ was not previously updated, then $\bar{V}_t^{k-1}(s_t^{k})=\bar{V}_t^{0}(s_t^{k})$. In this case, the induction hypothesis implies that $\forall s', \bar{V}_{t+h}^{k-1}(s')\le \bar{V}_{t+h}^{0}(s')$ and the result is proven similarly to the base case.
\end{proof}

\MultistepRdtpExpectedValueUpdate*

\begin{proof}
Let $n\in\{0\}\cup[\frac{H}{h}]$ and $t=n\multistep+1$ be a time step in which a value update takes place. By the definition of the update rule, the following holds for the value update at the visited state $s_t^k$:
\begin{align}
     \bar{V}_t^{k}(s_t^k)  &= (T^h \bar{V}_{t+h}^{k-1})(s_t^k) \label{eq: first relation before conditional exp}\\
                           &= (T^{\pi_{k}(t)}\cdot\cdot\cdot T^{\pi_{k}(t+h-1)}\bar{V}_{t+h}^{k-1})(s_t^k) = \E\brs*{\sum_{t'=t}^{t+h-1}r(s_{t'},a_{t'}) + \bar{V}^{k-1}_{t+h}(s_{t+h}) \mid \pi_k,s_t=s_t^k}\nonumber \\
                           &\stackrel{\text{(a)}}{=} \E\brs*{\sum_{t'=t}^{t+h-1}r(s^k_{t'},a^k_{t'}) + \bar{V}^{k-1}_{t+h}(s_{t+h}^k) \mid \mathcal{F}_{k-1},s_t^k}. \label{eq: model updates to filtration}
\end{align}
{\bf (a)} We prove this passage for each reward element $r(s_{t'},a_{t'})$ in the expectation. The proof for the expectation of $\bar{V}^{k-1}_{t+h}(s_{t+h})$ follows in a similar manner. Since the first expectation is w.r.t.~the dynamics of the true model, a consequence of updating by the true model, for any $t'\geq t$, we may write
\begin{align*}
    \E\brs*{r(s_{t'},a_{t'}) \mid \pi_k,s_t=s_t^k} &= \sum_{s_{t'}\in\mathcal{S}}p(s_{t'}\mid s_t=s_t^k,\pi_k)r(s_{t'},\pi_k(s_{t'},t'))\\
    &\stackrel{\text{(i)}}{=}\sum_{s_{t'}^k\in\mathcal{S}}p(s_{t'}^k\mid s_t^k,\mathcal{F}_{k-1})r(s^k_{t'},\pi_k(s_{t'}^k,t')) = \E\brs*{r(s^k_{t'},a^k_{t'}) \mid \mathcal{F}_{k-1},s_t^k},
\end{align*}
where $p(s_{t'}\mid s_t^k,\pi_k)$ is the probability of starting at state $s_t^k$, following $\pi_k$, and reaching state $s_{t'}$ in $t'-t$ steps.  

\noindent
{\bf (i)} We use the fact that $p(s_{t'}\mid s_t^k,\pi_k) = p(s_{t'}^k\mid s_t^k,\mathcal{F}_{k-1})$, in words, given the policy $\pi_k$ (which is $\mathcal{F}_{k-1}$ measurable) and $s_t^k$ the probability for a state $s_{t'}^k$ with $t'\geq t$ is independent of the rest of the history. 

Now that we proved {\bf (a)}, we take the conditional expectation of~\eqref{eq: first relation before conditional exp} w.r.t.~$\mathcal{F}_{k-1}$ and use the tower rule to obtain
\begin{align}
\label{eq:temp0}
    \E\brs*{\bar{V}_t^{k}(s_t^k)\mid \mathcal{F}_{k-1}} = \E\brs*{\sum_{t'=t}^{t+h-1}r(s^k_{t'},a^k_{t'}) + \bar{V}^{k-1}_{t+h}(s_{t+h}^k) \mid \mathcal{F}_{k-1}}.
\end{align}
Summing~\eqref{eq:temp0} for all $n\in\{0\}\cup[\frac{H}{\multistep}]$, and using the linearity of expectation and the fact that $\bar{V}^{k}_{H+1}(s)=0$ for all $s,k$, we have
\begin{align}
    &\sum_{n=0}^{\frac{H}{\multistep}-1 } \E\brs*{\bar{V}_{n\multistep+1}^{k}(s_{nh+1}^k) \mid \F_{k-1}} 
      = \E\brs*{\sum_{t=1}^H r(s_t^k,a_t^k)\mid \F_{k-1}} + \sum_{n=1}^{\frac{H}{\multistep}-1 } \E\brs*{ \bar{V}_{n\multistep+1}^{k-1}(s_{n\multistep+1}^k) \mid \F_{k-1}} \nonumber\\
     \iff & \bar{V}^{k}_{1}(s_1^k)+\sum_{n=1}^{\frac{H}{\multistep}-1 } \E\brs*{\bar{V}_{n\multistep+1}^{k}(s_t^k) \mid \F_{k-1}} = \E\brs*{\sum_{t=1}^H r(s_t^k,a_t^k)\mid \F_{k-1}} + \sum_{n=1}^{\frac{H}{\multistep}-1 } \E\brs*{ \bar{V}_{n\multistep+1}^{k-1}(s_{n\multistep+1}^k) \mid \F_{k-1}} \nonumber\\
     \iff & \bar{V}^{k}_{1}(s_1^k)+\sum_{n=1}^{\frac{H}{\multistep}-1 } \E\brs*{\bar{V}_{n\multistep+1}^{k}(s_t^k) \mid \F_{k-1}} = V^{\pi_k}(s_1^k) + \sum_{n=1}^{\frac{H}{\multistep}-1 } \E\brs*{ \bar{V}_{n\multistep+1}^{k-1}(s_{n\multistep+1}^k) \mid \F_{k-1}} \nonumber\\
     \iff & \bar{V}^{k}_{1}(s_1^k) - V^{\pi_k}(s_1^k) = \sum_{n=1}^{\frac{H}{\multistep}-1 } \E\brs*{ \bar{V}_{n\multistep+1}^{k-1}(s_{nh+1}^k) -\bar{V}_{n\multistep+1}^{k}(s_{nh+1}^k) \mid \F_{k-1}}. \label{eq: on trajecotry regret excat}
\end{align}
The second line holds by the fact that $s_1^k$ is measurable w.r.t.~$\mathcal{F}_{k-1}$ The third line holds since $$V^{\pi_k}_1(s_1^k) = \E\brs*{\sum_{t=1}^H r(s_t^k,a_t^k)\mid s_1^k,\pi_k } =\E\brs*{\sum_{t=1}^H r(s_t^k,a_t^k)\mid \F_{k-1}}.$$

Applying Lemma~\ref{lemma: On trajectory regret to Uniform regret} from Appendix~\ref{sec:useful-lemma} with $g_t^k=\bar V^k_t$ for $t=n\multistep+1$, we obtain
\begin{align*}
    \eqref{eq: on trajecotry regret excat} = \sum_{n=1}^{\frac{H}{\multistep}-1}\sum_{s\in\mathcal{S}} \bar{V}^{k-1}_{n\multistep+1}(s) - \E[\bar{V}^{k}_{n\multistep+1}(s)\mid \F_{k-1}],
\end{align*}
which concludes the proof. Note that the update of $\bar V^k_t$ occurs only at the visited state $s_t^k$ and the update rule uses $\bar V^{k-1}_{t+h}$, i.e.,~it is measurable w.r.t.~$\mathcal{F}_{k-1}$, and thus, it is valid to apply Lemma~\ref{lemma: On trajectory regret to Uniform regret}.
\end{proof}

\TheoremRegretMultistepRTDP*

\begin{proof}
We start by proving {\bf Claim (1)}. We know that the following bounds hold on the regret:
\begin{align}
     \Regret(K) &\eqdef \sum_{k=1}^K V^*_1(s^k_1)- V^{\pi_k}_1(s^k_1) \stackrel{\text{(a)}}{\leq} \sum_{k=1}^K \bar{V}_1^{k}(s^k_1)- V^{\pi_k}_1(s^k_1) \nonumber \\ 
     &\stackrel{\text{(b)}}{=} \sum_{k=1}^K \sum_{n=1}^{\frac{H}{\multistep}-1}\sum_{s\in\mathcal S} \bar{V}^{k-1}_{n\multistep+1}(s) - \E[\bar{V}^{k}_{n\multistep+1}(s)\mid \F_{k-1}]. \label{eq: regret bound multiple step rtdp}
\end{align}
{\bf (a)} is by the optimism of the value function (Lemma~\ref{lemma:multistep rtdp properties}), and {\bf (b)} is by Lemma~\ref{lemma: multistep RTDP expected value difference}. 

We would like to show that~\eqref{eq: regret bound multiple step rtdp} is the regret of a Decreasing Bounded Process (DBP). We start by defining
\begin{align}
    X_k \eqdef \sum_{n=1}^{\frac{H}{\multistep}-1}\sum_{s\in\mathcal S} \bar{V}^{k}_{n\multistep+1}(s). \label{def: DBP multiple rtdp}
\end{align}
We now prove that $\brc*{X_k}_{k\geq 0}$ is a DBP. Note that $\brc*{X_k}_{k\geq 0}$

\begin{enumerate}
    \item is decreasing, since $\forall s,t,\;\bar{V}^k_t(s)\leq \bar{V}^{k-1}_t(s)$ by Lemma~\ref{lemma:multistep rtdp properties}, and thus, their sum is also decreasing, and 
    \item is bounded since $\forall s,t\;\bar{V}^k_t(s)\geq V_t^*(s) \geq 0$ by Lemma~\ref{lemma:multistep rtdp properties}, and thus, the sum is bounded from below by $0$.
\end{enumerate}

We can show that the initial value $X_0$ is also bounded as
\begin{align*}
    X_0 = \sum_{n=1}^{\frac{H}{\multistep}-1}\sum_{s\in\mathcal S} \bar{V}^{0}_{n\multistep+1}(s)\leq \sum_{n=1}^{\frac{H}{\multistep}-1}\sum_{s\in\mathcal S} H= \frac{SH(H-\multistep)}{\multistep}.
\end{align*}
Using the linearity of expectation and the definition~\eqref{def: DBP multiple rtdp}, we observe that $\eqref{eq: regret bound multiple step rtdp}$ can be written as
\begin{align*}
    \Regret(K) \leq \eqref{eq: regret bound multiple step rtdp} = \sum_{k=1}^K X_{k-1} - \E[X_k\mid \mathcal{F}_{k-1}],
\end{align*}
which is regret of a DBP. Applying the bound on the regret of a DBP, Theorem~\ref{theorem: regret of decreasing bounded process}, we conclude the proof of the first claim. \\

\noindent
We now prove {\bf Claim (2)}. Here we use a different technique than the one used in~\cite{efroni2019tight}. The technique allows us to prove uniform-PAC bounds for the approximate versions of $h$-RTDP described in Section~\ref{sec: approximate mutiple step rtdp}. For these approximate versions, the uniform-PAC result is not a corollary of the regret bound and more careful analysis should be used.

For all $\epsilon>0$, the following relations hold:
\begin{align}
    \ind\brc*{V_1^*(s^k_1)-V_1^{\pi_k}(s^k_1) \geq \epsilon} \epsilon &\stackrel{\text{(a)}}{\leq}  \ind\brc*{\bar{V}_1^{k}(s^k_1)-V_1^{\pi_k}(s^k_1) \geq \epsilon} \epsilon \nonumber\\
    &\stackrel{\text{(b)}}{\leq} \ind\brc*{\bar{V}_1^{k}(s^k_1)-V_1^{\pi_k}(s^k_1) \geq \epsilon} \br*{\bar{V}_1^{k}(s^k_1)-V_1^{\pi_k}(s^k_1)}\nonumber\\
    &\stackrel{\text{(c)}}{=} \ind\brc*{\bar{V}_1^{k}(s^k_1)-V_1^{\pi_k}(s^k_1) \geq \epsilon}\br*{\sum_{n=1}^{\frac{H}{\multistep}-1}\sum_{s\in\mathcal S} \bar{V}^{k-1}_{n\multistep+1}(s) - \E[\bar{V}^{k}_{n\multistep+1}(s)\mid \F_{k-1}]}\nonumber\\
    &\stackrel{\text{(d)}}{=} \ind\brc*{\bar{V}_1^{k}(s^k_1)-V_1^{\pi_k}(s^k_1) \geq \epsilon}\br*{X_{k-1} -\E[X_{k}\mid \mathcal{F}_{k-1}] }. \label{eq: central pac multiple step exact}
\end{align}
{\bf (a)} holds since for all $t,s$, $\bar{V}^{k}_t(s)\geq V^*_t(s)$ by Lemma~\ref{lemma:multistep rtdp properties}. {\bf (b)} holds by the indicator function. {\bf (c)} holds by Lemma~\ref{lemma: multistep RTDP expected value difference}. {\bf (d)} holds by the definition of $X_k$ from~\eqref{def: DBP multiple rtdp} and the linearity of expectation.

Let define $N_\epsilon(K) = \sum_{k=1}^K \ind\brc*{V_1^{*}(s^k_1)-V_1^{\pi_k}(s^k_1) \geq \epsilon}$ as the  number of times $V^*_1(s_1^k)-V^{\pi_k}_1(s_1^k) \geq \epsilon$ at the first $K$ episodes. For all $\epsilon>0$, we may write
\begin{align*}
    N_\epsilon(K) \epsilon &\stackrel{\text{(a)}}{=} \sum_{k=1}^K \ind\brc*{V_1^{*}(s^k_1)-V_1^{\pi_k}(s^k_1) \geq \epsilon} \epsilon \stackrel{\text{(b)}}{\leq} \sum_{k=1}^K \ind\brc*{\bar{V}_1^{k}(s^k_1)-V_1^{\pi_k}(s^k_1) \geq \epsilon}\br*{X_{k-1} -\E[X_{k}\mid \mathcal{F}_{k-1}] }\\
    &\stackrel{\text{(c)}}{\leq} \sum_{k=1}^K X_{k-1} -\E[X_{k}\mid \mathcal{F}_{k-1}],
\end{align*}
{\bf (a)} holds by the definition of $N_\epsilon(K)$. {\bf (b)} follows from~\eqref{eq: central pac multiple step exact}. {\bf (c)} holds because $\brc*{X_{k}}_{k\geq 0}$ is a DBP, and thus, ${X_{k-1} - \E[X_{k}\mid \mathcal{F}_{k-1}] \geq 0}$ a.s. Therefore, the following relation holds:
\begin{align*}
     \brc*{\forall K>0: \sum_{k=1}^K X_{k-1} -\E[X_{k}\mid \mathcal{F}_{k-1}] \leq \frac{9SH(H-\multistep)}{\multistep} \ln\frac{3}{\delta} } \subseteq \brc*{\forall \epsilon>0: N_\epsilon(K) \epsilon \leq \frac{9SH(H-\multistep)}{\multistep} \ln\frac{3}{\delta} },
\end{align*}
from which we obtain that for any $K>0$,
\begin{align*}
    \Pr\br*{\forall \epsilon>0: N_\epsilon(K) \epsilon \leq \frac{9SH(H-\multistep)}{\multistep} \ln\frac{3}{\delta} } \geq \Pr\br*{\forall K>0: \sum_{k=1}^K X_{k-1} -\E[X_{k}\mid \mathcal{F}_{k-1}] \leq \frac{9SH(H\multistep)}{\multistep} \ln\frac{3}{\delta} } \stackrel{\text{(a)}}{\geq} 1- \delta. 
\end{align*}
{\bf (a)} holds because of the bound on the regret of DBP (see Theorem~\ref{theorem: regret of decreasing bounded process}). Equivalently, for any $K>0$, 
\begin{align}
    \Pr\br*{\exists \epsilon>0: N_\epsilon(K) \epsilon \geq \frac{9SH(H-\multistep)}{\multistep} \ln\frac{3}{\delta} } \leq \delta.\label{eq: final pac bound multistep rtdp}
\end{align}
Note that for all $\epsilon>0$, $K_1\geq K_2$, $\ind\brc*{N_\epsilon(K_2)\epsilon \geq C} = 1$ implies $\ind\brc*{N_\epsilon(K_1)\epsilon \geq C} = 1$, and thus, $\ind\brc*{N_\epsilon(K)\epsilon \geq C} \leq \lim_{K\rightarrow \infty} \ind\brc*{N_\epsilon(K)\epsilon \geq C}$. Furthermore, $\ind\brc*{N_\epsilon(K)\epsilon \geq C}\geq 0$ by definition. Thus, we can apply the Monotone Convergence Theorem to conclude the proof:
\begin{align*}
    &\Pr\br*{ \exists \epsilon>0: N_\epsilon \epsilon \geq \frac{9SH(H-\multistep)}{\multistep} \ln\frac{3}{\delta} }=\Pr\br*{\lim_{K\rightarrow \infty} \brc*{\exists \epsilon>0: N_\epsilon(K) \epsilon \geq \frac{9SH(H-\multistep)}{\multistep} \ln\frac{3}{\delta} }}\\
    &=\E\brs*{\lim_{K\rightarrow \infty} \ind\brc*{\exists \epsilon>0: N_\epsilon(K) \epsilon \geq \frac{9SH(H-\multistep)}{\multistep} \ln\frac{3}{\delta} }} \stackrel{\text{(a)}}{=} \lim_{K\rightarrow \infty}\E\brs*{\ind\brc*{\exists \epsilon>0: N_\epsilon(K) \epsilon \geq \frac{9SH(H-\multistep)}{\multistep} \ln\frac{3}{\delta}}}\\
    &= \lim_{K\rightarrow \infty}\Pr\br*{ \exists \epsilon>0: N_\epsilon(K) \epsilon \geq \frac{9SH(H-\multistep)}{\multistep} \ln\frac{3}{\delta} } \stackrel{\text{(b)}}{\leq} \delta.
\end{align*}
{\bf (a)} is by the Monotone Convergence Theorem by which $\E[\lim_{k\rightarrow\infty} X_k] =\lim_{k\rightarrow\infty}\E[ X_k]$, for $X_k\geq 0$ and $X_k\leq \lim_{k\rightarrow \infty} X_k$. {\bf (b)} holds by~\eqref{eq: final pac bound multistep rtdp}.
%
%
\end{proof}


\newpage
\section{$\multistep$-RTDP with Approximate Model}\label{supp: multistep rtdp approximate model}

\begin{algorithm}
\begin{algorithmic}
\caption{$h$-RTDP with Approximate Model ($h$-RTDP-AM)}
\label{algo: multi step approx model}
    \STATE init: $\forall s\in \mathcal S,\; \forall n\in \{0\}\cup[\frac{H}{\multistep}],\; \bar{V}^0_{n \multistep +1}(s)=H-n\multistep$
    \FOR{$k\in[K]$}
        \STATE Initialize $s^k_1$
        \FOR{$t\in[H]$}
            \IF{$(t-1)  \mod \multistep == 0 $}
                \STATE $h_c = t + h$
                \STATE $\bar{V}^{k}_{t}(s_t^k) = \hat{T}^{\multistep}\bar{V}^{k-1}_{h_c}(s_t^k)$
            \ENDIF
            \STATE $a_t^k\in \arg\max_a r(s_t^k,a) + \hat{p}(\cdot|s_t^k,a) \hat{T}^{h_c-t-1}\bar{V}^{k-1}_{h_c}$
            \STATE Act with $a_t^k$ and observe $s_{t+1}^k\sim p(\cdot \mid s_t^k,a_t^k)$ 
        \ENDFOR
    \ENDFOR
\end{algorithmic}
\end{algorithm}

Algorithm~\ref{algo: multi step approx model} contains the pseudocode of $h$-RTDP with approximate model. The algorithm is exactly the same as $h$-RTDP (Algorithm~\ref{algo: multi step RTDP}) with the model $p$ and optimal Bellman operator $T$ replaced by their approximations $\hat p$ and $\hat T$. Meaning, $h$-RTDP is agnostic whether it uses the true or approximate model.

We now provide the full proofs of all results in Section~\ref{sec: appr model} in their chronological order. We use the notation $\E_{\hat P}$ to denote expectation w.r.t. the approximate model, i.e., w.r.t. the dynamics $\hat{p}(s'\mid s,a)$ instead according to $p(s'\mid s,a)$.

\begin{restatable}{lemma}{rtdpApproximateModelProperties}
\label{lemma: approximate model properties}
For all $s\!\in\! \mathcal{S}$, $n\!\in\! \{0\}\cup[\frac{H}{\multistep}]$, and $k\in[K]$:
\begin{enumerate}[label=(\roman*)]
    \item Bounded / Optimism: $\hat{V}^*_{n\multistep +1}(s) \leq \bar{V}^k_{n\multistep +1}(s)$.
    \item Non-Increasing: $\bar{V}^{k}_{n\multistep +1}(s) \leq \bar{V}^{k-1}_{n\multistep +1}(s)$.
\end{enumerate}
\end{restatable}

\begin{proof}
Both claims are proven using induction. 

\paragraph{{\em (i)}} 
Let $n\in[0,\frac{H}{h}-1]$ and denote $\hat{T}, \hat V^*$ as the optimal Bellman operators and optimal value of the approximate MDP $(\mathcal{S},\mathcal{A},\hat{p},r,H)$. See that they satisfy usual Bellman equation~\ref{eq:multistep bellman}. 

By the initialization, $\forall s,t,\  \hat{V}^*_{1+hn}(s) \leq V^0_{1+hn}(s)$. Assume the claim holds for $k-1$ episodes. Let $s_t^{k}$ be the state the algorithm is at in the $t= 1+hn$ time step of the $k$'th episode, i.e., at a time step in which a value update is taking place. By the value update of Algorithm~\ref{algo: multi step approx model},
\begin{align*}
     \bar{V}_t^{k}(s_t^{k}) &=  (\hat{T}^h \bar{V}_{t+h})(s_t^k) \geq (\hat{T}^h \hat{V}^*_{t+h})(s_t^k) = \hat V^*_t(s_t^k).
\end{align*}
The second relation holds by the induction hypothesis and the monotonicity of $\hat{T}^h$, a consequnce of the monotonicity of $\hat{T}$, the optimal Bellman operator~\cite{bertsekas1996neuro}. The third relation holds by the recursion satisfied by the optimal value function~\eqref{eq:multistep bellman}. Thus, the induction step is proven for the first claim.

\paragraph{{\em (ii)}} 
Let $n\in[0,\frac{H}{\multistep}-1]$ and let $t= 1+\multistep n$ be a time step in which a value update is taking place. To prove the base case of the second claim we use the optimistic initialization. Let $s^1_t$ be the state the algorithm is at in the $t$'th time step of the first episode. By the update rule,
\begin{align*}
    \bar{V}^1_t(s^1_t)&= \; (\hat{T}^h \bar{V}^0_{t+h})(s_t^0) \\
    & \stackrel{(1)}{=} \max_{\pi_0,\pi_1,..,\pi_{h-1}} \; \E_{\hat{P'}}[\sum_{t'=0}^{h-1}r(s_t',\pi_{t'}(s_t')) +\bar{V}^0_{t+h}(s_h)\mid s_0=s_t^0] \\
    &\stackrel{(2)}{\le} h+H-(t+h-1) = H-(t-1) \stackrel{(3)}{=}\bar{V}^0_t(s^1_t).
\end{align*}
Relation $(1)$ is by the update rule (see Algorithm~\ref{algo: multi step approx model}), when the expectation is taken place w.r.t. the approximate model $\hat{P}$. Relation $(2)$ holds since $r(s,a)\in [0,1]$ and and by the optimistic initialization (see that for $t$ the values at times step $t+h$ were not updated and keep their initial value).  For $(3)$ observe that $H-(t-1)$ is the value of the optimistic initialization.

Assume the second claim holds for $k-1$ episodes. Let $s_t^{k}$ be the state that the algorithm is at in the $t$'th time step of the $k$'th episode. Again, assume that $t= 1+h n$, a time step in which a value update is being done. By the value update of Algorithm \ref{algo: multi step approx model}, we have
\begin{align*}
     \bar{V}_t^{k}(s_t^{k}) &= (\hat{T}^h \bar{V}^{k-1}_{t+h})(s_t^k).
\end{align*}
If $s_t^k$ was previously updated, let  $\bar{k}$ be the previous episode in which the update occured. By the induction hypothesis, we have that $\forall s,t,\ \bar{V}^{\bar{k}}_t(s)\geq \bar{V}^{k-1}_t(s)$. Using the monotonicity of $T^h$ (due to the monotonicity of the Bellman operator), 
\begin{align*}
    & (\hat{T}^h \bar{V}^{k-1}_{t+h})(s_t^k) \leq (\hat{T}^h \bar{V}^{\bar{k}}_{t+h})(s_t^k) = \bar{V}_t^{k-1}(s_t^k).
\end{align*}
Thus, $\bar{V}_t^{k}(s_t^{k}) \leq \bar{V}^{k-1}(s_t^k)$ and the induction step is proved. If $s_t^k$ was not previously updated, then $\bar{V}_t^{k-1}(s_t^{k})=\bar{V}_t^{0}(s_t^{k})$. In this case, the induction hypothesis implies that $\forall s', \bar{V}_{t+h}^{k-1}(s')\le \bar{V}_{t+h}^{0}(s')$ and the result is proven similarly to the base case.
\end{proof}

\begin{restatable}{lemma}{rdtpApproximateModelExpectedValueUpdate}
\label{lemma: RTDP approximate modle expected value difference}
The expected cumulative value update at the $k$'th episode of $h$-RTDP-AM satisfies the following relation:
\begin{align*}
    &\bar{V}_1^{k}(s^k_1)-V_1^{\pi_k}(s^k_1) = \frac{H(H-1)}{2}\epsilon_P \\
     &\quad\quad\quad +\sum_{n=1}^{\frac{H}{\multistep}-1}\sum_{s\in\mathcal{S}} \bar{V}^{k-1}_{n\multistep+1}(s) - \E[\bar{V}^{k}_{n\multistep+1}(s)\mid \F_{k-1}] .
\end{align*}
\vspace{-0.3cm}
\end{restatable}

\begin{proof}
Let $n\in[0,\frac{H}{\multistep}-1]$ and let $t= 1+\multistep n$ be a time step in which a value update is taking place. By the definition of the update rule, the following holds for the update at the visited state $s_t^k$:

\begin{align}
     &\bar{V}_t^{k}(s_t^k)  = (\hat{T}^\multistep \bar{V}_{t+\multistep}^{k-1})(s_t^k) \label{eq: approximate model first relation theorem}\\
     &= (\hat{T}^{\pi_{k}(t)}\cdot\cdot\cdot \hat{T}^{\pi_{k}(t+\multistep-1)}\bar{V}_{t+\multistep}^{k-1})(s_t^k)\nonumber\\
                           &= \E_{P'}\brs*{\sum_{t'=t}^{t+\multistep-1}r(s^k_{t'},a^k_{t'}) + \bar{V}^{k-1}_{t+\multistep}(s_{t+\multistep}^k) \mid \pi_k,s_t^k} \nonumber \\
                          &= \E \brs*{\sum_{t'=t}^{t+\multistep-1}r(s^k_{t'},a^k_{t'}) + \bar{V}^{k-1}_{t+\multistep}(s_{t+\multistep}^k) \mid \pi_k,s_t^k}\nonumber\\
                          &\quad + \E_{P'}\brs*{\sum_{t'=t}^{t+\multistep-1}r(s^k_{t'},a^k_{t'}) + \bar{V}^{k-1}_{t+\multistep}(s_{t+\multistep}^k) \mid \pi_k,s_t^k} - \E \brs*{\sum_{t'=t}^{t+\multistep-1}r(s^k_{t'},a^k_{t'}) + \bar{V}^{k-1}_{t+\multistep}(s_{t+\multistep}^k) \mid \pi_k,s_t^k}\nonumber\\
                          &= \E \brs*{\sum_{t'=t}^{t+\multistep-1}r(s^k_{t'},a^k_{t'}) + \bar{V}^{k-1}_{t+\multistep}(s_{t+\multistep}^k) \mid \pi_k,s_t^k}\nonumber\\
                          &\quad + \sum_{t'=t}^{t+\multistep - 1}\sum_{s_{t'}} \br*{P^{\pi_k}(s_{t'} \mid s_t^k) - \hat{P}^{\pi_k}(s_{t'} \mid s_t^k)}r(s^k_{t'},a^k_{t'}) + \sum_{s_{t+\multistep}} \br*{P^{\pi_k}(s_{t+\multistep} \mid s_t^k) - \hat{P}^{\pi_k}(s_{t+\multistep} \mid s_t^k)}\bar{V}^{k-1}_{t+\multistep}(s_{t+\multistep}^k) )\nonumber\\
                          &\leq \E \brs*{\sum_{t'=t}^{t+\multistep-1}r(s^k_{t'},a^k_{t'}) + \bar{V}^{k-1}_{t+\multistep}(s_{t+\multistep}^k) \mid \pi_k,s_t^k}\nonumber\\
                          &\quad + \sum_{t'=t}^{t+\multistep - 1}\sum_{s_{t'}} \bra*{P^{\pi_k}(s_{t'} \mid s_t^k) - \hat{P}^{\pi_k}(s_{t'} \mid s_t^k)} + (H-(t+\multistep-1))\sum_{s_{t+\multistep}} \bra*{P^{\pi_k}(s_{t+\multistep} \mid s_t^k) - \hat{P}^{\pi_k}(s_{t+\multistep} \mid s_t^k)}.\nonumber
\end{align}
Applying Lemma~\ref{lemma: model error propogation} we bound the above by,
\begin{align}
    &\eqref{eq: approximate model first relation theorem} \leq \E \brs*{\sum_{t'=t}^{t+\multistep-1}r(s^k_{t'},a^k_{t'}) + \bar{V}^{k-1}_{t+\multistep}(s_{t+\multistep}^k) \mid \pi_k,s_t^k}+ \sum_{t'=t}^{t+\multistep - 1} (t'-t)\epsilon_P + (H-(t+\multistep-1))\multistep\epsilon_P \nonumber \\
    & = \E \brs*{\sum_{t'=t}^{t+\multistep-1}r(s^k_{t'},a^k_{t'}) + \bar{V}^{k-1}_{t+\multistep}(s_{t+\multistep}^k) \mid \pi_k,s_t^k} -  \frac{1}{2}(\multistep-1)\multistep\epsilon_P + (H-t)\multistep\epsilon_P\nonumber \\
    & = \E \brs*{\sum_{t'=t}^{t+\multistep-1}r(s^k_{t'},a^k_{t'}) + \bar{V}^{k-1}_{t+\multistep}(s_{t+\multistep}^k) \mid \mathcal{F}_{k-1},s_t^k} -  \frac{1}{2}(\multistep-1)\multistep\epsilon_P + (H-t)\multistep\epsilon_P. \label{eq: approximate model another first relation theorem}
\end{align}
Where the second relation holds by using the close form of the arithmetic sum and by algebraic manipulations. For the third relation, we observe that given $\pi_k,s_t^k$ the state $s^k_{t'}$ with $t'\geq t$ is independent of the past episodes (see~\ref{eq: model updates to filtration}),
\begin{align*}
    \E \brs*{\sum_{t'=t}^{t+\multistep-1}r(s^k_{t'},a^k_{t'}) + \bar{V}^{k-1}_{t+\multistep}(s_{t+\multistep}^k) \mid \pi_k,s_t^k} = \E \brs*{\sum_{t'=t}^{t+\multistep-1}r(s^k_{t'},a^k_{t'}) + \bar{V}^{k-1}_{t+\multistep}(s_{t+\multistep}^k) \mid \mathcal{F}_{k-1},s_t^k}
\end{align*}

Taking the conditional expectation w.r.t. $\mathcal{F}_{k-1}$ of both ~\eqref{eq: approximate model first relation theorem} and its RHS~\eqref{eq: approximate model another first relation theorem}, using the tower property and the fact for all $s$, $\bar{V}_{H+1}(s)=0$ we get,
\begin{align*}
    \E\brs*{\bar{V}_t^{k}(s_t^k)\mid \mathcal{F}_{k-1}} \leq& \E\brs*{\sum_{t'=t}^{t+h-1}r(s^k_{t'},a^k_{t'}) + \bar{V}^{k-1}_{t+\multistep}(s_{t+\multistep}^k) \mid \mathcal{F}_{k-1}}\\
    &-  \frac{1}{2}(\multistep-1)\multistep\epsilon_P + (H-t)\multistep\epsilon_P
\end{align*}

Let us denote $d_n \eqdef -  \frac{1}{2}(\multistep-1)\multistep\epsilon_P + (H-n)\multistep\epsilon_P$. Summing the above relation for all $n\in [\frac{H}{\multistep}]-1$, using linearity of expectation, and the fact $\bar{V}^{k}_{H+1}(s)=$ for all $s,k$,
\begin{align}
    &\sum_{n=0}^{\frac{H}{\multistep}-1 } \E\brs*{\bar{V}_{1+n\multistep}^{k}(s_t^k) \mid \F_{k-1}} 
      = \E\brs*{\sum_{t=1}^H r(s_t^k,a_t^k)\mid \F_{k-1}} + \sum_{n=1}^{\frac{H}{\multistep}-1 } \E\brs*{ \bar{V}_{1+n\multistep}^{k-1}(s_{1+n\multistep}^k) \mid \F_{k-1}} + \sum_{n=0}^{\frac{H}{\multistep}-1 }d_{1+n\multistep} \label{eq: approximate model second relation theorem}
\end{align}

By simple algebraic manipulation we get $\sum_{n=0}^{\frac{H}{\multistep}-1 }d_{1+n\multistep}= \frac{1}{2}H(H-1)\epsilon_P$ (see Lemma~\ref{lemma: algebraic bound for approximate model}). Thus, \eqref{eq: approximate model second relation theorem} has the following equivalent forms, by which we conclude the proof of this lemma.
\begin{align*}
     \iff & \bar{V}^{k}_{1}(s_1^k)+\sum_{n=1}^{\frac{H}{\multistep}-1 } \E\brs*{\bar{V}_{1+n\multistep}^{k}(s_t^k) \mid \F_{k-1}} = \E\brs*{\sum_{t=1}^H r(s_t^k,a_t^k)\mid \F_{k-1}} + \sum_{n=1}^{\frac{H}{\multistep}-1 } \E\brs*{ \bar{V}_{1+n\multistep}^{k-1}(s_{1+n\multistep}^k) \mid \F_{k-1}} +\frac{1}{2}H(H-1)\epsilon_P\\
     \iff & \bar{V}^{k}_{1}(s_1^k)+\sum_{n=1}^{\frac{H}{\multistep}-1 } \E\brs*{\bar{V}_{1+n\multistep}^{k}(s_t^k) \mid \F_{k-1}} = V^{\pi_k}(s_1^k) + \sum_{n=1}^{\frac{H}{\multistep}-1 } \E\brs*{ \bar{V}_{1+n\multistep}^{k-1}(s_{1+n\multistep}^k) \mid \F_{k-1}} +\frac{1}{2}H(H-1)\epsilon_P \\
     \iff & \bar{V}^{k}_{1}(s_1^k) - V^{\pi_k}(s_1^k) = \sum_{n=1}^{\frac{H}{\multistep}-1 } \E\brs*{ \bar{V}_{1+n\multistep}^{k-1}(s_{1+n\multistep}^k) -\bar{V}_{1+n\multistep}^{k}(s_{1+n\multistep}^k) \mid \F_{k-1}}  +\frac{1}{2}H(H-1)\epsilon_P\\
     \iff & \bar{V}^{k}_{1}(s_1^k) - V^{\pi_k}(s_1^k) = \sum_{k=1}^K \sum_{n=1}^{\frac{H}{\multistep}-1}\sum_s \bar{V}^{k-1}_{n\multistep+1}(s) - \E[\bar{V}^{k}_{n\multistep+1}(s)\mid \F_{k-1}]  +\frac{1}{2}H(H-1)\epsilon_P
\end{align*}

The second line holds by the fact $s_1^k$ is measurable w.r.t. $\mathcal{F}_{k-1}$, the third line holds since $$V^{\pi_k}_1(s_1^k)=\E\brs*{\sum_{t=1}^H r(s_t^k,a_t^k)\mid \F_{k-1}}.$$ The forth line holds by Lemma~\ref{lemma: On trajectory regret to Uniform regret} with $\bar V^k_t=g_t^k$ for $t=n\multistep+1$. See that the update of $\bar V^k_t$ occurs only at the visited state $s_t^k$ and the update rule uses $\bar V^{k-1}_{t+1}$, i.e., it is measurable w.r.t. to $\mathcal{F}_{k-1}$, and it is valid to apply the lemma. 

\end{proof}

\TheoremRegretRTDPApproximateModel*
\begin{proof}
We start by proving { \bf claim (1)}. The following bounds on the regret hold.
\begin{align}
     \Regret(K)&\eqdef \sum_{k=1}^K V^*_1(s^k_1)- V^{\pi_k}_1(s^k_1)
      \nonumber \\
     &\leq \sum_{k=1}^K \hat{V}^*_1(s^k_1)- V^{\pi_k}_1(s^k_1) +\frac{H(H-1)}{2}\epsilon_P
      \nonumber \\
     &\leq  \sum_{k=1}^K \bar{V}_1^{k}(s^k_1)- V^{\pi_k}_1(s^k_1) +\frac{H(H-1)}{2}\epsilon_P\nonumber \\
     &= H(H-1)\epsilon_P K + \sum_{k=1}^K \sum_{n=1}^{\frac{H}{\multistep}-1}\sum_s \bar{V}^{k-1}_{n\multistep+1}(s) - \E[\bar{V}^{k}_{n\multistep+1}(s)\mid \F_{k-1}] \label{eq: regret bound multiple step rtdp approximate model}
\end{align}
The second relation holds by Lemma~\ref{lemma: approximate model bound for planning} which relates the optimal value of the approximate model to the optimal value of the environment. The third relation is by the optimism of the value function (Lemma \ref{lemma: approximate model properties}), and the forth relation is by Lemma \ref{lemma: RTDP approximate modle expected value difference}. 

We now observe the regret is a regret of a Decreasing Bounded Process. Let
\begin{align}
    X_k \eqdef \sum_{n=1}^{\frac{H}{\multistep}-1}\sum_s \bar{V}^{k}_{n\multistep+1}(s), \label{def: DBP multiple rtdp approximate model}
\end{align}
and observe that $\brc*{X_k}_{g\geq 0}$ is a Decreasing Bounded Process.
\begin{enumerate}
    \item It is decreasing since for all $s,t$ $\bar{V}^k_t(s)\leq \bar{V}^{k-1}_t(s)$ by Lemma~\ref{lemma: approximate model properties}. Thus, their sum is also decreasing.
    \item It is bounded since for all $s,t$ $\bar{V}^k_t(s)\geq V_t^*(s) \geq 0$ by Lemma~\ref{lemma: approximate model properties}. Thus, the sum is bounded from below by $0$.
\end{enumerate}

See that the initial value can be bounded as follows,
\begin{align*}
    X_0 &= \sum_{n=1}^{\frac{H}{\multistep}-1}\sum_s \bar{V}^{0}_{n\multistep+1}(s)\leq \sum_{n=1}^{\frac{H}{\multistep}-1}\sum_s H= \frac{SH(H-\multistep)}{\multistep}.
\end{align*}

Using linearity of expectation and the definition \eqref{def: DBP multiple rtdp approximate model} we observe that $\eqref{eq: regret bound multiple step rtdp approximate model}$ can be written,
\begin{align*}
    \Regret(K) \leq \eqref{eq: regret bound multiple step rtdp approximate model} =  H(H-1)\epsilon_P K + \sum_{k=1}^K X_{k-1} - \E[X_k\mid \mathcal{F}_{k-1}],
\end{align*}

which is regret of A Bounded Decreasing Process. Applying the regret bound on DBP, Theorem~\ref{theorem: regret of decreasing bounded process},  we conclude the proof of the first claim. 

We now prove the {\bf claim (2)} using the proving technique at Theorem~\ref{theorem: regret multistep rtdp}. Denote $\Delta_P = H(H-1)\epsilon_P$. The following relations hold for all $\epsilon>0$.
\begin{align}
    &\ind\brc*{\hat{V}_1^{*}(s^k_1)-V_1^{\pi_k}(s^k_1) \geq  \frac{\Delta_P}{2} +\epsilon} \br*{\epsilon + \frac{\Delta_P}{2} } \nonumber \\
    &\leq \ind\brc*{\bar{V}_1^{k}(s^k_1)-V_1^{\pi_k}(s^k_1) \geq  \frac{\Delta_P}{2} +\epsilon} \br*{\epsilon + \frac{\Delta_P}{2} }  \nonumber \\
    &\leq \ind\brc*{\bar{V}_1^{k}(s^k_1)-V_1^{\pi_k}(s^k_1)  \geq  \frac{\Delta_P}{2} +\epsilon } \br*{\bar{V}_1^{k}(s^k_1)-V_1^{\pi_k}(s^k_1) }
    \nonumber \\
    &= \ind\brc*{\bar{V}_1^{k}(s^k_1)-V_1^{\pi_k}(s^k_1) \geq \frac{\Delta_P}{2} +\epsilon}\br*{\sum_{n=1}^{\frac{H}{\multistep}-1}\sum_s \bar{V}^{k-1}_{n\multistep+1}(s) - \E[\bar{V}^{k}_{n\multistep+1}(s)\mid \F_{k-1}] + \frac{\Delta_P}{2}}\nonumber \\
    &= \ind\brc*{\bar{V}_1^{k}(s^k_1)-V_1^{\pi_k}(s^k_1) \geq \frac{\Delta_P}{2} +\epsilon}\br*{X_{k-1} -\E[X_{k}\mid \mathcal{F}_{k-1}] +  \frac{\Delta_P}{2}}. \nonumber
\end{align}
The first relation holds since for all $t,s$, $\bar{V}^{k}_t(s)\geq \hat{V}^*_t(s)$ by Lemma~\ref{lemma: approximate model properties}. The second relation holds by the indicator function and the third relation holds by Lemma~\ref{lemma: RTDP approximate modle expected value difference}. The forth relation holds by  the definition of $X_k$~\eqref{def: DBP multiple rtdp approximate model} and linearity of expectation. Using an algebraic manipulation the above leads to the following relation,
\begin{align}
    \ind\brc*{\hat{V}_1^{*}(s^k_1)-V_1^{\pi_k}(s^k_1) \geq  \frac{\Delta_P}{2} +\epsilon} \epsilon  \leq \ind\brc*{\bar{V}_1^{k}(s^k_1)-V_1^{\pi_k}(s^k_1) \geq \frac{\Delta_P}{2} +\epsilon}\br*{X_{k-1} -\E[X_{k}\mid \mathcal{F}_{k-1}]}\label{eq: theorem approximate mopdel first relation}
\end{align}

As we wish the final performance to be compared to $V^*$ and not $\hat{V}$ we use the the first claim of Lemma~\ref{lemma: approximate model bound for planning}, by which for all $s$, $\hat{V}_1^{*}(s) \geq V_1^{*}(s) - \frac{\Delta_P}{2}$. This implies that 
\begin{align}
    \ind\brc*{ V_1^{*}(s^k_1)- V_1^{\pi_k}(s^k_1) \geq  \Delta_P +\epsilon } \leq \ind\brc*{\hat{V}_1^{*}(s^k_1)-V_1^{\pi_k}(s^k_1) \geq  \frac{\Delta_P}{2} +\epsilon }\label{eq: theorem approximate mopdel second relation}.
\end{align}

Combining all the above, we get
\begin{align}
    &\ind\brc*{ V_1^{*}(s^k_1)- V_1^{\pi_k}(s^k_1) \geq  \Delta_P +\epsilon }\epsilon \nonumber \\
    &\leq \ind\brc*{\hat{V}_1^{*}(s^k_1)-V_1^{\pi_k}(s^k_1) \geq  \frac{\Delta_P}{2} +\epsilon} \epsilon  \nonumber\\
    &\leq \ind\brc*{\bar{V}_1^{k}(s^k_1)-V_1^{\pi_k}(s^k_1) \geq \frac{\Delta_P}{2} +\epsilon}\br*{X_{k-1} -\E[X_{k}\mid \mathcal{F}_{k-1}]}. \label{eq: central identity second claim approximate model theorem}
\end{align}
The first relation is by~\eqref{eq: theorem approximate mopdel second relation} and the second relation by~\eqref{eq: theorem approximate mopdel first relation}.

Define $N_\epsilon(K) = \sum_{k=1}^K \ind\brc*{V_1^{*}(s^k_1)-V_1^{\pi_k}(s^k_1) \geq \Delta_P + \epsilon}$ as the  number of times $V^*_1(s_1^k)-V^{\pi_k}_1(s_1^k) \geq \Delta_P + \epsilon $ at the first $K$ episodes. Summing the above inequality~\eqref{eq: central identity second claim approximate model theorem} for all $k\in[K]$ and denote  we get that for all $\epsilon>0$
\begin{align*}
    &N_\epsilon(K) \epsilon = \sum_{k=1}^K \ind\brc*{V_1^{*}(s^k_1)-V_1^{\pi_k}(s^k_1) \geq \Delta_P + \epsilon} \epsilon \\
    &\leq \sum_{k=1}^K \ind\brc*{\bar{V}_1^{k}(s^k_1)-V_1^{\pi_k}(s^k_1) \geq \frac{\Delta_P}{2} +\epsilon}\br*{X_{k-1} -\E[X_{k}\mid \mathcal{F}_{k-1}]}\\
    &\leq \sum_{k=1}^K X_{k-1} -\E[X_{k}\mid \mathcal{F}_{k-1}].
\end{align*}
The first relation holds by definition, the second by~\eqref{eq: central identity second claim approximate model theorem} and the third relation holds as $\brc*{X_{k}}_{k\geq 0}$ is a DBP~\eqref{def: DBP multiple rtdp approximate model} and, thus, ${X_{k-1} - \E[X_{k}\mid \mathcal{F}_{k-1}] \geq 0}$ a.s. . Thus, the following relation holds
\begin{align*}
     \brc*{\forall K>0: \sum_{k=1}^K X_{k-1} -\E[X_{k}\mid \mathcal{F}_{k-1}] \leq \frac{9SH(H-\multistep)}{\multistep} \ln\frac{3}{\delta} } \subseteq \brc*{\forall \epsilon>0: N_\epsilon(K) \epsilon \leq \frac{9SH(H-\multistep)}{\multistep} \ln\frac{3}{\delta} },
\end{align*}

from which we get that for any $K>0$
\begin{align*}
    &\Pr\br*{\forall \epsilon>0: N_\epsilon(K) \epsilon \leq \frac{9SH(H-\multistep)}{\multistep} \ln\frac{3}{\delta} } \\
    &\geq \Pr\br*{\forall K>0: \sum_{k=1}^K X_{k-1} -\E[X_{k}\mid \mathcal{F}_{k-1}] \leq \frac{9SH(H\multistep)}{\multistep} \ln\frac{3}{\delta} } \geq 1- \delta,
\end{align*}
and the third relation holds the bound on the regret of DBP, Theorem~\ref{theorem: regret of decreasing bounded process}.  Equivalently, for any $K>0$, 
\begin{align}
    \Pr\br*{\exists \epsilon>0: N_\epsilon(K) \epsilon \geq \frac{9SH(H-\multistep)}{\multistep} \ln\frac{3}{\delta} } \leq \delta.\label{eq: final pac bound multistep rtdp approximate model}
\end{align}

Applying the Monotone Convergence Theorem as in the proof of Theorem~\ref{theorem: regret multistep rtdp} we conclude the proof.

\end{proof}


\newpage
\section{$\multistep$-RTDP with Approximate Value updates}\label{supp: multistep rtdp approximate value updates}

\begin{algorithm}[t]
\begin{algorithmic}
\caption{$h$-RTDP with Approximate Value Updates ($h$-RTDP-AV)}
\label{algo: RTDP approximate value update}
    \STATE init: $\forall s\in \mathcal S,\; n\in \{0\}\cup[\frac{H}{\multistep}],\; \bar{V}^0_{n \multistep +1}(s)=H-n\multistep$
    \FOR{$k\in[K]$}
        \STATE Initialize $s^k_1$
        \FOR{$t\in[H]$}
            \IF{$(t-1) \mod \multistep == 0$}
                \STATE $h_c = t + h$
                \STATE $\bar{V}^{k}_{t}(s_t^k) = \epsilon_V(s_t^k) + T^{\multistep}\bar{V}^{k-1}_{h_c}(s_t^k)\;;$ $\qquad\bar{V}^{k}_{t}(s_t^k) \gets \min \brc*{\bar{V}^{k}_{t}(s_t^k),\bar{V}^{k-1}_{t}(s_t^k)}\;;$
            \ENDIF
            \STATE $a_t^k\in \arg\max_a r(s_t^k,a) + p(\cdot | s_t^k,a) T^{h_c-t-1}\bar{V}^{k-1}_{h_c}\;;$
            \STATE Act with $a_t^k$ and observe $s_{t+1}^k\sim p(\cdot \mid s_t^k,a_t^k)$ 
        \ENDFOR
    \ENDFOR
\end{algorithmic}
\end{algorithm}

\begin{restatable}{lemma}{rtdpApproximateValueProperties}
\label{lemma:multistep approximate value rtdp properties}
For all $s\!\in\! \mathcal{S}$, $n\!\in\! \{0\}\!\cup\![\frac{H}{\multistep}]$, and $k\!\in\![K]$:
\begin{enumerate}[label=(\roman*)]
    \item Bounded / Optimism: $$V^*_{n\multistep +1}(s) \!\leq\! \bar{V}^k_{n\multistep +1}(s)\!+\! \epsilon_V(\frac{H}{\multistep}\!-\! n).$$
    \item Non-Increasing: $\bar{V}^{k}_{n\multistep +1}(s)  \leq \bar{V}^{k-1}_{n\multistep +1}(s)$.
\end{enumerate}
\end{restatable}

\begin{proof}
We prove the first claim by induction. The second claim holds by construction.

\paragraph{{\em (i)}} 
Let $n\in\brc*{0}\cup [\frac{H}{\multistep}]$. By the optimistic initialization, $\forall s,n,\  V^*_{1+\multistep n}(s)-\epsilon_V(\frac{H}{\multistep} - n) \leq V^*_{1+\multistep n}(s) \leq V^0_{1+\multistep n}(s)$. Assume the claim holds for $k-1$ episodes. Let $s_t^{k}$ be the state the algorithm is at in the $t= 1+hn$ time step of the $k$'th episode, i.e., at a time step in which a value update is taking place. Let $e\in\mathbb{R}^S$ be the constant vector of ones. By the value update of Algorithm~\ref{algo: RTDP approximate value update},
\begin{align}
\bar{V}^{k}_{t}(s_t^k) = \min \brc*{\epsilon_V(s_t^k) + T^{\multistep}\bar{V}^{k-1}_{h_c}(s_t^k),\bar{V}^{k-1}_{t}(s_t^k)}. \label{eq: supp approximate value update minmum on two terms}
\end{align}

If the minimal value is $\bar{V}^{k-1}_{t}(s_t^k)$ then $\bar{V}^{k}_{t}(s_t^k)$ satisfies the induction hypothesis by the induction assumption. If $\epsilon_V(s_t^k) + T^{\multistep}\bar{V}^{k-1}_{h_c}(s_t^k)$ is the minimal value in~\eqref{eq: supp approximate value update minmum on two terms}, then the following relation holds, 
\begin{align*}
    \bar{V}_t^{k}(s_t^{k}) &= \epsilon_V(s_t^{k})+ T^{\multistep}\bar{V}^{k-1}_{t+\multistep}(s_t^k)\\
    &\geq -\epsilon_V+ T^{\multistep}\bar{V}^{k-1}_{t+\multistep}(s_t^k)\\
    &\geq -\epsilon_V+ T^{\multistep}\br*{V^*_{t+\multistep}-e \epsilon_V(\frac{H}{\multistep} - n-1)}(s_t^k)\\
    &= -\epsilon_V+ T^{\multistep}V^*_{t+\multistep}(s_t^k) -\epsilon_V(\frac{H}{\multistep} - n-1)\\
    &= T^{\multistep}V^*_{t+\multistep}(s_t^k) -\epsilon_V(\frac{H}{\multistep} - n)\\
    &= V^*_{t}(s_t^k) -\epsilon_V(\frac{H}{\multistep} - n).
\end{align*}

The second relation holds by the assumption $| \epsilon_V(s_t^k) |\leq \epsilon_V$. The third relation by the induction hypothesis and the monotonicity of $T^h$. The forth relation holds since for any constant $\alpha\in \mathbb{R}$ and $V\in\mathbb{R}^s$, $T(V+\alpha e)= T V +\alpha$~(e.g.,\cite{bertsekas1996neuro}) and thus $T^\multistep(V+\alpha e)= T^\multistep V +\alpha$. Lastly, the fifth relation holds by the Bellman equations~\eqref{eq:multistep bellman}.

\paragraph{{\em (ii)}}  The second claim holds by construction of the update rule $\bar{V}^{k}_{t}(s_t^k) \gets \min \brc*{\bar{V}^{k}_{t}(s_t^k),\bar{V}^{k-1}_{t}(s_t^k)}$ which enforces ${\bar{V}^{k}_{t}(s) \leq  \bar{V}^{k-1}_{t}(s) }$ for every updated state, and thus for all $s$ and $t$.
\end{proof}

\begin{restatable}{lemma}{rdtpApproximateValueExpectedValueUpdate}
\label{lemma: multistep approximate value update RTDP expected value difference}
The expected cumulative value update at the $k$'th episode of $h$-RTDP-AV satisfies the following relation:
\begin{align*}
    &\bar{V}_1^{k}(s^k_1)-V_1^{\pi_k}(s^k_1) \\
    &\leq \frac{H}{\multistep} \epsilon_V +\sum_{k=1}^K \sum_{n=1}^{\frac{H}{\multistep}-1}\sum_{s\in \mathcal{S}} \bar{V}^{k-1}_{n\multistep+1}(s) - \E[\bar{V}^{k}_{n\multistep+1}(s)\mid \F_{k-1}] .
\end{align*}
\vspace{-0.3cm}

\end{restatable}
\begin{proof}
Let $n\in\brc*{0}\cup[\frac{H}{h}-1]$ and let $t= 1+hn$ be a time step in which a value update is taking place. By the definition of the update rule, the following holds for the update at the visited state $s_t^k$:

\begin{align*}
     \bar{V}_t^{k}(s_t^k)  &= \epsilon_V(s_t^{k}) + (T^h \bar{V}_{t+h}^{k-1})(s_t^k)\\
                           &\leq \epsilon_V + (T^{\pi_{k}(t)}\cdot\cdot\cdot T^{\pi_{k}(t+h-1)}\bar{V}_{t+h}^{k-1})(s_t^k)\\
                           &= \epsilon_V + \E\brs*{\sum_{t'=t}^{t+h-1}r(s^k_{t'},a^k_{t'}) + \bar{V}^{k-1}_{t+h}(s_{t+h}^k) \mid \mathcal{F}_{k-1},s_t^k}.
\end{align*}
Where the third relation holds by the same argument as in~\eqref{eq: model updates to filtration}. Taking the conditional expectation w.r.t. $\mathcal{F}_{k-1}$, using the tower property and the fact for all $s$, $\bar{V}_{H+1}(s)=0$ we get,
\begin{align*}
    \E\brs*{\bar{V}_t^{k}(s_t^k)\mid \mathcal{F}_{k-1}} \leq \epsilon_V + \E\brs*{\sum_{t'=t}^{t+h-1}r(s^k_{t'},a^k_{t'}) + \bar{V}^{k-1}_{t+h}(s_{t+h}^k) \mid \mathcal{F}_{k-1}}.
\end{align*}
Summing the above relation for all $n\in \brc*{0}\cup[\frac{H}{h}-1]$, using linearity of expectation, and the fact $\bar{V}^{k}_{H+1}(s)=$ for all $s,k$,
\begin{align*}
    &\sum_{n=0}^{\frac{H}{\multistep}-1 } \E\brs*{\bar{V}_{1+n\multistep}^{k}(s_t^k) \mid \F_{k-1}} 
      \leq \frac{H}{\multistep}\epsilon_V + \E\brs*{\sum_{t=1}^H r(s_t^k,a_t^k)\mid \F_{k-1}} + \sum_{n=1}^{\frac{H}{\multistep}-1 } \E\brs*{ \bar{V}_{1+n\multistep}^{k-1}(s_{1+n\multistep}^k) \mid \F_{k-1}} \\
     \iff & \bar{V}^{k}_{1}(s_1^k)+\sum_{n=1}^{\frac{H}{\multistep}-1 } \E\brs*{\bar{V}_{1+n\multistep}^{k}(s_t^k) \mid \F_{k-1}} \leq \frac{H}{\multistep}\epsilon_V + \E\brs*{\sum_{t=1}^H r(s_t^k,a_t^k)\mid \F_{k-1}} + \sum_{n=1}^{\frac{H}{\multistep}-1 } \E\brs*{ \bar{V}_{1+n\multistep}^{k-1}(s_{1+n\multistep}^k) \mid \F_{k-1}} \\
     \iff & \bar{V}^{k}_{1}(s_1^k)+\sum_{n=1}^{\frac{H}{\multistep}-1 } \E\brs*{\bar{V}_{1+n\multistep}^{k}(s_t^k) \mid \F_{k-1}} \leq \frac{H}{\multistep}\epsilon_V+ V^{\pi_k}(s_1^k) + \sum_{n=1}^{\frac{H}{\multistep}-1 } \E\brs*{ \bar{V}_{1+n\multistep}^{k-1}(s_{1+n\multistep}^k) \mid \F_{k-1}} \\
     \iff & \bar{V}^{k}_{1}(s_1^k) - V^{\pi_k}(s_1^k) \leq \frac{H}{\multistep}\epsilon_V + \sum_{n=1}^{\frac{H}{\multistep}-1 } \E\brs*{ \bar{V}_{1+n\multistep}^{k-1}(s_{1+nh}^k) -\bar{V}_{1+n\multistep}^{k}(s_{1+nh}^k) \mid \F_{k-1}} \\
     \iff & \bar{V}^{k}_{1}(s_1^k) - V^{\pi_k}(s_1^k) \leq \frac{H}{\multistep}\epsilon_V +  \sum_{k=1}^K \sum_{n=1}^{\frac{H}{\multistep}-1}\sum_s \bar{V}^{k-1}_{n\multistep+1}(s) - \E[\bar{V}^{k}_{n\multistep+1}(s)\mid \F_{k-1}] 
\end{align*}

The second line holds by the fact $s_1^k$ is measurable w.r.t. $\mathcal{F}_{k-1}$, and the third line holds since $$V^{\pi_k}_1(s_1^k)=\E\brs*{\sum_{t=1}^H r(s_t^k,a_t^k)\mid \F_{k-1}}.$$ 
The fifth line holds by by Lemma~\ref{lemma: On trajectory regret to Uniform regret} with $\bar V^k_t=g_t^k$ for $t=n\multistep+1$. See that the update of $\bar V^k_t$ occurs only at the visited state $s_t^k$ and the update rule uses $\bar V^{k-1}_{t+1}$, i.e., it is measurable w.r.t. to $\mathcal{F}_{k-1}$, and it is valid to apply the lemma.

\end{proof}

\TheoremRegretApproximateValueRTDP*

\begin{proof}
We start by proving { \bf claim (1)}. The following bounds on the regret hold.
\begin{align}
     \Regret(K)&\eqdef \sum_{k=1}^K V^*_1(s^k_1)- V^{\pi_k}_1(s^k_1)
      \nonumber \\
     &\leq  \sum_{k=1}^K \bar{V}_1^{k}(s^k_1)- V^{\pi_k}_1(s^k_1)  + \frac{H}{\multistep}\epsilon_V \nonumber \\
     &= \frac{2H}{\multistep}\epsilon_V K + \sum_{k=1}^K \sum_{n=1}^{\frac{H}{\multistep}-1}\sum_s \bar{V}^{k-1}_{n\multistep+1}(s) - \E[\bar{V}^{k}_{n\multistep+1}(s)\mid \F_{k-1}] \label{eq: regret bound multiple step rtdp appro value updates}
\end{align}
The second relation is by the approximated optimism of the value function when approximate value updates are used (Lemma \ref{lemma:multistep approximate value rtdp properties}). The third relation is by Lemma \ref{lemma: multistep approximate value update RTDP expected value difference}. 

We now observe the regret is a regret of a Decreasing Bounded Process. Let
\begin{align}
    X_k \eqdef \sum_{n=1}^{\frac{H}{\multistep}-1}\sum_s \bar{V}^{k}_{n\multistep+1}(s), \label{def: DBP multiple rtdp approximate value}
\end{align}
and observe that $\brc*{X_k}_{g\geq 0}$ is a Decreasing Bounded Process.

\begin{enumerate}
    \item It is decreasing since for all $s,t$ $\bar{V}^k_t(s)\leq \bar{V}^{k-1}_t(s)$ by Lemma~\ref{lemma:multistep approximate value rtdp properties}. Thus, their sum is also decreasing.
    \item It is bounded since for all $s,n\in [\frac{H}{\multistep}]-1$, $$\bar{V}^k_{1+\multistep n}(s)\geq V_{1+\multistep n}^*(s) - \epsilon_V(\frac{H}{\multistep} - n) \geq - \epsilon_V(\frac{H}{\multistep}-n)\geq - \epsilon_V \frac{H}{\multistep}$$ by Lemma~\ref{lemma:multistep approximate value rtdp properties}. Thus, $X_0$ which is a sum of the above terms is bounded from below by $
    -  \frac{\epsilon_V}{\multistep}\frac{SH(H-\multistep)}{\multistep}$.
\end{enumerate}

See that the initial value can be bounded as follows,
\begin{align*}
    X_0 &= \sum_{n=1}^{\frac{H}{\multistep}-1}\sum_s \bar{V}^{0}_{n\multistep+1}(s)\leq \sum_{n=1}^{\frac{H}{\multistep}-1}\sum_s H= \frac{SH(H-\multistep)}{\multistep}.
\end{align*}

Using linearity of expectation and the definition \eqref{def: DBP multiple rtdp} we observe that $\eqref{eq: regret bound multiple step rtdp appro value updates}$ can be written,
\begin{align*}
    \Regret(K) \leq \eqref{eq: regret bound multiple step rtdp appro value updates} = \frac{2H}{\multistep}\epsilon_V K + \sum_{k=1}^K X_{k-1} - \E[X_k\mid \mathcal{F}_{k-1}],
\end{align*}

which is regret of A Bounded Decreasing Process. Applying the regret bound on DBP, Theorem~\ref{theorem: regret of decreasing bounded process} we conclude the proof of the first claim.

We now prove {\bf claim (2)} using the proving technique at Theorem~\ref{theorem: regret multistep rtdp}. Denote $\Delta_V = 2H\epsilon_V$. The following relations hold for all $\epsilon>0$. 
\begin{align}
    &\ind\brc*{\bar{V}_1^{k}(s^k_1)-V_1^{\pi_k}(s^k_1) \geq  \frac{\Delta_V}{2\multistep} +\epsilon} \br*{\epsilon + \frac{\Delta_V}{2\multistep} }  \nonumber \\
    &\leq \ind\brc*{\bar{V}_1^{k}(s^k_1)-V_1^{\pi_k}(s^k_1)  \geq  \frac{\Delta_V}{2\multistep} +\epsilon } \br*{\bar{V}_1^{k}(s^k_1)-V_1^{\pi_k}(s^k_1) }
    \nonumber \\
    &= \ind\brc*{\bar{V}_1^{k}(s^k_1)-V_1^{\pi_k}(s^k_1) \geq \frac{\Delta_V}{2\multistep} +\epsilon}\br*{\sum_{n=1}^{\frac{H}{\multistep}-1}\sum_s \bar{V}^{k-1}_{n\multistep+1}(s) - \E[\bar{V}^{k}_{n\multistep+1}(s)\mid \F_{k-1}] + \frac{\Delta_V}{2\multistep}}\nonumber \\
    &= \ind\brc*{\bar{V}_1^{k}(s^k_1)-V_1^{\pi_k}(s^k_1) \geq \frac{\Delta_V}{2\multistep} +\epsilon}\br*{X_{k-1} -\E[X_{k}\mid \mathcal{F}_{k-1}] +  \frac{\Delta_V}{2\multistep}}. \nonumber
\end{align}
The first relation holds by the indicator function and the second relation by Lemma~\ref{lemma: multistep approximate value update RTDP expected value difference}. The third relation holds by  the definition of $X_k$~\eqref{def: DBP multiple rtdp approximate value} and linearity of expectation. Using an algebraic manipulation the above leads to the following relation,
\begin{align}
    \ind\brc*{\bar{V}_1^{k}(s^k_1) - V_1^{\pi_k}(s^k_1) \geq  \frac{\Delta_V}{2\multistep} +\epsilon} \epsilon  \leq \ind\brc*{\bar{V}_1^{k}(s^k_1)-V_1^{\pi_k}(s^k_1) \geq \frac{\Delta_V}{2\multistep} +\epsilon}\br*{X_{k-1} -\E[X_{k}\mid \mathcal{F}_{k-1}]}\label{eq: theorem approximate value first relation}
\end{align}

As we wish the final performance to be compared to $V^*$ we use the the first claim of Lemma~\ref{lemma:multistep approximate value rtdp properties}, by which for all $s,k$, $\bar{V}_1^{k}(s) \geq V_1^{*}(s) - \frac{\Delta_V}{2\multistep}$. This implies that 
\begin{align}
    \ind\brc*{ V_1^{*}(s^k_1)- V_1^{\pi_k}(s^k_1) \geq  \frac{\Delta_V}{\multistep} +\epsilon } \leq \ind\brc*{\bar V_1^{k}(s^k_1)-V_1^{\pi_k}(s^k_1) \geq  \frac{\Delta_V}{2\multistep} +\epsilon }\label{eq: theorem approximate value second relation}.
\end{align}

Combining the above we get
\begin{align}
    &\ind\brc*{ V_1^{*}(s^k_1)- V_1^{\pi_k}(s^k_1) \geq  \frac{\Delta_V}{\multistep}+\epsilon }\epsilon \nonumber \\
    &\leq \ind\brc*{\bar{V}_1^{k}(s^k_1) - V_1^{\pi_k}(s^k_1) \geq  \frac{\Delta_V}{2\multistep} +\epsilon} \epsilon  \nonumber\\
    &\leq \ind\brc*{\bar{V}_1^{k}(s^k_1)-V_1^{\pi_k}(s^k_1) \geq \frac{\Delta_V}{2\multistep} +\epsilon}\br*{X_{k-1} -\E[X_{k}\mid \mathcal{F}_{k-1}]}. \label{eq: central identity second claim approximate value update}
\end{align}
The first relation is by~\eqref{eq: theorem approximate value second relation} and the second relation by~\eqref{eq: theorem approximate value first relation}.

Define $N_\epsilon(K) = \sum_{k=1}^K \ind\brc*{V_1^{*}(s^k_1)-V_1^{\pi_k}(s^k_1) \geq \frac{\Delta_V}{\multistep} + \epsilon}$ as the  number of times $V^*_1(s_1^k)-V^{\pi_k}_1(s_1^k) \geq \frac{\Delta_V}{\multistep} + \epsilon $ at the first $K$ episodes. Summing the above inequality~\eqref{eq: central identity second claim approximate value update} for all $k\in[K]$ and denote  we get that for all $\epsilon>0$
\begin{align*}
    &N_\epsilon(K)\epsilon = \sum_{k=1}^K \ind\brc*{V_1^{*}(s^k_1)-V_1^{\pi_k}(s^k_1) \geq \frac{\Delta_V}{\multistep} + \epsilon} \epsilon \\
    &\leq \sum_{k=1}^K \ind\brc*{\bar{V}_1^{k}(s^k_1)-V_1^{\pi_k}(s^k_1) \geq \frac{\Delta_V}{2\multistep} +\epsilon}\br*{X_{k-1} -\E[X_{k}\mid \mathcal{F}_{k-1}]}\\
    &\leq \sum_{k=1}^K X_{k-1} -\E[X_{k}\mid \mathcal{F}_{k-1}].
\end{align*}
The first relation holds by definition, the second by~\eqref{eq: central identity second claim approximate value update} and the third relation holds as $\brc*{X_{k}}_{k\geq 0}$ is a DBP~\eqref{def: DBP multiple rtdp approximate value} and, thus, ${X_{k-1} - \E[X_{k}\mid \mathcal{F}_{k-1}] \geq 0}$ a.s. . Thus, the following relation holds
\begin{align*}
     &\brc*{\forall K>0: \sum_{k=1}^K X_{k-1} -\E[X_{k}\mid \mathcal{F}_{k-1}] \leq \frac{9SH(H-\multistep)}{\multistep}(1+\frac{H}{\multistep}\epsilon_V) \ln\frac{3}{\delta} } \\
     &\subseteq \brc*{\forall \epsilon>0: N_\epsilon(K) \epsilon \leq \frac{9SH(H-\multistep)}{\multistep}(1+\frac{H}{\multistep}\epsilon_V) \ln\frac{3}{\delta} },
\end{align*}

from which we get for any $K>0$
\begin{align*}
    &\Pr\br*{\forall \epsilon>0: N_\epsilon(K) \epsilon \leq \frac{9SH(H-\multistep)}{\multistep} (1+\frac{H}{\multistep}\epsilon_V)\ln\frac{3}{\delta} } \\
    &\geq \Pr\br*{\forall K>0: \sum_{k=1}^K X_{k-1} -\E[X_{k}\mid \mathcal{F}_{k-1}] \leq \frac{9SH(H\multistep)}{\multistep} (1+\frac{H}{\multistep}\epsilon_V)\ln\frac{3}{\delta} } \geq 1- \delta,
\end{align*}
and the third relation holds the bound on the regret of DBP, Theorem~\ref{theorem: regret of decreasing bounded process}. Equivalently, for any $K>0$, 
\begin{align}
    \Pr\br*{\exists \epsilon>0: N_\epsilon(K) \epsilon \geq \frac{9SH(H-\multistep)}{\multistep} (1+\frac{H}{\multistep}\epsilon_V)\ln\frac{3}{\delta} } \leq \delta. \label{eq: final pac bound multistep rtdp approixamte value}
\end{align}
Applying the Monotone Convergence Theorem as in the proof of Theorem~\ref{theorem: regret multistep rtdp} we conclude the proof.

\end{proof}


\newpage
\section{$\multistep$-RTDP with Approximate State Abstraction}\label{supp: multistep rtdp abstractions}

\begin{algorithm}[t]
\begin{algorithmic}
\caption{$h$-RTDP with Approximate State Abstraction ($h$-RTDP-AA)}
\label{algo: RTDP with abstractions}
    \STATE init: $\forall s_\phi \in \mathcal S_\phi ,\; n\in \{0\}\cup[\frac{H}{\multistep}],\; \bar{V}^0_{\phi, n \multistep +1}(s_\phi)=H-n\multistep$
    \FOR{$k\in[K]$}
        \STATE Initialize $s^k_1$
        \FOR{$t\in[H]$}
            \IF{$(t-1) \mod \multistep == 0$}
                \STATE $h_c = t + h\;;$ $\qquad\bar{V}^{k}_{\phi,t}(\phi_t(s_t^k)) =  T_\phi^{\multistep}\bar{V}^{k-1}_{\phi, h_c}(s_{t}^k)\;;$ 
                \STATE ${\bar{V}^{k}_{\phi,t}(\phi_t(s_t^k)) \gets \min \brc*{\bar{V}^{k}_{\phi,t}(\phi_t(s_t^k)),\bar{V}^{k-1}_{\phi,t}(\phi_t(s_t^k))}}\;;$
            \ENDIF
            \STATE $a_t^k\in\arg\max_a r(s_t^k,a) + p(\cdot|s_t^k,a) T_\phi^{h_c-t-1}\bar{V}^{k-1}_{\phi,h_c}\;;$
            \STATE Act with $a_t^k$ and observe $s_{t+1}^k\sim p(\cdot \mid s_t^k,a_t^k)$ 
        \ENDFOR
    \ENDFOR
\end{algorithmic}
\end{algorithm}

In this section we analyze the performance of $\multistep$-RTDP performance which uses approximate abstraction. For clarity we restate the assumption we make on the approximate abstraction and the definition of equivalent set under abstraction.

\assumptionModelAbstraction*

An important quantity in our analysis is the set of states equivalent to a given state $s$ under $\phi_{nh+1}$.

\begin{restatable}[Equivalent Set Under Abstraction]{defn}{defnEquivalentSet}
For any $s\in\mathcal{S}$ and $n\in\{0\}\cup[\frac{H}{\multistep}-1]$, we define the set of states equivalent to $s$ under $\phi_{nh+1}$ as $\Phi_{nh+1}(s)\eqdef\{s'\in\mathcal{S}: \phi_{nh+1}(s)=\phi_{nh+1}(s')\}$.
\end{restatable}

Before we supply with the proof we emphasize an important difference in the definition of the value function $\bar V_t^k$ when using abstraction. Unlike the usual definition of $\bar V_t^k: \mathcal{S} \rightarrow \mathbb{R} $, in case of abstraction $\bar V_{\phi,t}^k$ is a mapping from the \emph{abstract state space} to the reals, i.e., $\bar V_{\phi,t}^k: \mathcal{S}_\phi \rightarrow \mathbb{R}$. Meaning, $\bar V_{\phi,t}^k$ is defined on the abstract state space. Given a state $s\in\mathcal{S}$ we need to query $\phi_t$ to obtain its value at time $t$ by $\bar V_t^k(\phi_t(s))$.

\begin{restatable}{lemma}{rtdpPropertiesAbstraction}
\label{lemma:multistep rtdp properties abstraction}
For all $s\in \mathcal{S}$, $n\in\{0\}\cup[\frac{H}{\multistep}]$, and $k\in[K]$:
\begin{enumerate}[label=(\roman*)]
    \item Optimism: $$\max_{s'\in \Phi_{nh+1}(s)}V_{n\multistep+1}^*(s') \leq \bar V^k_{n\multistep+1}(\phi_{n\multistep+1}(s)) + \epsilon_A(\frac{H}{\multistep}-n ).$$
    \item Bounded: $\bar{V}^{k}_{n\multistep+1}(\phi_{n\multistep+1}(s))\geq 0.$
    \item Non-Increasing: ${ \bar{V}^{k}_{n\multistep+1}(\phi_{n\multistep+1}(s)) \! \leq \! \bar{V}^{k-1}_{n\multistep+1}(\phi_{n\multistep+1}(s)).}$
\end{enumerate}
\end{restatable}

\begin{proof}
We prove the first claim by induction. The second and third claims hold by construction.

\paragraph{{\em (i)}} 
Let $n\in \brc*{0} \cup [\frac{H}{\multistep}-1]$. By the optimistic initialization, $\forall s,n,\  V^*_{1+\multistep n}(s)-\epsilon_A(\frac{H}{\multistep} - n) \leq V^*_{1+\multistep n}(s) \leq V^0_{1+\multistep n}(\phi_{1+\multistep n}(s))$. Assume the claim holds for $k-1$ episodes. Let $s_t^{k}$ be the state the algorithm is at in the $t= 1+hn$ time step of the $k$'th episode, i.e., at a time step in which a value update is taking place. By the value update of Algorithm~\ref{algo: RTDP with abstractions},
\begin{align}
\bar{V}^{k}_{t}(\phi(s_t^k)) = \min \brc*{T^{\multistep}\bar{V}^{k-1}_{ h_c}(s_{t}^k),\bar{V}^{k-1}_{t}(\phi(s_t^k))}. \label{eq: supp approximate abstraction update minmum on two terms}
\end{align}
%


If the minimal value is $\bar{V}^{k-1}_{t}(\phi(s_t^k))$ then $\bar{V}^{k}_{t}(\phi(s_t^k))$ satisfies the induction hypothesis by the induction assumption. If $T^{\multistep}\bar{V}^{k-1}_{h_c}(s_t^k)$ is the minimal value in~\eqref{eq: supp approximate abstraction update minmum on two terms}, then the following relation holds, 

\begin{align*}
    \bar{V}_t^{k}(\phi_t(s_t^k)) &= \max_{\pi_0,\pi_1,..,\pi_{h-1}} \; \E[\sum_{t'=0}^{h-1}r(s_t',\pi_{t'}(s_t')) +\bar{V}^{k-1}_{t+h}(\phi(s_h))\mid s_0=s_t^k]\\
    &\geq \max_{\pi_0,\pi_1,..,\pi_{h-1}} \; \E[\sum_{t'=0}^{h-1}r(s_t',\pi_{t'}(s_t')) + \max_{s'\in \Phi_{t+\multistep}(s_h)} V^*_{t+h}(s') - \epsilon_A\br*{\frac{H}{\multistep} -n -1 } \mid s_0=s_t^k]\\
    &= \max_{\pi_0,\pi_1,..,\pi_{h-1}} \; \E[\sum_{t'=0}^{h-1}r(s_t',\pi_{t'}(s_t')) + \max_{s'\in \Phi_{t+\multistep}(s_h)} V^*_{t+h}(s')  \mid s_0=s_t^k]- \epsilon_A\br*{\frac{H}{\multistep}-n  -1}\\
    &\geq \max_{\pi_0,\pi_1,..,\pi_{h-1}} \; \E[\sum_{t'=0}^{h-1}r(s_t',\pi_{t'}(s_t')) +  V^*_{t+h}(s_h)  \mid s_0=s_t^k]- \epsilon_A\br*{\frac{H}{\multistep}-n  -1}\\
    &=  V^*_{t}(s_t^k) - \epsilon_A\br*{\frac{H}{\multistep}-n  -1}\\
    &\geq \max_{s'\in \Phi_t(s_t^k)} V^*_{t}(s_t^k) - \epsilon_A - \epsilon_A\br*{\frac{H}{\multistep}-n-1} \\
    &= \max_{s'\in \Phi_t(s_t^k)} V^*_{t}(s_t^k) - \epsilon_A\br*{\frac{H}{\multistep}-n}.
\end{align*}

The first relation is the definition of the update rule. The second relation holds by the monotonicity of the $\max$ operator together with the induction assumption. The third relation as the extracted term out of the $\max$ is constant. The forth relation holds by the definition of the $\max$ operation. The fifth relation by the Bellman equations $V^*_t$ satisfies~\eqref{eq:multistep bellman}, and the sixth relation by Assumption~\ref{assumption: model abstraction}.

\paragraph{{\em (ii)}}  The second claim holds by construction of the update rule $\bar{V}^{k}_{t}(s_t^k) \gets \min \brc*{\bar{V}^{k}_{t}(s_t^k),\bar{V}^{k-1}_{t}(s_t^k)}$ which enforces ${\bar{V}^{k}_{t}(s) \leq  \bar{V}^{k-1}_{t}(s) }$ for every updated state, and thus for all $s$ and $t$.

\paragraph{{\em (iii)}} The third claim holds since $V_t^k(\phi_t(s))$ is initialized with positive elements and is updated by itself and positive elements, as $r(s,a)\geq 0 $. Thus, it remains positive a.s. .
\end{proof}

\begin{restatable}{lemma}{rdtpExpectedValueUpdateAbstraction}
\label{lemma: multistep RTDP expected value difference abstraction}
The expected cumulative value update at the $k$'th episode of $h$-RTDP-AA satisfies the following relation:

\begin{align*}
    &\bar{V}_1^{k}(\phi(s^k_1))-V_1^{\pi_k}(s^k_1) \\
    &\leq \sum_{k=1}^K \sum_{n=1}^{\frac{H}{\multistep}-1}\sum_{s_\phi\in \mathcal{S}_\phi} \bar{V}^{k-1}_{n\multistep+1}(s_\phi) - \E[\bar{V}^{k}_{n\multistep+1}(s_\phi)\mid \F_{k-1}].
\end{align*}
\vspace{-0.3cm}
\end{restatable}
\begin{proof}
Let $n\in \brc*{0} \cup [\frac{H}{\multistep}-1]$ and let $t= 1+hn$ be a time step in which a value update is taking place. By the definition of the update rule, the following holds for the update at the visited state $s_t^k$:

\begin{align*}
     \bar{V}_t^{k}(\phi_t(s_t^k)) &\leq  \E\brs*{\sum_{t'=t}^{t+h-1}r(s^k_{t'},a^k_{t'}) + \bar{V}^{k-1}_{t+h}(\phi_{t+h}(s_{t+h}^k)) \mid \pi_k,s_t^k}\\
     &=\E\brs*{\sum_{t'=t}^{t+h-1}r(s^k_{t'},a^k_{t'}) + \bar{V}^{k-1}_{t+h}(\phi_{t+h}(s_{t+h}^k)) \mid \mathcal{F}_{k-1},s_t^k}
\end{align*}
where the last relation follows by the same argument as in~\eqref{eq: model updates to filtration}.

Taking the conditional expectation w.r.t. $\mathcal{F}_{k-1}$ and using the tower property we get,
\begin{align*}
    \E\brs*{\bar{V}_t^{k}(\phi_t(s_t^k))\mid \mathcal{F}_{k-1}} \leq \E\brs*{\sum_{t'=t}^{t+h-1}r(s^k_{t'},a^k_{t'}) + \bar{V}^{k-1}_{t+h}(\phi_{t+h}(s_{t+h}^k)) \mid \mathcal{F}_{k-1}}.
\end{align*}
Denote $s_{\phi,t}^k\eqdef \phi_t(s_t^k).$ Summing the above relation for all $n\in \brc*{0} \cup [\frac{H}{\multistep}-1]$, using linearity of expectation, and the fact $\bar{V}^{k}_{H+1}(\phi_{H+1}(s))=0$ for all $s,k$,
\begin{align*}
    &\sum_{n=0}^{\frac{H}{\multistep}-1 } \E\brs*{\bar{V}_{1+n\multistep}^{k}(s_{\phi,1+n\multistep}^k) \mid \F_{k-1}} 
      \leq \E\brs*{\sum_{t=1}^H r(s_t^k,a_t^k)\mid \F_{k-1}} + \sum_{n=1}^{\frac{H}{\multistep}-1 } \E\brs*{ \bar{V}_{1+n\multistep}^{k-1}(s_{\phi,1+n\multistep}^k) \mid \F_{k-1}} \\
     \iff & \bar{V}^{k}_{1}(s_{\phi,1}^k)+\sum_{n=1}^{\frac{H}{\multistep}-1 } \E\brs*{\bar{V}_{1+n\multistep}^{k}(s_{\phi,1+n\multistep}^k) \mid \F_{k-1}} \leq \E\brs*{\sum_{t=1}^H r(s_t^k,a_t^k)\mid \F_{k-1}} + \sum_{n=1}^{\frac{H}{\multistep}-1 } \E\brs*{ \bar{V}_{1+n\multistep}^{k-1}(s_{\phi,1+n\multistep}^k) \mid \F_{k-1}} \\
     \iff & \bar{V}^{k}_{1}(s_{\phi,1}^k)+\sum_{n=1}^{\frac{H}{\multistep}-1 } \E\brs*{\bar{V}_{1+n\multistep}^{k}(s_{\phi,1+n\multistep}^k) \mid \F_{k-1}} \leq V^{\pi_k}(s_1^k) + \sum_{n=1}^{\frac{H}{\multistep}-1 } \E\brs*{ \bar{V}_{1+n\multistep}^{k-1}(s_{\phi,1+n\multistep}^k) \mid \F_{k-1}} \\
     \iff & \bar{V}^{k}_{1}(s_{\phi,1}^k) - V^{\pi_k}(s_1^k) \leq \sum_{n=1}^{\frac{H}{\multistep}-1 } \E\brs*{ \bar{V}_{1+n\multistep}^{k-1}(s_{\phi,1+n\multistep}^k) -\bar{V}_{1+n\multistep}^{k}(s_{\phi,1+n\multistep}^k) \mid \F_{k-1}} \\
      \iff & \bar{V}^{k}_{1}(s_{\phi,1}^k) - V^{\pi_k}(s_1^k) \leq \sum_{k=1}^K \sum_{n=1}^{\frac{H}{\multistep}-1}\sum_{s_\phi\in \mathcal{S}_\phi} \bar{V}^{k-1}_{n\multistep+1}(s_\phi) - \E[\bar{V}^{k}_{n\multistep+1}(s_\phi)\mid \F_{k-1}]
\end{align*}

The second line holds by the fact $s_1^k$ is measurable w.r.t. $\mathcal{F}_{k-1}$, the third line holds since $$V^{\pi_k}_1(s_1^k)=\E\brs*{\sum_{t=1}^H r(s_t^k,a_t^k)\mid \F_{k-1}}.$$ 

The fifth line holds by Lemma~\ref{lemma: On trajectory regret to Uniform regret} with $\bar V^k_t=g_t^k$ for $t=n\multistep+1$. Furthermore, we set $\tilde{\mathcal{S}}$ of Lemma~\ref{lemma: On trajectory regret to Uniform regret} to be $\mathcal{S}_\phi$. See that the update of $\bar V^k_t$ occurs only at the visited state $s_{\phi,t}^k=\phi(s_t^k)$ of the abstracted state space. Furthermore, the update rule uses $\bar V^{k-1}_{\phi,t+1}$, i.e., it is measurable w.r.t. to $\mathcal{F}_{k-1}$, and it is valid to apply the lemma. 
\end{proof}

\TheoremRegretRTDPAbstraction*

Before supplying with the proof observe the following remark.

\begin{proof}
We start by proving { \bf claim (1)}. The following bounds on the regret hold.

\begin{align}
     \Regret(K)&\eqdef \sum_{k=1}^K V^*_1(s^k_1)- V^{\pi_k}_1(s^k_1)
      \nonumber \\
     &\leq  \sum_{k=1}^K \max_{s\in \Phi_1(s^k_1)}V^*_1(s)- V^{\pi_k}_1(s^k_1)  \nonumber \\
     &\leq  \sum_{k=1}^K \bar{V}_1^{k}(\phi_1(s^k_1))- V^{\pi_k}_1(s^k_1) + \epsilon_A\frac{H}{\multistep} \nonumber \\
     &\leq  \epsilon_A\frac{H}{\multistep} K + \sum_{k=1}^K \sum_{n=1}^{\frac{H}{\multistep}-1}\sum_{s_\phi\in \mathcal{S}_\phi} \bar{V}^{k-1}_{n\multistep+1}(s_\phi) - \E[\bar{V}^{k}_{n\multistep+1}(s_\phi)\mid \F_{k-1}] \label{eq: regret bound multiple step rtdp abtractions}
\end{align}
The second relation holds the definition of the $\max$ operator and since $s_1^k\in \phi(s_1^k)$ (by definition we have that $s\in \Phi_{t}(s)$, as $\phi_t(s)=\phi_t(s)$ for any $t$). The third relation holds by the approximate optimism of the value function (Lemma \ref{lemma:multistep rtdp properties abstraction}), and the forth relation is by Lemma \ref{lemma: multistep RTDP expected value difference abstraction}. 

We now observe the regret is a regret of a Decreasing Bounded Process. Let
\begin{align}
    X_k \eqdef \sum_{n=1}^{\frac{H}{\multistep}-1}\sum_{s_\phi\in \mathcal{S}_\phi} \bar{V}^{k}_{n\multistep+1}(s_\phi), \label{def: DBP multiple rtdp abstractions}
\end{align}
and observe that $\brc*{X_k}_{g\geq 0}$ is a Decreasing Bounded Process.
\begin{enumerate}
    \item It is decreasing since for all $s_\phi\in\mathcal{S}_\phi,t$ $\bar{V}^k_t(s_\phi)\leq \bar{V}^{k-1}_t(s_\phi)$ by Lemma~\ref{lemma:multistep rtdp properties abstraction}. Thus, their sum is also decreasing.
    \item It is bounded since for all $s\in \mathcal{S}_\phi,t$ $\bar{V}^k_t(s_\phi)\geq 0$ by Lemma~\ref{lemma:multistep rtdp properties abstraction}. Thus, the sum is bounded from below by $0$.
\end{enumerate}

See that the initial value can be bounded as follows,
\begin{align*}
    X_0 &= \sum_{n=1}^{\frac{H}{\multistep}-1}\sum_{s_\phi\in \mathcal{S}_\phi} \bar{V}^{0}_{n\multistep+1}(s_\phi)\leq \sum_{n=1}^{\frac{H}{\multistep}-1}\sum_{s_\phi\in \mathcal{S}_\phi} H= \frac{S_\phi H(H-\multistep)}{\multistep}.
\end{align*}

Using linearity of expectation and the definition \eqref{def: DBP multiple rtdp} we observe that $\eqref{eq: regret bound multiple step rtdp abtractions}$ can be written,
\begin{align*}
    \Regret(K) \leq \eqref{eq: regret bound multiple step rtdp abtractions} = \epsilon_A\frac{H}{\multistep} K + \sum_{k=1}^K X_{k-1} - \E[X_k\mid \mathcal{F}_{k-1}],
\end{align*}

which is regret of A Bounded Decreasing Process. Applying the bound on the regret of a DRP, Theorem~\ref{theorem: regret of decreasing bounded process}, we conclude the proof of the first claim. 

We now prove {\bf claim (2)} using the proving technique at Theorem~\ref{theorem: regret multistep rtdp}. Denote $\Delta_A = H\epsilon_A$. The following relations hold for all $\epsilon>0$. 
\begin{align}
    &\ind\brc*{\bar{V}_1^{k}(\phi_1(s^k_1))-V_1^{\pi_k}(s^k_1) \geq  \epsilon} \epsilon  \nonumber \\
    &\leq \ind\brc*{\bar{V}_1^{k}(\phi_1(s^k_1))-V_1^{\pi_k}(s^k_1)  \geq  \epsilon } \br*{\bar{V}_1^{k}(\phi_1(s^k_1))-V_1^{\pi_k}(s^k_1) }
    \nonumber \\
    &\leq \ind\brc*{\bar{V}_1^{k}(\phi_1(s^k_1))-V_1^{\pi_k}(s^k_1) \geq \epsilon}\br*{\sum_{n=1}^{\frac{H}{\multistep}-1}\sum_{s_\phi\in\mathcal{S}_\phi} \bar{V}^{k-1}_{n\multistep+1}(s_\phi) - \E[\bar{V}^{k}_{n\multistep+1}(s_\phi)\mid \F_{k-1}] }\nonumber \\
    &= \ind\brc*{\bar{V}_1^{k}(s^k_1)-V_1^{\pi_k}(s^k_1) \geq \epsilon}\br*{X_{k-1} -\E[X_{k}\mid \mathcal{F}_{k-1}]}. \label{eq: theorem abstraction value first relation}
\end{align}
The first relation holds by the indicator function and the second relation holds by Lemma~\ref{lemma: multistep RTDP expected value difference abstraction}. The forth relation holds by  the definition of $X_k$~\eqref{def: DBP multiple rtdp abstractions} and linearity of expectation. 


As we wish the final performance to be compared to $V^*$ we use the the first claim of Lemma~\ref{lemma:multistep rtdp properties abstraction}, by which for all $s,k$, $\bar{V}_1^{k}(\phi_1(s)) \geq V_1^{*}(s) - \frac{ \Delta_A}{\multistep}$. This implies that 
\begin{align}
    \ind\brc*{ V_1^{*}(s^k_1)- V_1^{\pi_k}(s^k_1) \geq  \frac{ \Delta_A}{\multistep} +\epsilon } \leq \ind\brc*{\bar V_1^{k}(\phi_1(s^k_1))-V_1^{\pi_k}(s^k_1) \geq  \epsilon }\label{eq: theorem approximate abstraction second relation}.
\end{align}

Combining the above we get
\begin{align}
    &\ind\brc*{ V_1^{*}(s^k_1)- V_1^{\pi_k}(s^k_1) \geq  \frac{ \Delta_A}{\multistep}  +\epsilon }\epsilon \nonumber \\
    &\leq \ind\brc*{\bar{V}_1^{k}(\phi_1(s^k_1)) - V_1^{\pi_k}(s^k_1) \geq  \epsilon} \epsilon  \nonumber\\
    &\leq \ind\brc*{\bar{V}_1^{k}(\phi_1(s^k_1))-V_1^{\pi_k}(s^k_1) \geq \epsilon}\br*{X_{k-1} -\E[X_{k}\mid \mathcal{F}_{k-1}]}. \label{eq: central identity second claim approximate abstraction}
\end{align}
The first relation is by~\eqref{eq: theorem approximate abstraction second relation} and the second relation by~\eqref{eq: theorem abstraction value first relation}.

Define $N_\epsilon(K) = \sum_{k=1}^K \ind\brc*{V_1^{*}(s^k_1)-V_1^{\pi_k}(s^k_1) \geq \frac{ \Delta_A}{\multistep} + \epsilon}$ as the  number of times $V^*_1(s_1^k)-V^{\pi_k}_1(s_1^k) \geq \frac{ \Delta_A}{\multistep} + \epsilon $ at the first $K$ episodes. Summing the above inequality~\eqref{eq: central identity second claim approximate abstraction} for all $k\in[K]$ and denote  we get that for all $\epsilon>0$
\begin{align*}
    &N_\epsilon(K)\epsilon = \sum_{k=1}^K \ind\brc*{V_1^{*}(s^k_1)-V_1^{\pi_k}(s^k_1) \geq \frac{ \Delta_A}{\multistep} + \epsilon} \epsilon \\
    &\leq \sum_{k=1}^K \ind\brc*{\bar{V}_1^{k}(\phi_1(s^k_1))-V_1^{\pi_k}(s^k_1) \geq \epsilon}\br*{X_{k-1} -\E[X_{k}\mid \mathcal{F}_{k-1}]}\\
    &\leq \sum_{k=1}^K X_{k-1} -\E[X_{k}\mid \mathcal{F}_{k-1}].
\end{align*}
The first relation holds by definition, the second by~\eqref{eq: central identity second claim approximate abstraction} and the third relation holds as $\brc*{X_{k}}_{k\geq 0}$ is a DBP~\eqref{def: DBP multiple rtdp abstractions} and, thus, ${X_{k-1} - \E[X_{k}\mid \mathcal{F}_{k-1}] \geq 0}$ a.s. . Thus, the following relation holds
\begin{align*}
     \brc*{\forall K>0: \sum_{k=1}^K X_{k-1} -\E[X_{k}\mid \mathcal{F}_{k-1}] \leq \frac{9SH(H-\multistep)}{\multistep} \ln\frac{3}{\delta} } \subseteq \brc*{\forall \epsilon>0: N_\epsilon(K) \epsilon \leq \frac{9SH(H-\multistep)}{\multistep} \ln\frac{3}{\delta} },
\end{align*}

from which we get that for any $K>0$
\begin{align*}
    &\Pr\br*{\forall \epsilon>0: N_\epsilon(K) \epsilon \leq \frac{9SH(H-\multistep)}{\multistep} \ln\frac{3}{\delta} } \\
    &\geq \Pr\br*{\forall K>0: \sum_{k=1}^K X_{k-1} -\E[X_{k}\mid \mathcal{F}_{k-1}] \leq \frac{9SH(H\multistep)}{\multistep} \ln\frac{3}{\delta} } \geq 1- \delta,
\end{align*}
and the third relation holds the bound on the regret of DBP, Theorem~\ref{theorem: regret of decreasing bounded process}. Equivalently, for any $K>0$, 
\begin{align}
    \Pr\br*{\exists \epsilon>0: N_\epsilon(K) \epsilon \geq \frac{9SH(H-\multistep)}{\multistep} \ln\frac{3}{\delta} } \leq \delta.\label{eq: final pac bound multistep rtdp approixamte abstractions}
\end{align}

Applying the Monotone Convergence Theorem as in the proof of Theorem~\ref{theorem: regret multistep rtdp} we conclude the proof.

\end{proof}


\newpage
\section{Useful Lemmas}
\label{sec:useful-lemma}

The following lemma is a generalization of Lemma~34 in~\cite{efroni2019tight}. 

\begin{lemma}[On Trajectory Regret to Uniform Regret]\label{lemma: On trajectory regret to Uniform regret}
For any $t\in[H]$, let $\brc*{s_t^k,\mathcal{F}_{k}}_{k\geq 0}$ be a random process where $\brc*{s_t^k}_{k\geq 0}$ is adapted to the filtration $\brc*{\mathcal{F}_{k}}_{k\geq 0}$ and $s_t^k\in \tilde{\mathcal{S}}$ where $ \tilde{\mathcal{S}}$ is a finite set of all possible realizations of $s_t^k$ with cardinally $\tilde{S}\eqdef |\tilde{\mathcal{S}}|$. Let $g^k_t\in \mathbb{R}^{\tilde{S}}$ and denoting the $s\in \tilde{\mathcal{S}}$ entry of the vector as $g^k_t(s)$. Furthermore, let $g^k_t(s)$ be updated only at the state $s_t^k$ by an update rule which is $\mathcal{F}_{k-1}$ measurable, i.e.,
\begin{align*}
    g^k_t(s) = \begin{cases}
    f^{k-1}_t(s),\mathrm{if\ } s=s_t^k,\\
    g^{k-1}_t(s),\mathrm{o.w. . }
    \end{cases}
\end{align*}
Where $f^{k-1}_t(s)$ is an update rule $\mathcal{F}_{k-1}$ measurable. Then,

$$\sum_{k=1}^K \E[ g^{k-1}_t(s_t^k) - g^{k}_t(s_t^k)\mid \F_{k-1}] = \sum_{k=1}^K \sum_{s\in \tilde{\mathcal{S}}}  g^{k-1}_t(s)- \E[g^{k}_t(s)  \mid \F_{k-1}] $$

\end{lemma}
\begin{proof}
The following relations hold.
\begin{align}
     &\sum_{k=1}^K\sum_{t=1}^{H} \E[ g_t^{k-1}(s_t^k) - g_t^{k}(s_t^k)\mid \F_{k-1}] \nonumber \\ 
     &=\sum_{k=1}^K\sum_{t=1}^{H}\sum_{s\in \tilde{\mathcal{S}}} \E[ \ind\brc*{s =s_t^k}g_t^{k-1}(s) - \ind\brc*{s =s_t^k}g_t^{k}(s)\mid \F_{k-1}] \nonumber\\
     &\stackrel{(1)}{=}\sum_{k=1}^K\sum_{t=1}^{H}\sum_{s\in \tilde{\mathcal{S}}} \E[ \ind\brc*{s =s_t^k}g_t^{k-1}(s) - \ind\brc*{s =s_t^k}f_t^{k-1}(s)\mid \F_{k-1}] \nonumber\\
      &\stackrel{(2)}{=}\sum_{t=1}^{H}\sum_{s\in \tilde{\mathcal{S}}} \sum_{k=1}^K \E[\ind\brc*{s =s_t^k}g_t^{k-1}(s) + \ind\brc*{s \neq s_t^k}g_t^{k-1}(s)\mid \F_{k-1}]   \nonumber\\
      &\quad\quad\quad\quad\quad- \E[\ind\brc*{s =s_t^k}f_t^{k-1}(s) +\ind\brc*{s \neq s_t^k}g_t^{k-1}(s)  \mid \F_{k-1}] \nonumber\\
        &\stackrel{(3)}{=}\sum_{t=1}^{H}\sum_{s\in \tilde{\mathcal{S}}} \sum_{k=1}^K g_t^{k-1}(s)- \E[\ind\brc*{s =s_t^k}f_t^{k-1}(s) +\ind\brc*{s \neq s_t^k}g_t^{k-1}(s)  \mid \F_{k-1}] \nonumber\\
      &\stackrel{(4)}{=}\sum_{t=1}^{H}\sum_{s\in \tilde{\mathcal{S}}} \sum_{k=1}^K g_t^{k-1}(s)- \E[g_t^{k}(s)  \mid \F_{k-1}].
\end{align}
Relation $(1)$ holds since for $s=s_t^k$ the vector $g_k^t$ is updated according by $f^{k-1}$. Relation $(2)$ holds by adding and subtracting $\ind\brc*{s \neq s_t^k}g_t^{k-1}(s)$ while using the linearity of expectation. $(3)$ holds since for any event $\ind\brc{A}+\ind\brc{A^c}=1$ and since $ g_t^{k-1}$ is $\F_{k-1}$ measurable. $(4)$ holds by the definition of the update rule,
\begin{align*}
    &\E[\ind\brc*{s =s_t^k}f_t^{k-1}(s) +\ind\brc*{s \neq s_t^k}g_t^{k-1}(s)  \mid \F_{k-1}] \\
    &=\E[\ind\brc*{s =s_t^k}  \mid \F_{k-1}] f_t^{k-1}(s) +\E[\ind\brc*{s \neq s_t^k}\mid \F_{k-1}] g_t^{k-1}(s) \\
    &= \Pr(s_t^k=s\mid \mathcal{F}_{k-1})f_t^{k-1}(s) + \Pr(s_t^k\neq s\mid \mathcal{F}_{k-1})g_t^{k-1}(s) = \E[g_t^{k}(s) \mid \mathcal{F}_{k-1}].
\end{align*}
Where we used that $g^{k-1}_t(s)$ is $\mathcal{F}_{k-1}$ measurable and the assumption that $f^{k-1}_t(s)$ is $\mathcal{F}_{k-1}$ measurable in the first relation.
\end{proof}

The following lemma is a variant of a well known error propagation analysis in case of an approximate model.
\begin{lemma}[Model Error Propagation]\label{lemma: model error propogation}
Let $\norm{(P(\cdot \mid s,a) - \hat{P}(\cdot \mid s,a))}\leq \epsilon_P$ for any $s,a$. Then, for any policy $\pi$,
\begin{align*}
    \forall s_1\in \mathcal{S},\ \sum_{s_{n}} \bra*{P^{\pi}(s_{n}\mid s_1) - \hat{P}^{\pi}(s_{n}\mid s_1)} \leq n \epsilon_P
\end{align*}
\end{lemma}
\begin{proof}
We prove the claim by induction. For the base case $n=1$ we get that for any $s_1\in \mathcal{S}$
\begin{align*}
        &\sum_{s_{2}} \bra*{P^{\pi}(s_{2}\mid s_1) - \hat{P}^{\pi}(s_{2}\mid s_1)} \\
        &= \sum_{s_{2}} \bra*{\sum_{a}\pi(a\mid s_1)\br*{P(s_{2}\mid s_1,a) - \hat{P}^{\pi}(s_{2}\mid s_1,a)}} \\
        &\leq \sum_{a}\pi(a\mid s_1)\sum_{s_{2}} \bra*{P(s_{2}\mid s_1,a) - \hat{P}^{\pi}(s_{2}\mid s_1,a)} \\
        &= \sum_{a}\pi(a\mid s_1)\norm{P(\cdot \mid s_1,a)- P(\cdot \mid s_1,a)}_1 \leq \epsilon_P.
\end{align*}

Assume the induction step, i.e., assume the claim holds for $k=n-1$. We now prove the induction step, i.e., for $k=n$

\begin{align*}
        &\sum_{s_{n}} \bra*{P^{\pi}(s_{n}\mid s_1) - \hat{P}^{\pi}(s_{n}\mid s_1)} \\
        &=\sum_{s_{n}} \bra*{ \sum_{s_{2}} P^{\pi}(s_{n}\mid 
s_{2})P^{\pi}(s_{2}\mid s_1) - \hat{P}^{\pi}(s_{n}\mid s_{2})\hat{P}^{\pi}(s_{2}\mid s_1)} \\
        &\leq \sum_{s_{n}} \sum_{s_{2}} \bra*{ P^{\pi}(s_{n}\mid 
s_{2})P^{\pi}(s_{2}\mid s_1) - \hat{P}^{\pi}(s_{n}\mid s_{2})\hat{P}^{\pi}(s_{2}\mid s_1)} \\
        &\leq \sum_{s_{n}} \sum_{s_{2}} \bra*{ P^{\pi}(s_{n}\mid 
s_{2})P^{\pi}(s_{2}\mid s_1) - \hat{P}^{\pi}(s_{n}\mid 
s_{2})P^{\pi}(s_{2}\mid s_1)}\\
&\quad\quad\quad\quad + \bra*{\hat{P}^{\pi}(s_{n}\mid s_{2})\hat{P}^{\pi}(s_{2}\mid s_1) - \hat{P}^{\pi}(s_{n}\mid 
s_{2})P^{\pi}(s_{2}\mid s_1)} \\
        &\leq \sum_{s_{n}} \sum_{s_{2}} P^{\pi}(s_{2}\mid s_1) \bra*{ P^{\pi}(s_{n}\mid 
s_{2}) - \hat{P}^{\pi}(s_{n}\mid 
s_{2})}\\
&\quad\quad\quad\quad + \hat{P}^{\pi}(s_{n}\mid s_{2})\bra*{\hat{P}^{\pi}(s_{2}\mid s_1) - P^{\pi}(s_{2}\mid s_1)} \\
        &\leq  \underbrace{\sum_{s_{2}} P^{\pi}(s_{2}\mid s_1)}_{=1} \br*{\max_{s'_{2}}\sum_{s_{n}}\bra*{ P^{\pi}(s_{n}\mid 
s'_{2}) - \hat{P}^{\pi}(s_{n}\mid 
s'_{2})}}\\
&\quad\quad\quad\quad +  \sum_{s_{2}}\underbrace{\br*{\sum_{s_{n}}\hat{P}^{\pi}(s_{n}\mid s_{2})}}_{=1}\bra*{\hat{P}^{\pi}(s_{2}\mid s_1) - P^{\pi}(s_{2}\mid s_1)}\\
&=  \max_{s_{2}}\sum_{s_{n}}\bra*{ P^{\pi}(s_{n}\mid 
s_{2}) - \hat{P}^{\pi}(s_{n}\mid 
s_{2})} +   \sum_{s_{2}}\bra*{\hat{P}^{\pi}(s_{2}\mid s_1) - P^{\pi}(s_{2}\mid s_1)}.
\end{align*}
By the induction hypothesis and the base case,
\begin{align*}
    &\max_{s'_{2}}\sum_{s_{n}}\bra*{ P^{\pi}(s_{n}\mid 
s'_{2}) - \hat{P}^{\pi}(s_{n}\mid 
s'_{2})} \leq \epsilon(n-1)\\
    &\sum_{s_{2}}\bra*{\hat{P}^{\pi}(s_{2}\mid s_1) - P^{\pi}(s_{2}\mid s_1)} \leq \epsilon_P,
\end{align*}
from which we prove the induction step,
\begin{align*}
    \forall s_1\in\mathcal{S},\ \norm{P^{\pi}(\cdot \mid s_1)_1 - \hat{P}^{\pi}(\cdot \mid s_1)} = \sum_{s_{n}} \bra*{P^{\pi}(s_{n}\mid s_1) - \hat{P}^{\pi}(s_{n}\mid s_1)} \leq n \epsilon_P.
\end{align*}
\end{proof}

\begin{lemma}\label{lemma: approximate model bound for planning}

Let $V^*_t(s), \hat{V}^{*}_t(s)$ be the optimal values on the MDP $\mathcal{M}, \hat{\mathcal{M}}$, respectively, and let  $V^\pi_t(s), \hat{V}^{\pi}_t(s)$ be the value of a fixed policy $\pi$ on the MDP $\mathcal{M}, \hat{\mathcal{M}}$, respectively. Then,
\begin{align*}
    &i)\ ||V^{*}_1 - \hat{V}^{*}_1||_\infty \leq \frac{H(H-1)}{2}\epsilon_P,\\
    &ii) \ \forall \pi, \norm{V^{\pi}_1 - \hat{V}^{\pi}_1}_\infty \leq \frac{H(H-1)}{2}\epsilon_P.
\end{align*}

\end{lemma}
\begin{proof}

Both claims follow standard techniques based on the Simulation Lemma~\cite{kearns2002near,strehl2009reinforcement}.

\paragraph{{\em (i)}}
Let $\Delta_t(s)\eqdef \hat{V}_t^{*}(s) - V^{*}_t(s), \Delta_t=\max_s \bra*{\Delta_t(s)}$. For $t=H$ we have that for all $s$
\begin{align}
    \Delta_H(s) &= \max_{a} r(s,a)+\sum_{s'}\hat{P}(s'\mid s,a)\hat{V}_{H+1}^{*}(s')-\max_a r(s,a)+\sum_{s'}P(s'\mid s,a) V^{*}_{H+1}(s')\nonumber \\
    &= \max_{a} r(s,a)+\sum_{s'}\hat{P}(s'\mid s,a)\cdot0 -\max_a r(s,a)+\sum_{s'}P(s'\mid s,a) \cdot0 = 0, \label{eq: Delta H is 0 for approximate model}
\end{align}
and the base case holds. Assume the claim holds for any $t\geq k+1$, we now prove it holds for $t=k$. The following relations hold for any $s$,
\begin{align*}
    &\Delta_t(s) = \max_{a} r(s,a)+ \sum_{s'}\hat{P}(s'\mid s,a)\hat{V}_{t+1}^{*}(s')-\max_a r(s,a)+\sum_{s'}P(s'\mid s,a) v^{*}_{t+1}(s')\\
    &\leq r(s,a^*)+\sum_{s'}\hat{P}(s'\mid s,a^*)\hat{V}_{t+1}^{*}(s')- r(s,a^*)+\sum_{s'}P(s'\mid s,a^*) V^{*}_{t+1}(s')\\
    &=\sum_{s'}\hat{P}(s'\mid s,a^*)\hat{V}_{t+1}^{*}(s')- P(s'\mid s,a^*) V^{*}_{t+1}(s')\\
    &\leq \sum_{s'}\hat{P}(s'\mid s,a^*)\underbrace{\bra*{\hat{V}_{t+1}^{*}(s') - V^{*}_{t+1}(s')}}_{\eqdef \Delta_{t+1}(s')}+ \bra*{P(s'\mid s,a^*) - \hat P(s'\mid s,a^*)} V^{*}_{t+1}(s')\\
    &\leq \sum_{s'}\hat{P}(s'\mid s,a^*)\bra*{\Delta_{t+1}(s')} +  (H-t)\epsilon_P\\
    &\leq \Delta_{t+1}\sum_{s'}\hat{P}(s'\mid s,a^*)+  (H-t)\epsilon_P = \Delta_{t+1}+(H-t)\epsilon_P \\
\end{align*}
The second relation holds by choosing $a^*$ to maximize the first term first. The forth relation by adding and subtracting $\hat{P}(s'\mid s,a^*)\hat{V}_{t+1}^{*}(s')$ and standard inequalities. The fifth relation by the fact $V^*_{t+1}(s)\leq H-t$ and the assumption that for all $s,a$ $\norm{P(\cdot\mid s,a)-\hat{P}(\cdot \mid s,a)}\leq \epsilon_P$. The sixth by the fact $\hat{P}(\cdot \mid s,a)$ is a probability distribution and thus sums to $1$. 

Lower bounding $\Delta_t(s)$ using similar technique with opposite inequalities yields,
\begin{align*}
    &\Delta_t(s)\geq - (\Delta_{t+1}+(H-t)\epsilon_P),
\end{align*}
and thus,
\begin{align*}
    |\Delta_t(s)|\leq \Delta_{t+1}+(H-t)\epsilon_P.
\end{align*}

As the above holds for any $s$ it holds for the maximizer. Thus,
\begin{align*}
    \Delta_t \leq \Delta_{t+1}+(H-t)\epsilon_P.
\end{align*}

Iterating on this relation while using $\Delta_H(s)= 0 $ by \eqref{eq: Delta H is 0 for approximate model},

\begin{align*}
    ||V^{*}_1 - \hat{V}^{*}_1||_\infty = \Delta_1 \leq \sum_{t=1}^H (H-t)\epsilon_P=\epsilon_P\sum_{t=1}^{H-1} t = \frac{H(H-1)}{2}\epsilon_P.
\end{align*}

\paragraph{{\em (ii)}} The proof of the second claim follows the same proof of the first claim, without while replacing the $\max$ operator with the a fixed policy $\pi$.
\end{proof}

\begin{lemma}[Total Contribution of Approximate Model Errors]\label{lemma: algebraic bound for approximate model}
Let $d_n \eqdef -  \frac{1}{2}(\multistep-1)\multistep\epsilon_P + (H-n)\multistep\epsilon_P$. Then,
\begin{align*}
    \sum_{n=0}^{\frac{H}{\multistep}-1 }d_{1+n\multistep} = \frac{1}{2}H(H-1)\epsilon_P.
\end{align*}
\end{lemma}
\begin{proof}
The following relations hold.

\begin{align*}
    \sum_{n=0}^{\frac{H}{\multistep}-1 }d_{1+n\multistep} &= -  \frac{1}{2}H(\multistep-1)\epsilon_P +\sum_{n=0}^{\frac{H}{\multistep}-1 } (H-1-n\multistep)\multistep\epsilon_P\\
    &= -  \frac{1}{2}H(\multistep-1)\epsilon_P +H(H-1)\epsilon_P - \multistep^2\epsilon_P\sum_{n=0}^{\frac{H}{\multistep}-1 } n\\
    &= -  \frac{1}{2}H(\multistep-1)\epsilon_P +H(H-1)\epsilon_P - \frac{1}{2}\multistep^2\epsilon_P(\frac{H-\multistep}{\multistep})\frac{H}{\multistep}\\
    &= -  \frac{1}{2}H(\multistep-1)\epsilon_P +H(H-1)\epsilon_P - \frac{1}{2}H(H-\multistep)\epsilon_P\\
    &= -  \frac{1}{2}H(H-1)\epsilon_P +H(H-1)\epsilon_P = \frac{1}{2}H(H-1)\epsilon_P.
\end{align*}

\end{proof}


\newpage
\section{Approximate Dynamic Programming in Finite-Horizon MDPs}\label{supp: approximate dp bounds}

\begin{center}
\begin{minipage}{.45\linewidth}
\begin{algorithm}[H]
\begin{algorithmic}
\caption{(Exact) $\multistep$-DP}
\label{alg: backward induction dp}
    \STATE init: $\forall s\in \mathcal S,\; \forall n\in [\frac{H}{\multistep}],\; V_{n \multistep +1}(s)=H-n\multistep$
    \FOR{$n=\frac{H}{\multistep}-1,\frac{H}{\multistep}-2,\ldots, 1$}
    \FOR{$s\in \mathcal{S}$}
        \STATE $V_{n\multistep+1}(s) =   \br*{T^{\multistep} V_{(n+1)\multistep+1}}(s)$
    \ENDFOR
    \ENDFOR
    \STATE {\bf return:} $\brc*{ V_{n \multistep +1} }_{n=1}^{H/\multistep}$
\end{algorithmic}
\end{algorithm}
\end{minipage}
\hspace{0.5cm}
\begin{minipage}{.45\linewidth}
\begin{algorithm}[H]
\begin{algorithmic}
\caption{$\multistep$-DP with Approximate Model}
\label{alg: backward induction misspecified model dp}
    \STATE init: $\forall s\in \mathcal S,\; \forall n\in [\frac{H}{\multistep}],\; V_{n \multistep +1}(s)=H-n\multistep$
    \FOR{$n=\frac{H}{\multistep}-1,\frac{H}{\multistep}-2,\ldots, 1$}
    \FOR{$s\in \mathcal{S}$}
        \STATE $V_{n\multistep+1}(s) = \br*{\hat{T}^{\multistep}V_{(n+1)\multistep+1}}(s)$
    \ENDFOR
    \ENDFOR
    \STATE {\bf return:} $\brc*{V_{n \multistep +1}}_{n=1}^{H/\multistep}$
\end{algorithmic}
\end{algorithm}
\end{minipage}
\end{center}
\vspace{0.5cm}

In this section, we follow standard analysis~\cite{kearns2002near,strehl2009reinforcement} and establish bounds on the performance of approximate DP algorithms which update by an $\multistep$-step optimal Bellman operator~\eqref{eq:multistep bellman}. We abbreviate this class of algorithms by $h$-ADP. See that unlike previous analysis~\cite{kearns2002near,strehl2009reinforcement}, we focus on finite horizon MDPs, which is the setup in which $\multistep$-RTDP is analyzed. The different approximation setting we analyze in this section corresponds to the ones anlayzed for $\multistep$-RTDP: approximate model,  approximate value update, and approximate state abstraction. 

As a reminder and for the sake of completeness, we start by considering $\multistep$-DP Algorithm~\ref{alg: backward induction dp}, which is the exact, approximate-free, version of the following $\multistep$-ADP algorithms. The algorithm uses backward induction and a $\multistep$-step optimal Bellman operator $T^\multistep$ by which it outputs the values $\brc*{V_{n \multistep +1}}_{n=2}^{\frac{H}{\multistep}}$. Notice that it holds $\brc*{V_{n \multistep +1}}_{n=2}^{\frac{H}{\multistep}}=\brc*{V^*_{n \multistep +1}}_{n=2}^{\frac{H}{\multistep}}$ by standard arguments on the Backward Induction algorithm. Furthermore, $T^\multistep$ can be solved by Backward induction with the total computational complexity of $O(SA\multistep)$ by using Backward Induction. Thus, the total computational complexity of $\multistep$-DP is $O(SAH)$ similar to the one of standard DP, e.g., Backward Induction. 

In terms of space complexity, $\multistep$-DP stores in memory $O(S\frac{H}{\multistep})$ value entries. Observe that an $\multistep$-greedy policy~\eqref{eq: lookahead h greedy preliminaries} w.r.t. $\brc*{V_{n \multistep +1}}_{n=2}^{\frac{H}{\multistep}}$ is an optimal policy, as these values are the optimal values as previously observed.  Ultimately, one would like using these values to act in the environment by the optimal policy. If one uses the Forward-Backward DP~(Section~\ref{supp: epsiodic complexity h rtdp}) to calculate such an $\multistep$-greedy policy, then an extra $O(\multistep S_\multistep)$ space should be used, which results in total $O(S\frac{H}{\multistep}+ \multistep S_\multistep)$ space complexity (as in $\multistep$-RTDP) that decrease in $\multistep$ if $S_\multistep$ is not too big (see Remark~\ref{remark: space-comp compleixty of h rtdp}). Furthermore, the computational complexity of such approach is $O(H\multistep A S_{\multistep}S_1)$ which increases in $\multistep$.  


In next sections, we consider approximate settings of $\multistep$-DP and establish that an $\multistep$-greedy policy~\eqref{eq: lookahead h greedy preliminaries} w.r.t. the output values~$\brc*{V_{n \multistep +1}}_{n=2}^{\frac{H}{\multistep}}$  has an equivalent performance to the asymptotic policy by which $\multistep$-RTDP acts.


\subsection{$h$-ADP with an Approximate Model}

In the case of an approximate model, we replace the Bellman operator $T$ used in $\multistep$-DP with $\hat{T}$, the Bellman operator of the approximate model $\hat{p}$ instead the true one $p$ (we assume $r$ is exactly known, which correspond to the assumption made in Section~\ref{sec: appr model}). This results in Algorithm~\ref{alg: backward induction misspecified model dp}. Similarly to Section~\ref{sec: appr model}, we assume $\norm{\hat{p}(\cdot \mid s,a)-p(\cdot\mid s,a)}_{TV}\leq \epsilon_P$, for all $(s,a)\in \mathcal{S}\times \mathcal{A}$. Furthermore, denote $\pi_P^*$ as the optimal policy of the approximate MDP. 

Equivalently to $h$-DP, Algorithm~\ref{alg: backward induction misspecified model dp} returns the optimal values of the \emph{approximate model} (Algorithm~\ref{alg: backward induction misspecified model dp} can be interpreted as exact $\multistep$-DP applied on the approximate model). Thus, the $\multistep$-greedy policy w.r.t. to the outputs of Algorithm~\ref{alg: backward induction misspecified model dp} $\brc*{V_{n \multistep +1}}_{n=2}^{\frac{H}{\multistep}}$ is the optimal policy of the approximate MDP, $\pi_P^*$. The performance of $\pi^*_{P}$ is measured by relatively to the performance of the optimal policy, i.e., we wish to bound $\norm{V^*_1 - V^{\pi^*_{p}}_1}_\infty$. This term represents the performance gap between the optimal policy of the `real' MDP to the performance of the optimal policy of the approximate MDP evaluated on the real MDP, and is bounded in the following proposition.

\begin{proposition}\label{prop: misspecified model bound}
Assume for all $(s,a)\in\mathcal{S}\times \mathcal{A}: \norm{\hat{p}(\cdot \mid s,a)-p(\cdot\mid s,a)}_{TV}\leq \epsilon_P$  and let $\pi_{P}^*$ be the optimal policy of the approximate MDP. Then,
\begin{align*}
    \norm{ V^*_1 - V^{\pi^*_{P}}_1 }_\infty \leq H(H-1)\epsilon_P.
\end{align*}
\end{proposition}
\begin{proof}
Let $\hat{V}^{\pi^*_{P}}$ be the optimal value on the approximate MDP. By using the triangle inequality, the first and second claim of Lemma~\ref{lemma: approximate model bound for planning} we conclude the proof,
\begin{align*}
     \norm{ V^*_1 - V^{\pi^*_{P}}_1 }_\infty \leq  \norm{ V^*_1 - \hat{V}^{\pi^*_{P}}_1 }_\infty + \norm{\hat{V}^{\pi^*_{P}}_1 - V^{\pi^*_{P}}_1 }_\infty \leq H(H-1)\epsilon_P.
\end{align*}
\end{proof}


\begin{center}
\begin{minipage}{.43\linewidth}
\begin{algorithm}[H]
\begin{algorithmic}
\caption{$\multistep$-DP with Approximate Value Updates}
\label{alg: backward induction dp approximate value updates}
    \STATE init: $\forall s\in \mathcal S,\; \forall n\in [\frac{H}{\multistep}],\; V_{n \multistep +1}(s)=H-n\multistep$
    \FOR{$n=\frac{H}{\multistep}-1,\frac{H}{\multistep}-2,\ldots, 1$}
    \FOR{$s\in \mathcal{S}$}
        \STATE $\bar{V}^{k}_{t}(s) =   \epsilon_V(s) + \br*{T^{\multistep}V_{t+ \multistep}}(s)$
    \ENDFOR
    \ENDFOR
    \STATE {\bf return:} $\brc*{V_{n \multistep +1}}_{n=1}^{H/\multistep}$
\end{algorithmic}
\end{algorithm}
\end{minipage}
\hspace{0.5cm}
\begin{minipage}{.53\linewidth}
\begin{algorithm}[H]
\begin{algorithmic}
\caption{$\multistep$-DP with Approximate State abstraction}
\label{alg: backward induction dp abstractions}
    \STATE init: $\forall s\in \mathcal S,\; \forall n\in [\frac{H}{\multistep}],\; V_{n \multistep +1}(s)=H-n\multistep$
    \FOR{$n=\frac{H}{\multistep}-1,\frac{H}{\multistep}-2,\ldots, 1$}
    \FOR{$s\in \mathcal{S}$}
        \STATE $\bar{V}_{n\multistep+1}(\phi(s)) =   \min\brc*{ \br*{T^{\multistep}V_{(n+1)\multistep+1}}(s),\bar{V}^{k}_{n\multistep+1}(\phi(s))}$
    \ENDFOR
    \ENDFOR
    \STATE {\bf return:} $\brc*{V_{n \multistep +1}}_{n=1}^{H/\multistep}$
\end{algorithmic}
\end{algorithm}
\end{minipage}
\end{center}

\subsection{$h$-DP with Approximate Value Updates}
In the case of a approximate value updates Algorithm~\ref{alg: backward induction dp} is replaced by an value updates with added noise $\epsilon_V(s)$, by which Algorithm~\ref{alg: backward induction dp approximate value updates} is formulated. Similarly to the assumption used for $\multistep$-RTDP with approximate value updates (see Section~\ref{sec: appr value}) we assume for all $s\in\mathcal{S}$, $|\epsilon_V(s)|\leq \epsilon_V>0$. The following proposition bounds the performance of an $\multistep$-greedy policy w.r.t. the values output by Algorithm~\ref{alg: backward induction dp approximate value updates}.

\begin{proposition}\label{prop: approximate value updates}
Assume for all $s\in\mathcal{S},\ |\epsilon_V(s)|\leq \epsilon_V$. Let $\pi_{V}^*$ be the $\multistep$-greedy policy~\eqref{eq: lookahead h greedy preliminaries} w.r.t. output Algorithm~\ref{alg: backward induction dp approximate value updates}. Then,
\begin{align*}
    \norm{ V^*_1 - V^{\pi_{V}^*}_1 }_\infty \leq \frac{2H}{\multistep}\epsilon_V.
\end{align*}
\end{proposition}
\begin{proof}

Let $\brc*{\hat V^*_{n \multistep +1}}_{n=1}^{H/\multistep}$ denote the output of Algorithm~\ref{alg: backward induction dp approximate value updates}. We establish two claims which are of similarity to the two claims of Lemma~\ref{lemma: approximate model bound for planning}. Combining the two we prove the result.

\paragraph{{\em (i)}} The following relations hold for all $s\in\mathcal{S}$ and $n\in\brc*{0}\cup [\frac{H}{\multistep}-1]$.
\begin{align}
    &\Delta_{1+n\multistep}(s) \eqdef \hat{V}_{1+n\multistep}^{*}(s) - V^{*}_{1+n\multistep}(s) \nonumber\\
    &=\epsilon_V(s) + T^{\multistep}\hat{V}_{1+(n+1)\multistep}^{*}(s)- T^{\multistep} V^{*}_{1+(1+n)\multistep}(s')  \nonumber\\
    &=\epsilon_V(s) + \max_{a_0,\ldots,a_{\multistep-1}}\E\brs*{ \sum_{t'=0}^{\multistep-1}r(s_{t'},a_{t'}(s_{t'})) + \hat{V}_{t+h}^*(s_\multistep)\mid s_0=s}- \max_{a_0,\ldots,a_{\multistep-1}}\E\brs*{ \sum_{t'=0}^{\multistep-1}r(s_{t'},a_{t'}(s_{t'})) + V_{t+h}^*(s_\multistep)\mid s_0=s}\label{eq supp: approximate value eq 1}
\end{align}
The second relation holds by the updating equation and the third relation by definition~\eqref{eq:multistep bellman}. Let $\brc*{\hat a_0,\hat a_1,..,\hat a_{\multistep-1}}$ be the set of policies maximizes the second terms, then, by plugging this sequence to the third term we necessarily decrease it. Thus,
\begin{align*}
    \eqref{eq supp: approximate value eq 1} &\leq \epsilon_V(s) + \E\brs*{ \sum_{t'=0}^{\multistep-1}r(s_{t'},a_{t'}(s_{t'})) + \hat{V}_{t+h}^*(s_\multistep)\mid s_0=s, \brc*{ a_{t'}}_{t'=0}^{\multistep-1} = \brc*{ \hat a_{t'}}_{t'=0}^{\multistep-1} }\\
    &\quad - \E\brs*{ \sum_{t'=0}^{\multistep-1}r(s_{t'},a_{t'}(s_{t'})) + V_{t+h}^*(s_\multistep)\mid s_0=s, \brc*{ a_{t'}}_{t'=0}^{\multistep-1} = \brc*{ \hat a_{t'}}_{t'=0}^{\multistep-1} }\\
    &= \epsilon_V(s) + \E\brs*{ \hat{V}_{t+h}^*(s_\multistep) -  V_{t+h}^*(s_\multistep)\mid s_0=s,\brc*{ a_{t'}}_{t'=0}^{\multistep-1} = \brc*{ \hat a_{t'}}_{t'=0}^{\multistep-1} }\\
    &= \epsilon_V(s) + \E\brs*{ \Delta_{1+(n+1)\multistep}(s) \mid s_0=s,\brc*{ a_{t'}}_{t'=0}^{\multistep-1} = \brc*{ \hat a_{t'}}_{t'=0}^{\multistep-1} } \leq \epsilon_V + \norm{\Delta_{1+(n+1)\multistep}}_{\infty}.
\end{align*}
The second relation holds by linearity of expectation, the third relation by definition, and the forth by assumption on $\epsilon_V(s)$ and by the standard bounded $E[X]\leq \norm{X}_{\infty}$.

Repeating the above arguments while choosing the sequence which maximizes the third term in \eqref{eq supp: approximate value eq 1} allows us to lower bound \eqref{eq supp: approximate value eq 1} as follows $$\eqref{eq supp: approximate value eq 1}\geq -\epsilon_V - \norm{\Delta_{1+(n+1)\multistep}}_{\infty},$$ and thus,
\begin{align*}
    \norm{\Delta_{1+n\multistep}}_{\infty} \leq \epsilon_V +\norm{\Delta_{1+(n+1)\multistep}}_{\infty}.
\end{align*}

Solving the recursion while using $\Delta_{H+1}(s)=0$ for all $s\in\mathcal{S}$ we get
\begin{align}
    \norm{\Delta_{1}}_{\infty} \leq \frac{H}{\multistep}\epsilon_V. \label{eq supp: approximate value claim 1}
\end{align}

\paragraph{{\em (ii)}}
The following relations hold for all $s\in\mathcal{S}$ and $n\in[\frac{H}{\multistep}]$.
\begin{align*}
    &{\Delta}^{\pi^*_{V}}_{1+n\multistep}(s) \eqdef \hat{V}_{1+n\multistep}^{*}(s) - V^{\pi^*_{V}}_{1+n\multistep}(s) \\
    &= \epsilon_V(s) + \max_{a_0,..,a_{\multistep-1}}\E\brs*{ \sum_{t'=0}^{\multistep-1}r(s_{t'},a_{t'}(s_{t'})) + \hat{V}_{1+(n+1)\multistep}^*(s_\multistep)\mid s_0=s}\\
    &\quad - \E\brs*{ \sum_{t'=0}^{\multistep-1}r(s_{t'},a_{t'}(s_{t'})) + V^{\pi^*_{V}}_{1+(n+1)\multistep}(s_\multistep)\mid s_0=s,\pi^*_{V}}.
\end{align*}

By definition, the sequence which maximizes the second term is $\pi_V^*$ as it is the $\multistep$-greedy policy w.r.t. $\hat{V}^*$. Using the linearity of expectation we get 
\begin{align*}
    {\Delta}^{\pi^*_{V}}_{1+n\multistep}(s) &= \epsilon_V(s) + \E\brs*{  \hat{V}_{1+(n+1)\multistep}^*(s_\multistep)- V^{\pi^*_{V}}_{1+(n+1)\multistep}(s_\multistep)\mid s_0=s,\pi^*_{V}}\\
    &= \epsilon_V(s) + \E\brs*{  {\Delta}^{\pi^*_{V}}_{1+(n+1)\multistep}(s_{1+(n+1)\multistep})\mid s_0=s,\pi^*_{V}}.
\end{align*}

As for all $s$, $|\epsilon_V(s)|\leq \epsilon_V$, using the triangle inequality and $E[X]\leq \norm{X}_\infty$ we get the following recursion,
\begin{align*}
    \norm{{\Delta}^{\pi^*_{V}}_{1+n\multistep}}_{\infty} \leq \epsilon_V +\norm{{\Delta}^{\pi^*_{V}}_{1+(n+1)\multistep}}_{\infty}.
\end{align*}

Using $\norm{{\Delta}^{\pi^*_{V}}_{1+H}}_{\infty}=0$ we arrive to its solution,
\begin{align}
    \norm{{\Delta}^{\pi^*_{V}}_{1}}_{\infty} \leq \frac{H}{\multistep}\epsilon_V. \label{eq supp: approximate value claim 2}
\end{align}

which proves the second needed result.

\vspace{0.5cm}
Finally, using the triangle inequality and the two proven claims, \eqref{eq supp: approximate value claim 1} and \eqref{eq supp: approximate value claim 2}, we conclude the proof. 
\begin{align*}
  \norm{ V^*_1 - V^{\pi^*_{V}}_1 }_\infty &\leq \norm{ V^*_1 - \hat{V}_1 }_\infty +\norm{ \hat{V}^*_1 - V^{\pi^*_{V}}_1 }_\infty = \norm{ \Delta_1 }_\infty +\norm{ \Delta^{\pi^*_{V}}_1 }_\infty \leq 2\frac{H}{\multistep}\epsilon_V.
\end{align*}

\end{proof}

\subsection{$h$-DP with Approximate State Abstraction}\label{supp: adp approximate abstractions}
When an approximate state abstraction $\brc*{\phi_{1+n\multistep}}_{n=0}^{\frac{H}{\multistep}-1}$ is given, Algorithm~\ref{alg: backward induction dp} can be replaced by an exact value update in the reduced state space $\mathcal{S}_\phi$, as given in Algorithm~\ref{alg: backward induction dp abstractions}. This corresponds to updating a value $V\in \mathbb{R}^{S_\phi}$, instead a value $\mathbb{R}^{S}$. An obvious advantage of such an algorithm, relatively to $\multistep$-DP, is its reduced space complexity, as it only needs to store $O(\frac{H}{\multistep}S_\phi)$ value entries, instead of $O(\frac{H}{\multistep}S)$ as $\multistep$-DP.

Yet, as seen in Algorithm~\ref{alg: backward induction dp abstractions}, its computational complexity remains $O(SAH)$ as it needs to uniformly update on the entire (non-abstracted) state space. Would have we being given a representative from each equivalence classes under $\phi_{1+n\multistep}$ for every $n\in\brc*{0}\cup [\frac{H}{\multistep}]$\footnote{Differently put, if we interpret $\phi$ as clustering multiple states $s\in\mathcal{S}$ together, we would require a single representative from each such a cluster.} we could suggest an alternative Backward Induction algorithm with computational complexity of $O(S_\phi A H)$.  However, as we do not assume access to this knowledge, we are obliged to scan the entire state space, without further assumptions.

The following proposition bounds the performance of an $\multistep$-greedy policy w.r.t. the values output by  Algorithm~\ref{alg: backward induction dp abstractions}. Similarly to the analysis of the corresponding  $\multistep$-RTDP algorithm (see Section~\ref{sec: appr abstractions}), we assume  $\brc*{\phi_{1+n\multistep}}_{n=0}^{\frac{H}{\multistep}-1}$ satisfy Assumption~\ref{assumption: model abstraction}.

\begin{proposition}\label{prop: approximate abstraction}
Let $\brc*{\phi_{1+n\multistep}}_{n=0}^{\frac{H}{\multistep}-1}$ satisfy Assumption~\ref{assumption: model abstraction}. Let $\brc*{\hat{V}^*_{n \multistep +1}}_{n=1}^{\frac{H}{\multistep}}$ denote the output of Algorithm~\ref{alg: backward induction dp abstractions}  and let $\pi^*_{A}$ be the $\multistep$-greedy policy w.r.t. these approximate values~\eqref{eq: lookahead h greedy preliminaries}. Then,
\begin{align*}
    \norm{ V^*_1 - V^{\pi^*_{A}}_1 }_\infty \leq \frac{H}{\multistep}\epsilon_A.
\end{align*}
\end{proposition}
\begin{proof}
We establish two claims which are of similarily to the two claims of Lemma~\ref{lemma: approximate model bound for planning} and Proposition~\ref{prop: approximate value updates}. Combining the two we prove the result.

\paragraph{{\em (i)}} The following relations hold for any $s\in\mathcal{S}$.
\begin{align}
    &\hat{V}_{1+n\multistep}^{*}(\phi_{1+n\multistep}(s)) - V^{*}_{1+n\multistep}(s) \nonumber\\
    &= T^{\multistep}\hat{V}_{\phi,1+(n+1)\multistep}^{*}(s)- T^{\multistep} V^{*}_{1+(1+n)\multistep}(s)  \nonumber\\
    &= \max_{a_0,..,a_{\multistep-1}}\E\brs*{ \sum_{t'=0}^{\multistep-1}r(s_{t'},a_{t'}(s_{t'})) + \hat{V}_{1+(n+1)\multistep}^*(\phi_{1+(n+1)\multistep}(s_\multistep))\mid s_0=s }\nonumber\\
    &\quad - \max_{a_0,..,a_{\multistep-1}}\E\brs*{ \sum_{t'=0}^{\multistep-1}r(s_{t'},a_{t'}(s_{t'})) + V_{1+(n+1)\multistep}^*(s_\multistep)\mid s_0=s}\label{eq: abstractions first relation}
\end{align}
The second and third relation holds by the updating rule of Algorithm~\ref{alg: backward induction dp abstractions}. Let $\brc*{\hat a_0,\hat a_1,..,\hat a_{\multistep-1}}$ be the set of policies which maximizes the first term. Then, by plugging this sequence to the second term we necessarily decrease it, and the following holds.
\begin{align}
    &\eqref{eq: abstractions first relation} \leq \E\brs*{ \sum_{t'=0}^{\multistep-1}r(s_{t'},a_{t'}(s_{t'})) + \hat{V}_{1+(n+1)\multistep}^*(\phi_{1+(n+1)\multistep}(s_\multistep))\mid s_0=s, \brc*{ a_{t'}}_{t'=0}^{\multistep-1} = \brc*{ \hat a_{t'}}_{t'=0}^{\multistep-1} } \nonumber\\
    &\quad - \E\brs*{ \sum_{t'=0}^{\multistep-1}r(s_{t'},a_{t'}(s_{t'})) + V_{1+(n+1)\multistep}^*(s_\multistep)\mid s_0=s, \brc*{ a_{t'}}_{t'=0}^{\multistep-1} = \brc*{ \hat a_{t'}}_{t'=0}^{\multistep-1} } \nonumber\\
    &=  \E\brs*{ \hat{V}_{1+(n+1)\multistep}^*(\phi_{1+(n+1)\multistep}(s_\multistep)) -  V_{1+(n+1)\multistep}^*(s_\multistep)\mid s_0=s,\brc*{ a_{t'}}_{t'=0}^{\multistep-1} = \brc*{ \hat a_{t'}}_{t'=0}^{\multistep-1} }\label{eq: abstractions second relation}
\end{align}
Where the second relation holds by linearity of expectation. By Assumption~\ref{assumption: model abstraction} the following inequality holds,
\begin{align}
     \eqref{eq: abstractions second relation}&\leq  \E\brs*{ \hat{V}_{1+(n+1)\multistep}^*(\phi_{1+(n+1)\multistep}(s_\multistep)) -  \max_{\bar s_h\in \Phi_{1+(n+1)\multistep}(s_h)}V_{1+(n+1)\multistep}^*(\bar s_h) +\epsilon_A \mid s_0=s,\brc*{ a_{t'}}_{t'=0}^{\multistep-1} = \brc*{ \hat a_{t'}}_{t'=0}^{\multistep-1} }\nonumber\\
     &=  \epsilon_A + \E\brs*{ \hat{V}_{1+(n+1)\multistep}^*(\phi_{1+(n+1)\multistep}(s_\multistep)) -  \max_{\bar s_h\in \Phi_{1+(n+1)\multistep}(s_h)}V_{1+(n+1)\multistep}^*(\bar s_\multistep)  \mid s_0=s,\brc*{ a_{t'}}_{t'=0}^{\multistep-1} = \brc*{ \hat a_{t'}}_{t'=0}^{\multistep-1} } \nonumber\\
     &\leq  \epsilon_A + \max_s\left|  \hat{V}_{1+(n+1)\multistep}^*(\phi_{1+(n+1)\multistep}(s)) -  \max_{\bar s\in \Phi_{1+(n+1)\multistep}(s)}V_{1+(n+1)\multistep}^*(\bar s) \right| \label{eq: upper bound abstraction first}
\end{align}
By choosing the sequence of polices which maximizes the second term in~\eqref{eq: abstractions first relation} and repeating similar arguments to the above we arrive to the following relations.
\begin{align}
     &\eqref{eq: abstractions first relation} \geq \E\brs*{ \hat{V}_{1+(n+1)\multistep}^*(\phi_{1+(n+1)\multistep}(s_\multistep)) -  V_{1+(n+1)\multistep}^*(s_\multistep)\mid s_0=s,\brc*{ a_{t'}}_{t'=0}^{\multistep-1} = \brc*{ \hat a_{t'}}_{t'=0}^{\multistep-1} }\nonumber\\
     &\geq  \E\brs*{ \hat{V}_{1+(n+1)\multistep}^*(\phi_{1+(n+1)\multistep}(s_\multistep)) -  \max_{\bar s_h\in \Phi_{1+(n+1)\multistep}(s_h)}V_{1+(n+1)\multistep}^*(\bar s_\multistep)\mid s_0=s,\brc*{ a_{t'}}_{t'=0}^{\multistep-1} = \brc*{ \hat a_{t'}}_{t'=0}^{\multistep-1} }\nonumber\\
     &\geq -\E\brs*{ \left|\hat{V}_{1+(n+1)\multistep}^*(\phi_{1+(n+1)\multistep}(s_\multistep)) -  \max_{\bar s_h\in \phi_{1+(n+1)\multistep}^{-1}(s_h)}V_{1+(n+1)\multistep}^*(\bar s_\multistep)\right| \mid s_0=s,\brc*{ a_{t'}}_{t'=0}^{\multistep-1} = \brc*{ \hat a_{t'}}_{t'=0}^{\multistep-1} }\nonumber\\
     &\geq -\max_{s}\left|\hat{V}_{1+(n+1)\multistep}^*(\phi_{1+(n+1)\multistep}(s)) -  \max_{\bar{s}\in \Phi_{1+(n+1)\multistep}(s)}V_{1+(n+1)\multistep}^*(\bar{s})\right|  \label{eq: lower bound abstraction first}
\end{align}

Let $\Delta_{\phi,1+n\multistep}(s) \eqdef \hat{V}_{1+n\multistep}^{*}(\phi_{1+n\multistep}(s)) -  \max_{\bar s \in \Phi_{1+n\multistep}(s)}V^{*}_{1+n\multistep}(\bar s)$. The following upper bound holds,
\begin{align*}
    &\Delta_{\phi,1+n\multistep}(s) \eqdef \hat{V}_{1+n\multistep}^{*}(\phi_{1+n\multistep}(s)) -  \max_{\bar s \in \phi_{1+n\multistep}^{-1}(s)}V^{*}_{1+n\multistep}(\bar s)\\
    &\leq   \hat{V}_{1+n\multistep}^{*}(\phi_{1+n\multistep}(s)) -  V^{*}_{1+n\multistep}(s)\\
    &\leq \max_{s}|\hat{V}_{1+(n+1)\multistep}^*(\phi_{1+(n+1)\multistep}(s)) -  \max_{\bar s\in \Phi_{1+(n+1)\multistep}(s)}V_{1+(n+1)\multistep}^*(\bar s)| + \epsilon_A \\
    &= \norm{\Delta_{\phi,1+(n+1)\multistep}}_\infty + \epsilon_A.
\end{align*}
where the third relation is by \eqref{eq: upper bound abstraction first}. Furthermore, the following lower bounds holds,
\begin{align*}
    &\Delta_{\phi,1+n\multistep}(s) \eqdef \hat{V}_{1+n\multistep}^{*}(\phi_{1+n\multistep}(s)) -  \max_{\bar s \in \Phi_{1+n\multistep}(s)}V^{*}_{1+n\multistep}(\bar s)\\
    &\geq \hat{V}_{1+n\multistep}^{*}(\phi_{1+n\multistep}(s)) -  V^{*}_{1+n\multistep}(s) - \epsilon_A\\
    &\geq -\max_{s}\left|\hat{V}_{1+(n+1)\multistep}^*(\phi_{1+(n+1)\multistep}(s)) -  \max_{\bar{s}\in \Phi_{1+(n+1)\multistep}(s)}V_{1+(n+1)\multistep}^*(\bar{s})\right| - \epsilon_A \\
    &= - \norm{\Delta_{\phi,1+(n+1)\multistep}}_\infty - \epsilon_A,
\end{align*}
where the second relation holds by Assumption~\ref{assumption: model abstraction} and the third by~\eqref{eq: lower bound abstraction first}. 

By the upper and lower bounds on $\Delta_{\phi,1+n\multistep}(s)$ which holds for all $s$ we conclude that 
\begin{align*}
    \norm{\Delta_{\phi,1+n\multistep}}_\infty \leq \norm{\Delta_{\phi,1+(n+1)\multistep}}_\infty + \epsilon_A.
\end{align*}

Using $\norm{\Delta_{\phi,H+1}}_\infty = 0$ we solve the recursion and conclude that
\begin{align}
    \norm{\Delta_{\phi,1}}_\infty \leq \frac{H}{\multistep}\epsilon_A. \label{eq: ado abstractions first claim}
\end{align}

\paragraph{{\em (ii)}}
The following relations hold based on similar arguments as in~\eqref{eq supp: approximate value eq 1}. Let ${\Delta}^{\pi_{A}^*}_{1+n\multistep} \eqdef \max_{s} \hat{V}_{1+n\multistep}^{*}(\phi_{1+n\multistep}(s)) - V^{\pi_{A}^*}_{1+n\multistep}(s)$. For all $s$ the following relations hold.
\begin{align}
    &\hat{V}_{1+n\multistep}^{*}(\phi(s)) - V^{\pi_{A}^*}_{1+n\multistep}(s) \nonumber \\
    &\leq \max_{a_0,..,a_{\multistep-1}}\E\brs*{ \sum_{t'=0}^{\multistep-1}r(s_{t'},a_{t'}(s_{t'})) + \hat{V}_{1+(n+1)\multistep}^*(\phi(s_\multistep))\mid s_0=s} \nonumber \\
    &\quad - \E^{\pi_{A}^*}\brs*{ \sum_{t'=0}^{\multistep-1}r(s_{t'},a_{t'}(s_{t'})) + V^{\pi_{A}^*}_{1+(n+1)\multistep}(s_\multistep)\mid s_0=s}, \label{eq supp: abstractions second claim 1}
\end{align}
the first relation holds by the updating rule which update by the~(see Algorithm~\ref{alg: backward induction dp abstractions}), and since $V_t^{\pi} = (T^\pi)^\multistep V_{t+\multistep}^{\pi}$, similarly to the optimal Bellman operator~\eqref{eq:multistep bellman}.

By definition, the sequence which maximizes the first term is $\pi_A^*$ as it is the $\multistep$-greedy policy w.r.t. $\hat{V}^*$. Using the linearity of expectation we get 
\begin{align}
    \eqref{eq supp: abstractions second claim 1} &= \E^{\pi_{A}^*}\brs*{  \hat{V}_{1+(n+1)\multistep}^*(\phi(s_\multistep))- V^{\pi_{A}^*}_{1+(n+1)\multistep}(s_\multistep)\mid s_0=s} \nonumber \\
    &\leq  \max_{s} \hat{V}_{1+(n+1)\multistep}^*(\phi(s))- V^{\pi_{A}^*}_{1+(n+1)\multistep}(s) \eqdef {\Delta}^{\pi_{A}^*}_{1+(n+1)\multistep}.\label{eq supp: abstractions second claim 1 additional claim}
\end{align}

Since~\eqref{eq supp: abstractions second claim 1 additional claim} for all $s$ it also holds for the maximum, i.e.,
\begin{align*}
    {\Delta}^{\pi_{A}^*}_{1+n\multistep} \eqdef \max_s \hat{V}_{1+n\multistep}^{*}(\phi(s)) - V^{\pi_{A}^*}_{1+n\multistep}(s) \leq {\Delta}^{\pi_{A}^*}_{1+(n+1)\multistep}.
\end{align*}

As ${\Delta}^{\pi_{A}^*}_{H+1}=0$ and iterating on the above recursion we get,
\begin{align}
    {\Delta}^{\pi_{A}^*}_{1}\leq 0 \label{eq: ado abstractions second claim}.
\end{align}

We are now ready to prove the proposition. For any $s$ the following holds,
\begin{align*}
    V_1^*(s) - V_1^{\pi_{A}^*}(s) = \underbrace{V_1^*(s)  - \hat{V}_1(\phi(s))}_{(A)} + \underbrace{\hat{V}_1(\phi(s))- V_1^{\pi_{A}^*}(s)}_{B}.
\end{align*}

By \eqref{eq: ado abstractions first claim}
\begin{align*}
    (A) &\leq \max_{\bar s \in \Phi_1(s)}V_1^*(\bar s)  - \hat{V}_1(\phi_1(s))\\
    &\eqdef - \Delta_{\phi,1}(s)\leq \norm{\Delta_{\phi,1}}_\infty \leq \frac{H}{\multistep}\epsilon_A.
\end{align*}

By \eqref{eq: ado abstractions second claim},
\begin{align*}
    \hat{V}_1(\phi_1(s))- V_1^{\pi_{A}^*}(s) \leq \max_{\bar s} \br*{\hat{V}_1(\phi_1(\bar s))- V_1^{\pi_{A}^*}(\bar s)} = {\Delta}^{\pi_{A}^*}_{1} \leq 0.
\end{align*}

Lastly, combining the above and using $V^*\geq V^\pi$, we get that for all $s$
\begin{align*}
    &0\leq V_1^*(s) - V_1^{\pi_{A}^*}(s) \leq \frac{H}{\multistep}\epsilon_A.\\
    &\rightarrow  \norm{V_1^* - V_1^{\pi_{A}^*}}_{\infty} \leq \frac{H}{\multistep}\epsilon_A.
\end{align*}

\end{proof}

\end{document}